\documentclass[11pt,letterpaper]{article}
\usepackage[in]{fullpage}  
\usepackage{float}
\usepackage{mathpazo}

\usepackage[utf8]{inputenc}
\usepackage[T1]{fontenc}
\usepackage[english]{babel}
\usepackage{hyperref}       
\usepackage{url}            
\usepackage{booktabs}       
\usepackage{amsfonts}       
\usepackage{nicefrac}       
\usepackage{microtype}      
\usepackage{amsmath, amsthm, amssymb}
\usepackage{fancybox}
\usepackage{graphicx}
\usepackage{float}
\usepackage{algorithm}
\usepackage{color}
\usepackage{tikz}
\usepackage{pgfplots}
\usepackage{array}
\usepackage{subcaption}
\usepackage[textsize=scriptsize]{todonotes}
\usepackage{enumitem}
\usepackage{wrapfig}
\usepackage{enumitem}
\usepackage{etoolbox}
\usepackage{diagbox}
\usepackage{comment}

\newtheoremstyle{slplain}
  {.4\baselineskip\@plus.1\baselineskip\@minus.1\baselineskip}
  {.3\baselineskip\@plus.1\baselineskip\@minus.1\baselineskip}
  {\itshape}
  {}
  {\bfseries}
  {.}
  { }
  {}
\theoremstyle{slplain} 

\newtheorem*{definition*}{Definition}
\newtheorem*{theorem*}{Theorem}
\newtheorem{theorem}{Theorem}[section]

\newtheorem{corollary}[theorem]{Corollary}

\makeatletter
\newtheorem*{rep@theorem}{\rep@title}
\newcommand{\newreptheorem}[2]{%
\newenvironment{rep#1}[1]{%
 \def\rep@title{#2 \ref{##1}}%
 \begin{rep@theorem}}%
 {\end{rep@theorem}}}
\makeatother

\newreptheorem{theorem}{Theorem}
\newreptheorem{lemma}{Lemma}
\newreptheorem{claim}{Claim}
\newreptheorem{corollary}{Corollary}
\newreptheorem{proposition}{Proposition}

\theoremstyle{definition}

\theoremstyle{plain} 

\numberwithin{equation}{section}

\newtheoremstyle{etplain}
  {.0\baselineskip\@plus.1\baselineskip\@minus.1\baselineskip}
  {.0\baselineskip\@plus.1\baselineskip\@minus.1\baselineskip}
  {\itshape}
  {}
  {\bfseries}
  {.}
  { }
  {}

\newcommand\eps{\epsilon}

\newcommand\loss{L}
\newcommand\hood{\mathcal{S}}
\newcommand{\adv}{\ensuremath{\text{adv}}}
\newcommand{\D}{\mathcal{D}}
\newcommand{\E}{\mathbb{E}}
\newcommand{\R}{\mathbb{R}}

\newcommand\AND{
    \end{tabular}\hfil\linebreak[4]\hfil%
    \begin{tabular}[t]{c}\ignorespaces%
}

\renewcommand\bar\overline
\renewcommand\epsilon\varepsilon

\title{Towards Deep Learning Models Resistant to Adversarial Attacks}

\author{
  Aleksander M\k{a}dry\thanks{Authors ordered alphabetically.}\\
  MIT\\
  \texttt{madry@mit.edu} 
  \and
  Aleksandar Makelov\footnotemark[1]\\
  MIT\\
  \texttt{amakelov@mit.edu}
  \and
  Ludwig Schmidt\footnotemark[1]\\
  MIT\\
  \texttt{ludwigs@mit.edu}
  \AND
  Dimitris Tsipras\footnotemark[1]\\
  MIT\\
  \texttt{tsipras@mit.edu}
  \and
  Adrian Vladu\footnotemark[1]\\
  MIT\\
  \texttt{avladu@mit.edu} 
}

\date{}

\begin{document}
\maketitle

\begin{abstract}
    Recent work has demonstrated that deep neural networks are vulnerable to
adversarial examples---inputs that are almost indistinguishable from natural data and yet classified incorrectly by the network.
In fact, some of the latest findings suggest that the existence of adversarial attacks may be an inherent weakness of deep learning models.
To address this problem, we study the adversarial robustness of neural networks through the lens of robust optimization.
This approach provides us with a broad and unifying view on much of the prior work on this topic.
Its principled nature also enables us to identify methods for both training and attacking neural networks that are reliable and, in a certain sense, universal.
In particular, they specify a concrete security guarantee that would protect against \emph{any} adversary.
These methods let us train networks with significantly improved resistance to a wide range of adversarial attacks.
They also suggest the notion of security against a \emph{first-order adversary} as a natural and broad security guarantee.
We believe that robustness against such well-defined classes of adversaries is an important stepping stone towards fully resistant deep learning models.
\footnote{Code and pre-trained models are available at
\url{https://github.com/MadryLab/mnist_challenge}
and \url{https://github.com/MadryLab/cifar10_challenge}.}

\end{abstract}

\section{Introduction}
Recent breakthroughs in computer
vision~\cite{krizhevsky2012imagenet,he2015delving} and natural language
processing~\cite{collobert2008unified} are bringing
trained classifiers into the center of security-critical systems. Important
examples include vision for autonomous cars, face recognition, and malware
detection. These developments make security aspects of machine learning
increasingly important. In particular, resistance to \emph{adversarially chosen
inputs} is becoming a crucial design goal. While trained models tend to be very
effective in classifying benign inputs, recent
work~\cite{biggio2013evasion,szegedy2014intriguing, nguyen2015deep} shows that
an adversary is often able to manipulate the input so that the model produces an
incorrect output.

This phenomenon has received particular attention in the context of deep neural
networks, and there is now a quickly growing body of work on this
topic~\cite{goodfellow2015explaining,fawzi2018analysis,sokolic2017robust,kurakin2017adversarial,papernot2016transferability,tramer2017space}.
Computer vision presents a particularly striking challenge: very small changes
to the input image can fool state-of-the-art neural networks with high
confidence~\cite{szegedy2014intriguing,dezfooli2016deepfool}.
This holds even when the benign example was classified
correctly, and the change is imperceptible to a human. Apart from the security
implications, this phenomenon also demonstrates that our current models are not
learning the underlying concepts in a robust manner. All these findings raise a
fundamental question:

\begin{center}
  \emph{How can we train deep neural networks that are robust to adversarial inputs?}
\end{center}

There is now a sizable body of work proposing various attack and defense
mechanisms for the adversarial setting. Examples include defensive
distillation~\cite{papernot2016distillation,carlini2017towards},
feature squeezing~\cite{xu2018feature,he2017adversarial}, and several other adversarial example detection approaches \cite{carlini2017adversarial}. These works constitute important first steps in exploring the realm of possibilities here.  They, however, 
do not offer a good understanding of the \emph{guarantees} they
provide. We can never be certain that a given attack finds the ``most
adversarial'' example in the context, or that a particular defense mechanism prevents the
existence of some well-defined \emph{class} of adversarial attacks. This makes it difficult to
navigate the landscape of adversarial robustness or to fully evaluate the possible
security implications.

In this paper, we study the adversarial robustness of neural networks through
the lens of robust optimization. We use a natural saddle point (min-max)
formulation to capture the notion of security against adversarial attacks in a
principled manner. This formulation allows us to be precise about the type of
security \emph{guarantee} we would like to achieve, i.e., the broad \emph{class}
of attacks we want to be resistant to (in contrast to defending only against
specific known attacks). The formulation also enables us to cast both
\emph{attacks} and \emph{defenses} into a common theoretical framework,
naturally encapsulating most prior work on adversarial examples. In
particular, adversarial training directly corresponds to optimizing this saddle
point problem. Similarly, prior methods for attacking neural networks correspond
to specific algorithms for solving the underlying constrained optimization problem.

Equipped with this perspective, we make the following contributions.

\begin{enumerate} 
\item We conduct a careful experimental study of the optimization landscape
corresponding to this saddle point formulation. Despite the non-convexity and
non-concavity of its constituent parts, we find that the underlying optimization
problem \emph{is} tractable after all. In particular, we provide strong evidence
that first-order methods can reliably solve this problem. We supplement these
insights with ideas from real analysis to further motivate projected gradient
descent (PGD)   as a universal ``first-order adversary'', i.e., the strongest
attack utilizing the local first order information about the network.

\item We explore the impact of network architecture on adversarial robustness
and find that model capacity plays an important role here. To reliably withstand
strong adversarial attacks, networks require a larger capacity
than for correctly classifying benign examples only. This shows that a robust
decision boundary of the saddle point problem can be significantly more
complicated than a decision boundary that simply separates the benign data
points.

\item Building on the above insights, we train networks on
MNIST~\cite{lecun1998mnist} and CIFAR10~\cite{krizhevsky2009learning}
that are robust to a wide range of adversarial attacks. Our approach is based on
optimizing the aforementioned saddle point formulation and uses PGD as a
reliable first-order adversary. Our best MNIST model achieves an accuracy of more
than 89\% against the {strongest} adversaries in our test suite. In particular, our MNIST network is even robust against
\emph{white box} attacks of an \emph{iterative} adversary. Our CIFAR10 model achieves an accuracy of 46\% against the same adversary. Furthermore, in case of the weaker {\em black box/transfer} attacks, our MNIST and CIFAR10 networks achieve the accuracy of more than 95\% and 64\%, respectively. (More detailed overview can be found in Tables \ref{fig:super_mnist} and\ref{fig:super-cifar}.) To the best of our knowledge, we are the first to achieve these levels
of robustness on MNIST and CIFAR10 against such a broad set of attacks.
\end{enumerate}

Overall, these findings suggest that secure neural networks are within reach. In
order to further support this claim, we invite the community to attempt attacks
against our MNIST and CIFAR10 networks in the form of a challenge. This will let
us evaluate its robustness more accurately, and potentially lead to novel attack
methods in the process. The complete code, along with the description of the
challenge, is available at \url{https://github.com/MadryLab/mnist_challenge}
and \url{https://github.com/MadryLab/cifar10_challenge}.

\section{An Optimization View on Adversarial Robustness}
\label{sec:minmax}
Much of our discussion will revolve around an optimization view of adversarial
robustness. This perspective not only captures the phenomena we want to study in
a precise manner, but will also inform our investigations. To this end, let us
consider a standard classification task with an underlying data distribution $\D$ 
over pairs of examples $x \in \R^d$ and corresponding labels $y \in [k]$. We also
assume that we are given a suitable loss function $\loss(\theta, x, y)$, for instance the
cross-entropy loss for a neural network. As usual, $\theta \in \R^p$ is the set of
model parameters. Our goal then is to find model parameters $\theta$ that minimize
the risk $\E_{(x, y) \sim \D}[\loss(x, y, \theta)]$.

Empirical risk minimization (ERM) has been tremendously successful as a recipe
for finding classifiers with small population risk. Unfortunately, ERM often
does not yield models that are robust to adversarially crafted examples
\cite{biggio2013evasion,szegedy2014intriguing}. Formally, there
are efficient algorithms (``adversaries'') that take an example $x$ belonging
to class $c_1$ as input and find examples $x^\adv$ such that $x^\adv$ is very
close to $x$ but the model incorrectly classifies $x^\adv$ as belonging to
class $c_2 \neq c_1$.

In order to \emph{reliably} train models that are robust to adversarial attacks,
it is necessary to augment the ERM paradigm appropriately.  Instead of resorting
to methods that directly focus on improving the robustness to specific attacks,
our approach is to first propose a concrete \emph{guarantee} that an
adversarially robust model should satisfy. We then adapt our training methods
towards achieving this guarantee.

The first step towards such a guarantee is to specify an \emph{attack model},
i.e., a precise definition of the attacks our models should be resistant to. For
each data point $x$, we introduce a set of allowed perturbations $\hood \subseteq \R^d$
that formalizes the manipulative power of the adversary. In image classification, we choose
$\hood$ so that it captures perceptual similarity between images. For instance,
the $\ell_\infty$-ball around $x$ has recently been studied as a natural notion for
adversarial perturbations~\cite{goodfellow2015explaining}. While we focus on robustness against
$\ell_\infty$-bounded attacks in this paper, we remark that more comprehensive notions
of perceptual similarity are an important direction for future research.

Next, we modify the definition of population risk $\E_\D[\loss]$ by incorporating the
above adversary. Instead of feeding samples from the distribution $\D$
directly into the loss $\loss$, we allow the adversary to perturb the input first.
This gives rise to the following saddle point problem, which is our central
object of study:

\begin{equation}
\label{eq:minmax}
	\min_\theta \rho(\theta),\quad \text{ where }\quad \rho(\theta) =
    \mathbb{E}_{(x,y)\sim\mathcal{D}}\left[\max_{\delta\in 
    \hood}
    \loss(\theta,x+\delta,y)\right] \; .
\end{equation}
Formulations of this type (and their finite-sample counterparts) have a long
history in robust optimization, going back to Wald~\cite{wald1945statistical}.
It turns out that this formulation is also particularly useful in our context.

First, this formulation gives us a unifying perspective that encompasses much
prior work on adversarial robustness. Our perspective stems from viewing the
saddle point problem as the composition of an \emph{inner maximization} problem
and an \emph{outer minimization} problem. Both of these problems have a natural
interpretation in our context. The inner maximization problem aims to find an
adversarial version of a given data point $x$ that achieves a high loss. This is
precisely the problem of attacking a given neural network. On the other hand,
the goal of the outer minimization problem is to find model parameters so that
the ``adversarial loss'' given by the inner attack problem is minimized.
This is precisely the problem of training a robust classifier using adversarial
training techniques.

Second, the saddle point problem specifies a clear goal that an ideal robust
classifier should achieve, as well as a quantitative measure of its robustness.
In particular, when the parameters $\theta$ yield a (nearly) vanishing risk, the
corresponding model is perfectly robust to attacks specified by our attack model.

Our paper investigates the structure of this saddle point problem in the context
of deep neural networks. These investigations then lead us to training
techniques that produce models with high resistance to a wide range of
adversarial attacks. Before turning to our contributions, we briefly review
prior work on adversarial examples and describe in more detail how it fits into
the above formulation.

\subsection{A Unified View on Attacks and Defenses}
Prior work on adversarial examples has focused on two main questions:
\begin{enumerate}
\item How can we produce strong adversarial examples, i.e., adversarial examples that fool
a model with high confidence while requiring only a small perturbation? 
\item How can we train a model so that there are no adversarial examples, or at least so that an adversary cannot find them easily?
\end{enumerate}

Our perspective on the saddle point problem~\eqref{eq:minmax} gives answers to both these
questions. On the attack side, prior work has proposed methods such as the Fast
Gradient Sign Method (FGSM)~\cite{goodfellow2015explaining} and multiple
variations of it~\cite{kurakin2017adversarial}. FGSM is an
attack for an $\ell_\infty$-bounded adversary and computes an adversarial example as
\begin{align*}
x + \varepsilon \operatorname{sgn}(\nabla_x \loss(\theta,x,y)).
\end{align*}

One can interpret this attack as a simple one-step scheme for maximizing the
inner part of the saddle point formulation. A more powerful adversary is the
multi-step variant, which is essentially projected gradient descent (PGD)
on the negative loss function
\begin{align*}
x^{t+1} = \Pi_{x+\hood} \left( x^t +
\alpha\operatorname{sgn}(\nabla_x \loss(\theta,x,y))\right).
\end{align*}

Other methods like FGSM with random perturbation have also been proposed
\cite{tramer2017space}. Clearly, all of these approaches can be viewed as specific attempts to
solve the inner maximization problem in~\eqref{eq:minmax}.

On the defense side, the training dataset is often augmented with adversarial
examples produced by FGSM. This approach also directly follows from
(\ref{eq:minmax})
when linearizing the inner maximization problem.
To solve the simplified robust optimization problem, we replace every
training example with its FGSM-perturbed counterpart. More sophisticated defense
mechanisms such as training against multiple adversaries can be seen as better, more exhaustive approximations of the inner maximization problem.

\section{Towards Universally Robust Networks}
Current work on adversarial examples usually focuses on specific defensive mechanisms, or on attacks against such defenses. An important feature of formulation \eqref{eq:minmax} is that attaining small adversarial loss gives a \emph{guarantee} that no allowed attack will fool the network. By definition, no adversarial perturbations are possible because the loss is small for \emph{all} perturbations allowed by our attack model. Hence, we now focus our attention on obtaining a good solution to \eqref{eq:minmax}.
 
Unfortunately, while the overall guarantee provided by the saddle point problem is evidently useful, it is not clear whether we can actually find a good solution in reasonable time. Solving the saddle point problem \eqref{eq:minmax} involves tackling both a non-convex outer minimization problem \emph{and} a non-concave inner maximization problem. One of our key contributions is demonstrating that, in practice, one \emph{can} solve the saddle point problem after all. In particular, we now discuss an experimental exploration of the structure given by the non-concave inner problem. We argue that the loss landscape corresponding to this problem has a surprisingly tractable structure of local maxima. This structure also points towards projected gradient descent as the “ultimate” first-order adversary. Sections \ref{sec:capacity} and \ref{sec:experiments} then show that the resulting trained networks are indeed robust against a wide range of attacks, provided the networks are sufficiently large.

\subsection{The Landscape of Adversarial Examples}
Recall that the inner problem corresponds to finding an adversarial example for
a given network and data point (subject to our attack model). As this problem
requires us to maximize a highly non-concave function, one would expect it to be
intractable. Indeed, this is the conclusion reached by prior work which then
resorted to linearizing the inner maximization problem
\cite{huang2015learning,shaham2018understanding}. As pointed out above, this
linearization approach yields well-known methods such as FGSM. While training
against FGSM adversaries has shown some successes, recent work also highlights
important shortcomings of this one-step
approach~\cite{tramer2017space}---slightly more sophisticated adversaries can
still find points of high loss.
 
To understand the inner problem in more detail, we investigate the landscape of local maxima for multiple models on MNIST and CIFAR10. The main tool in our experiments is projected gradient descent (PGD), since it is the standard method for large-scale constrained optimization. In order to explore a large part of the loss landscape, we re-start PGD from many points in the $\ell_\infty$ balls around data points from the respective evaluation sets.
 
Surprisingly, our experiments show that the inner problem \emph{is} tractable after all, at least from the perspective of first-order methods. While there are many local maxima spread widely apart within $x_i + \hood$, they tend to have very {\em well-concentrated} loss {\em values}. This echoes the folklore belief that training neural networks is possible because the loss (as a function of model parameters) typically has many local minima with very similar values.
 
Specifically, in our experiments we found the following phenomena:
 
\begin{itemize}[leftmargin=*]
  \item We observe that the loss achieved by the adversary increases in a fairly consistent way and plateaus rapidly when performing projected $\ell_\infty$ gradient descent for randomly chosen starting points inside $x+\hood$ (see Figure~\ref{fig:mnist_progress}).
\begin{figure}[htb]
\begin{center}
{\setlength\tabcolsep{-.0cm}
\begin{tabular}{c c c c c}
\includegraphics[height=.19\textwidth]{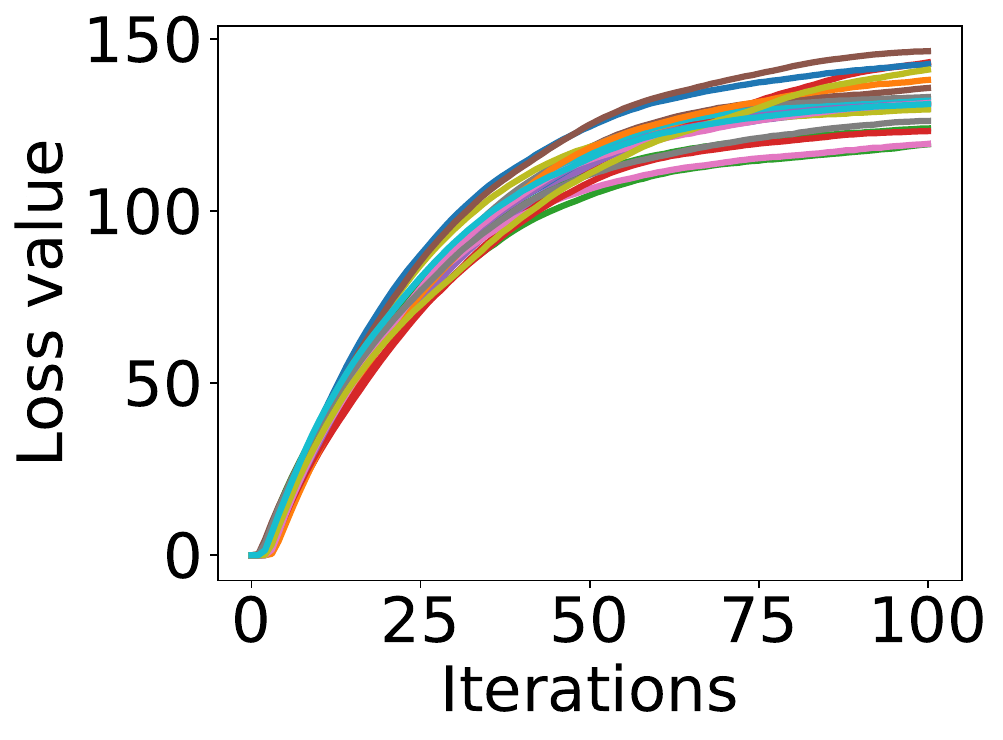} &
\includegraphics[height=.19\textwidth]{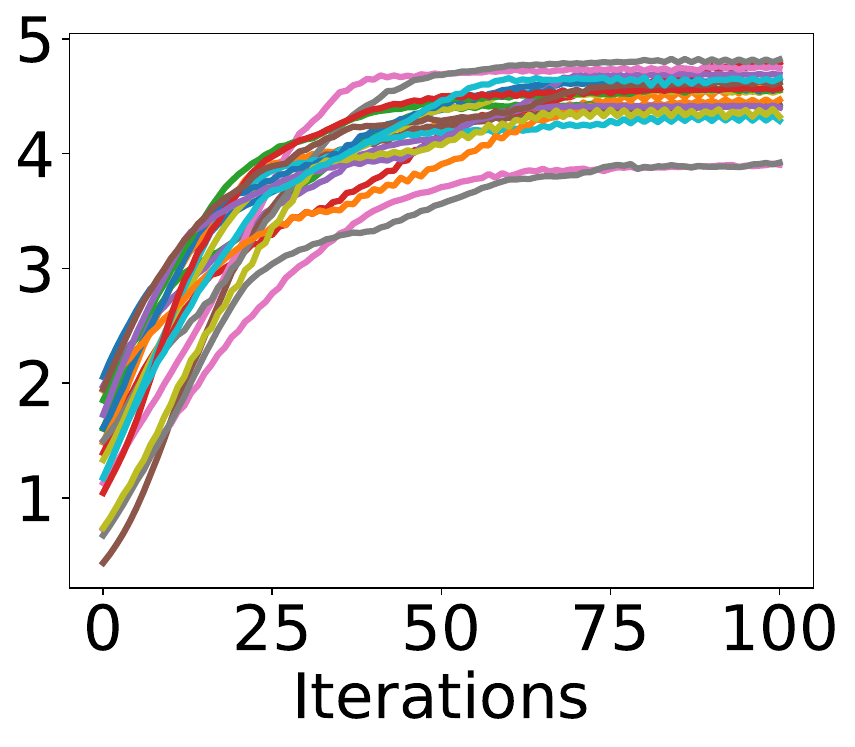} &
\includegraphics[height=.19\textwidth]{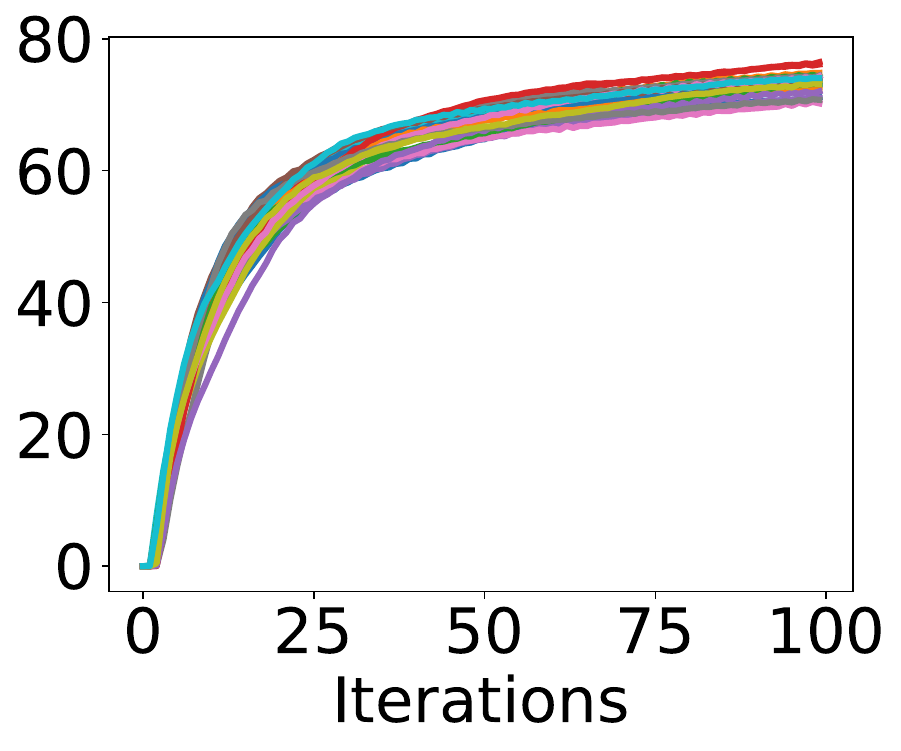} &
\includegraphics[height=.19\textwidth]{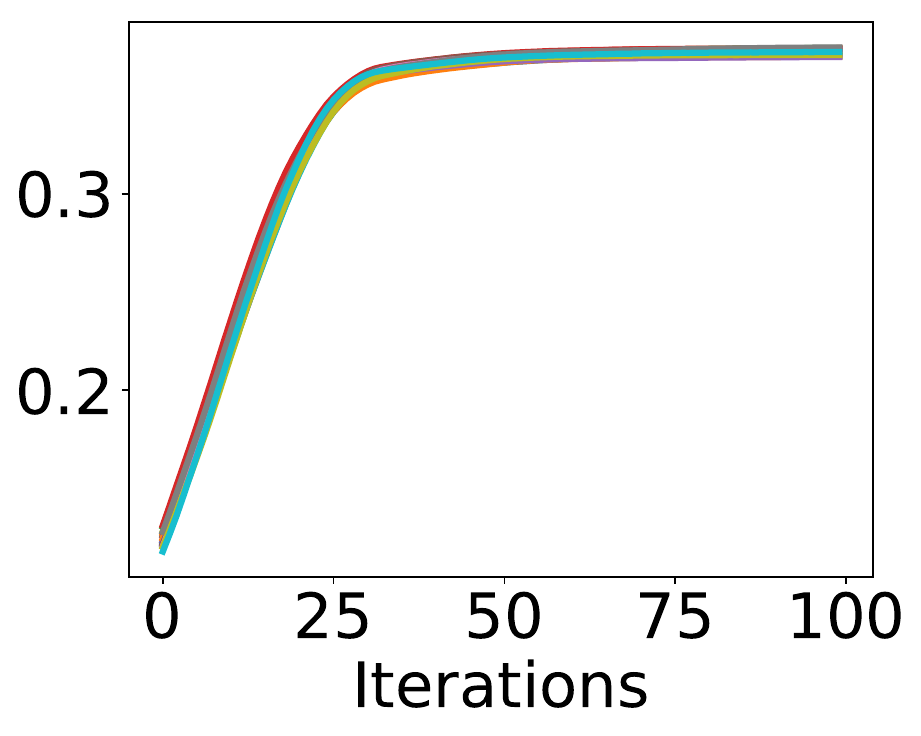}
\\
\hspace{.3cm} (a) MNIST  & (b) MNIST & (c) CIFAR10  & (d) CIFAR10 \\
\hspace{.3cm} Standard training  & Adversarial training & Natural training  &  Adversarial training \\
\end{tabular}}
\end{center}
\caption{Cross-entropy loss values while creating an adversarial example from the MNIST and CIFAR10 evaluation datasets.
The plots show how the loss evolves during 20 runs of projected gradient descent (PGD).
Each run starts at a uniformly random point in the $\ell_\infty$-ball around the same natural example (additional plots for different examples appear in Figure~\ref{fig:mnist_progress_appendix}).
The adversarial loss plateaus after a small number of iterations.
The optimization trajectories and final loss values are also fairly clustered, especially on CIFAR10.
Moreover, the final loss values on adversarially trained networks are
significantly smaller than on their standard counterparts.
}
\label{fig:mnist_progress}
\end{figure}

  \item Investigating the concentration of maxima further, we observe that over a large number of random restarts, the loss of the final iterate follows a well-concentrated distribution without extreme outliers (see Figure~\ref{fig:mnist_concentration}; we verified this concentration based on $10^5$ restarts).
\begin{figure}[htb]
\begin{center}
{\setlength\tabcolsep{-.0cm}
\begin{tabular}{c c c c c}
& & MNIST & &\\
\includegraphics[height=.17\textwidth]{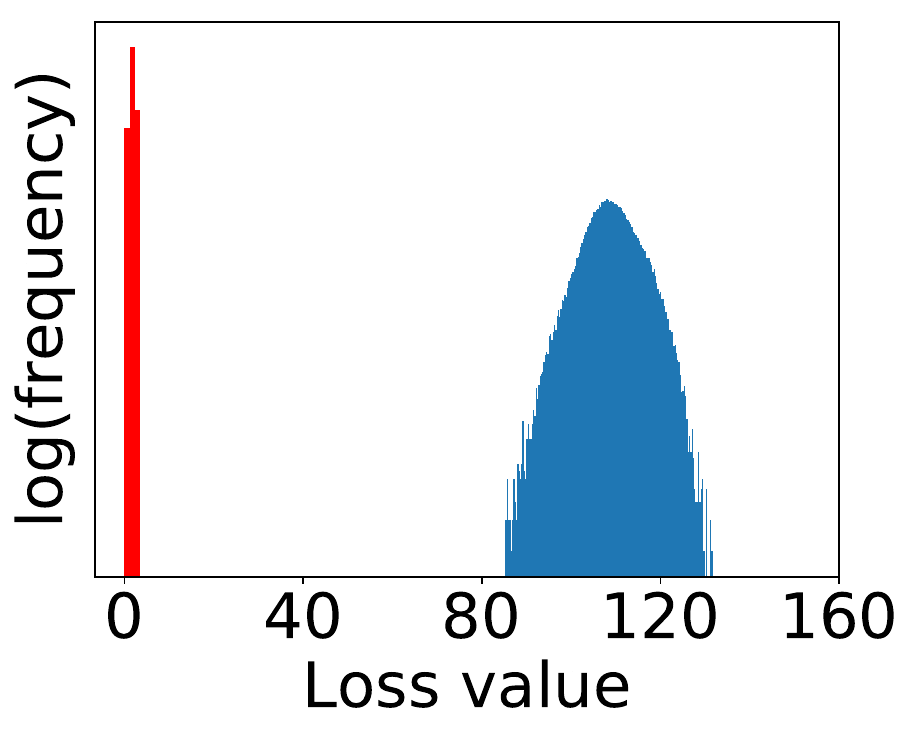} &
\includegraphics[height=.17\textwidth]{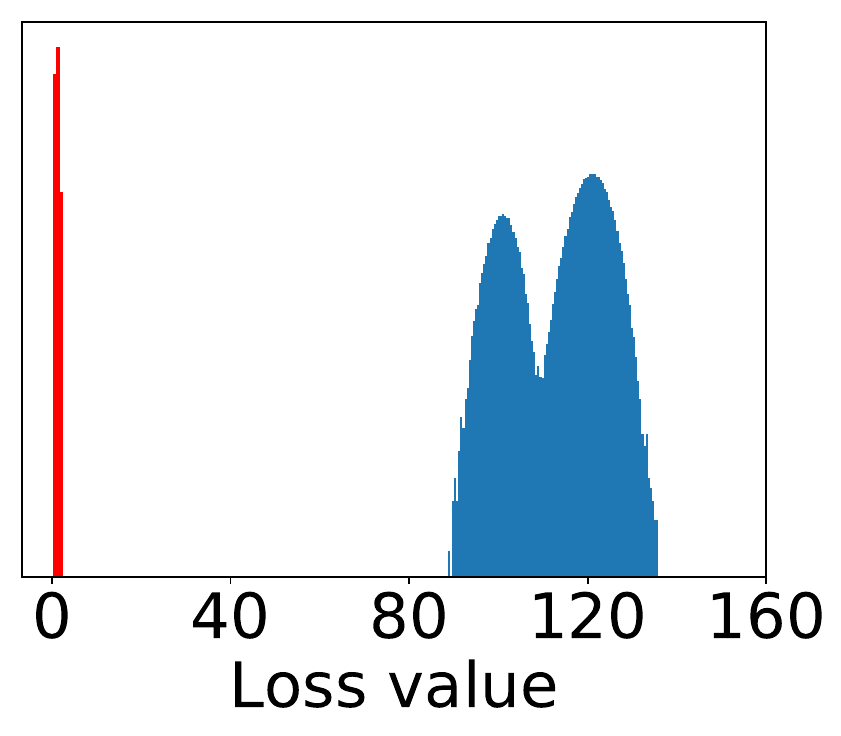} &
\includegraphics[height=.17\textwidth]{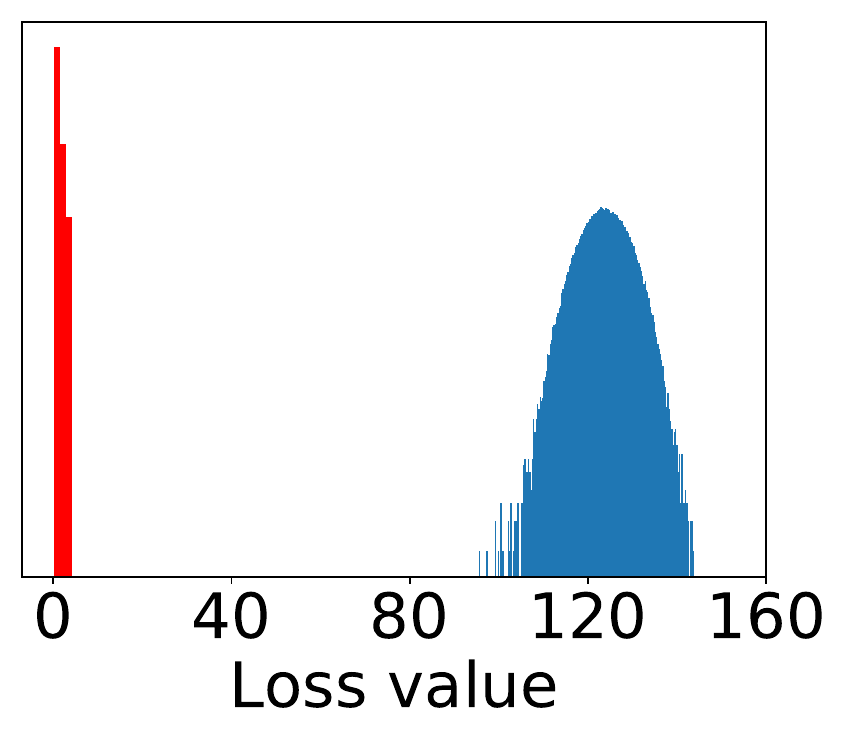} &
\includegraphics[height=.17\textwidth]{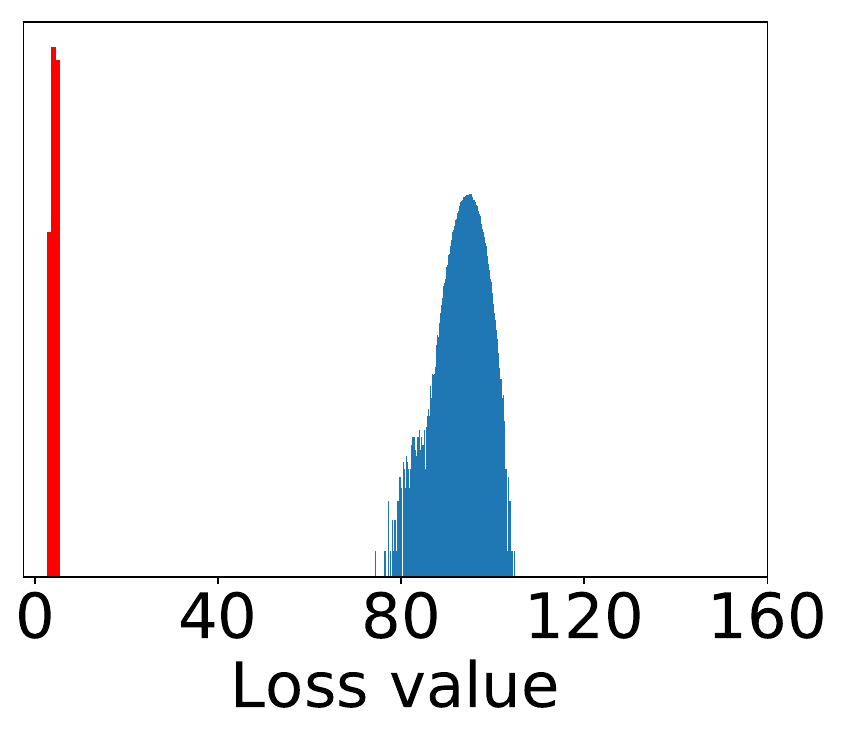} &
\includegraphics[height=.17\textwidth]{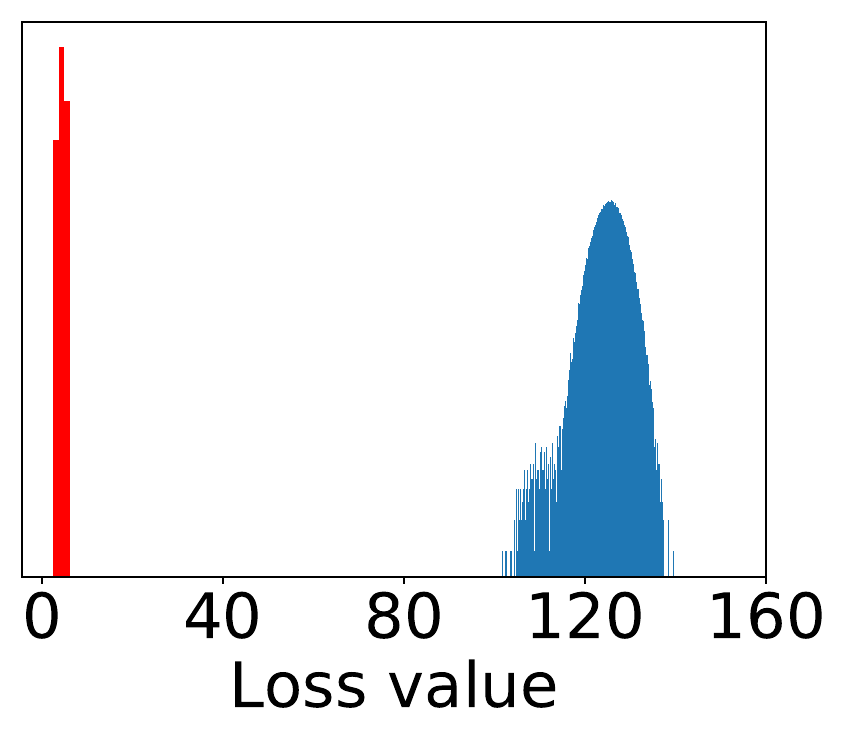} \\
& & CIFAR10 & &\\
\includegraphics[height=.17\textwidth]{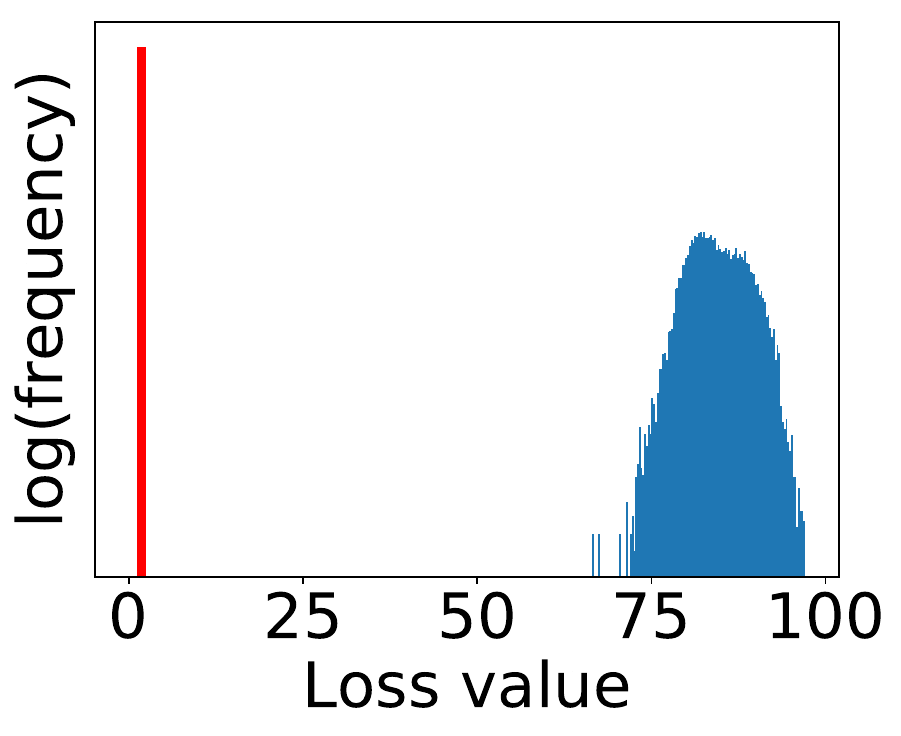} &
\includegraphics[height=.17\textwidth]{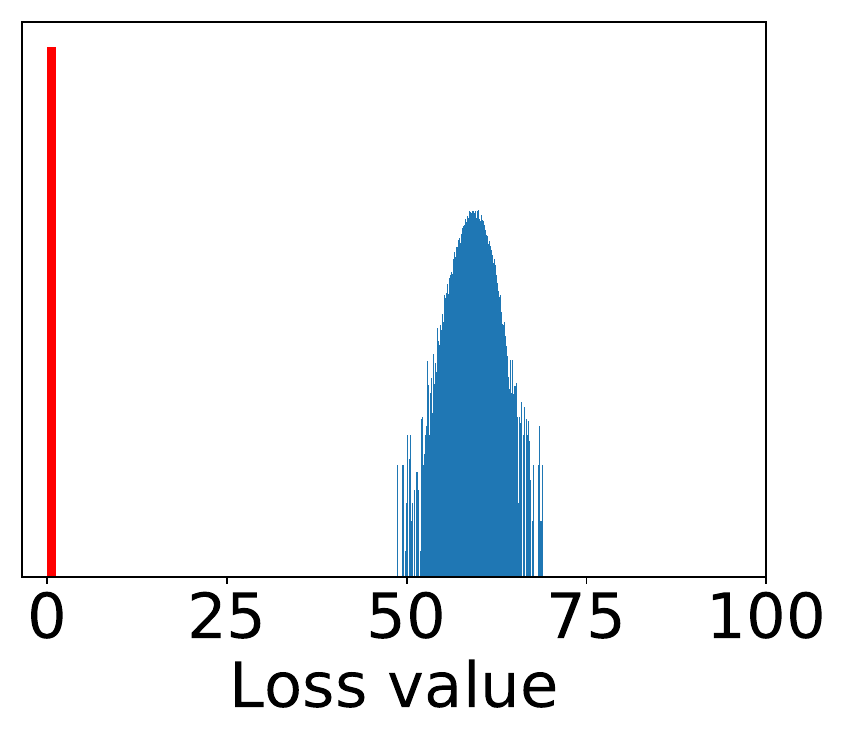} &
\includegraphics[height=.17\textwidth]{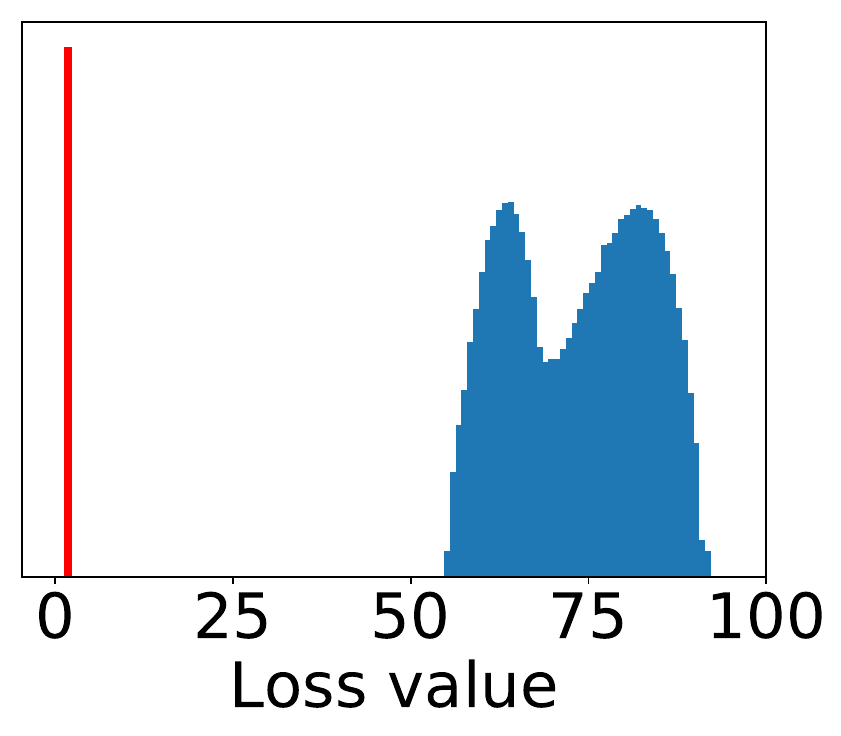} &
\includegraphics[height=.17\textwidth]{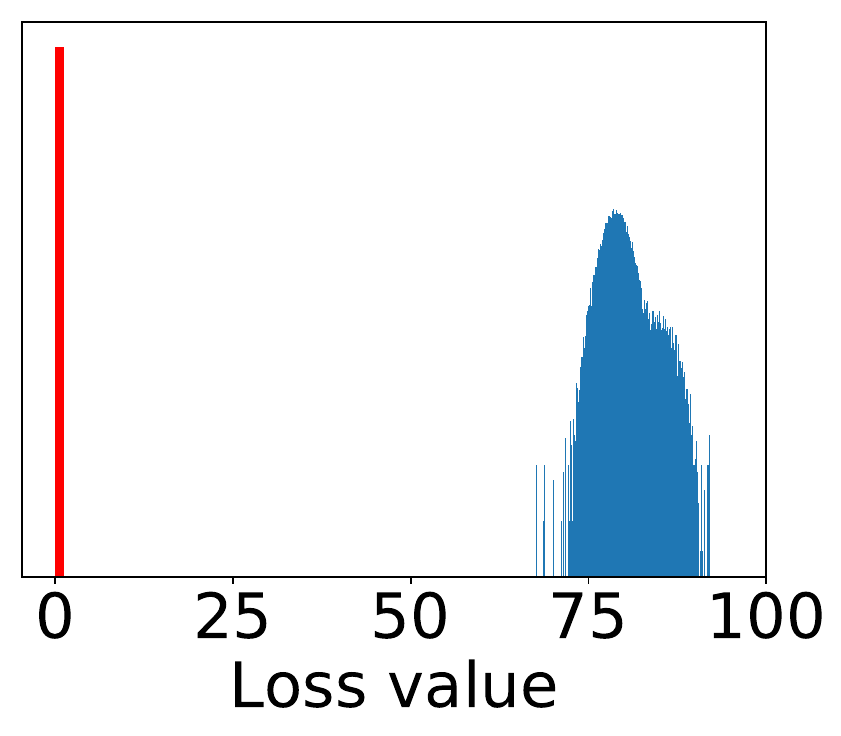} &
\includegraphics[height=.17\textwidth]{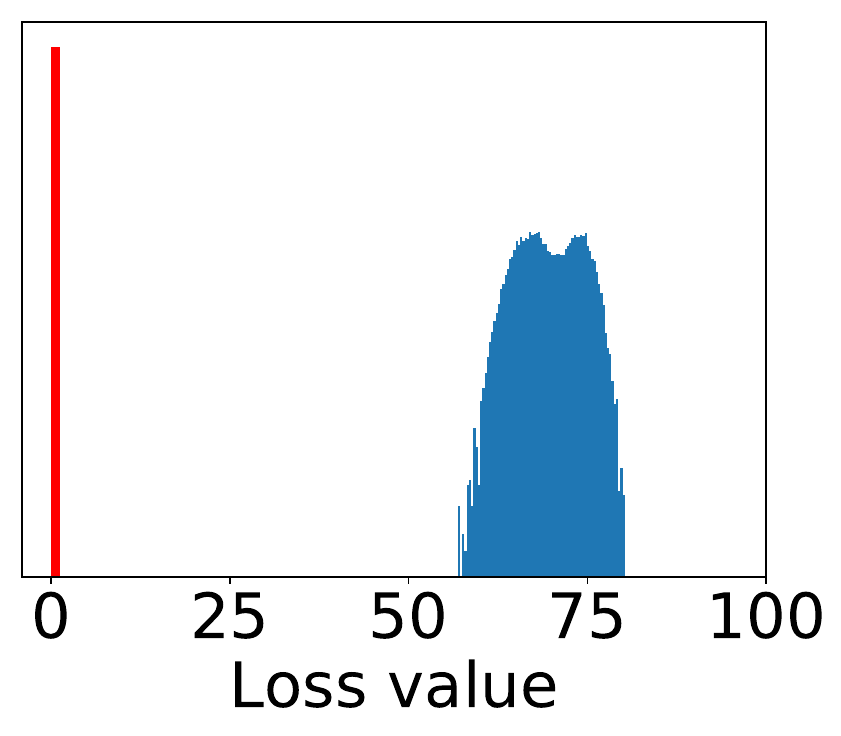} 
\end{tabular}}
\end{center}
\caption{Values of the local maxima given by the cross-entropy loss for five examples from the MNIST and CIFAR10 evaluation datasets.
For each example, we start projected gradient descent (PGD) from $10^5$
uniformly random points in the $\ell_\infty$-ball around the example and iterate PGD until the loss plateaus.
The blue histogram corresponds to the loss on a standard network, while the red histogram corresponds to the adversarially trained counterpart.
The loss is significantly smaller for the adversarially trained networks, and the final loss values are very concentrated without any outliers.
}
\label{fig:mnist_concentration}
\end{figure}

  \item To demonstrate that maxima are noticeably distinct, we also measured the
  $\ell_2$ distance and angles between all pairs of them and observed that distances are distributed close to the expected distance between two random points in the $\ell_\infty$ ball, and angles are close to $90^\circ$.
  Along the line segment between local maxima, the loss is convex, attaining its
  maximum at the endpoints and is reduced by a constant factor in the middle.
  Nevertheless, for the entire segment, the loss is considerably higher than that of a random point.
  \item Finally, we observe that the distribution of maxima suggests that the
      recently developed subspace view of adversarial examples is not fully
      capturing the richness of attacks~\cite{tramer2017space}. In particular, we observe adversarial perturbations with negative inner product with the gradient of the example, and deteriorating overall correlation with the gradient direction as the scale of perturbation increases.
\end{itemize}
 
All of this evidence points towards PGD being a ``universal'' adversary among first-order approaches, as we will see next.
 
\subsection{First-Order Adversaries}
Our experiments show that the local maxima found by PGD all have similar loss values, both for normally trained networks and adversarially trained networks. This concentration phenomenon suggests an intriguing view on the problem in which robustness against the PGD adversary yields robustness against \emph{all} first-order adversaries, i.e., attacks that rely only on first-order information. As long as the adversary only uses gradients of the loss function with respect to the input, we conjecture that it will not find significantly better local maxima than PGD. We give more experimental evidence for this hypothesis in Section \ref{sec:experiments}: if we train a network to be robust against PGD adversaries, it becomes robust against a wide range of other attacks as well.
 
Of course, our exploration with PGD does not preclude the existence of some isolated maxima with much larger function value. However, our experiments suggest that such better local maxima are \emph{hard to find} with first order methods: even a large number of random restarts did not find function values with significantly different loss values. Incorporating the computational power of the adversary into the attack model should be reminiscent of the notion of  \emph{polynomially bounded} adversary that is a cornerstone of modern cryptography. There, this classic attack model allows the adversary to only solve problems that require at most polynomial computation time. Here, we employ an \emph{optimization-based} view on the power of the adversary as it is more suitable in the context of machine learning. After all, we have not yet developed a thorough understanding of the computational complexity of many recent machine learning problems. However, the vast majority of optimization problems in ML is solved with first-order methods, and variants of SGD are the most effective way of training deep learning models in particular. Hence we believe that the class of attacks relying on first-order information is, in some sense, universal for the current practice of deep learning.
 
Put together, these two ideas chart the way towards machine learning models with \emph{guaranteed} robustness. If we train the network to be robust against PGD adversaries, it will be robust against a wide range of attacks that encompasses all current approaches.
 
In fact, this robustness guarantee would become even stronger in the context of
{\em black-box attacks}, i.e., attacks in which the adversary does not have a
direct access to the target network. Instead, the adversary only has less
specific information such as the (rough) model architecture and the training
data set. One can view this attack model as an example of ``zero order''
attacks, i.e., attacks in which the adversary has no direct access to the
classifier and is only able to evaluate it on chosen examples without gradient
feedback.
 
We discuss transferability in Section~\ref{sec:trans} of the appendix. We
observe that increasing network capacity and strengthening the adversary we
train against (FGSM or PGD training, rather than standard training) improves resistance against transfer attacks. Also, as expected, the resistance of our best models to such attacks tends to be significantly larger than to the (strongest) first order attacks.

\subsection{Descent Directions for Adversarial Training}
 
The preceding discussion suggests that the inner optimization problem can be
successfully solved by applying PGD.  In order to train adversarially robust
networks, we also need to solve the \emph{outer} optimization problem of the
saddle point formulation~\eqref{eq:minmax}, that is find
model parameters that minimize the ``adversarial loss'', the value of the inner
maximization problem.

In the context of training neural networks, the main method for minimizing the
loss function is Stochastic Gradient Descent (SGD). A natural way of computing
the gradient of the outer problem, $\nabla_\theta\rho(\theta)$, is computing the
gradient of the loss function at a maximizer of the inner problem. This
corresponds to replacing the input points by their corresponding adversarial
perturbations and normally training the network on the perturbed input.
A priori, it is not clear that this is a valid descent direction for the saddle
point problem. However, for the case of continuously differentiable functions,
Danskin's theorem---a classic theorem in optimization---states this is indeed
true and gradients at inner maximizers corresponds to descent directions for the
saddle point problem.

Despite the fact that the exact assumptions of Danskin's theorem do not hold for our
problem (the function is not continuously differentiable due to ReLU and
max-pooling units, and we are only computing approximate maximizers of the inner
problem), our experiments suggest that we can still use these gradients to
optimize our problem. By applying SGD using the gradient of the loss at
adversarial examples we can consistently reduce the loss of the saddle point
problem during training, as can be seen in  Figure~\ref{fig:train_loss}.
These observations suggest that we reliably optimize the saddle point
formulation~\eqref{eq:minmax} and thus train robust classifiers.
We formally state Danskin's theorem and describe how it applies to our problem in
Section~\ref{sec:app_math} of the Appendix.
  
\section{Network Capacity and Adversarial Robustness}
\label{sec:capacity}

Solving the problem from Equation~\eqref{eq:minmax} successfully is not sufficient
to guarantee robust and accurate classification.
We need to also argue that the \emph{value} of the problem (i.e. the final loss we
achieve against adversarial examples) is small, thus providing guarantees for the
performance of our classifier.  In particular, achieving a very small value corresponds to a perfect classifier, which is robust to adversarial inputs.

For a fixed set $\hood$ of possible perturbations, the value of the problem is entirely dependent on the architecture of the
classifier we are learning. Consequently, the architectural capacity of the model becomes a major
factor affecting its overall performance. At a high level, classifying examples
in a robust way requires a stronger classifier, since the presence of adversarial examples
changes the decision boundary of the problem to a more complicated one (see
Figure~\ref{fig:separation} for an illustration).

\begin{figure}[htb]
\begin{center}
\includegraphics[height=.25\textwidth]{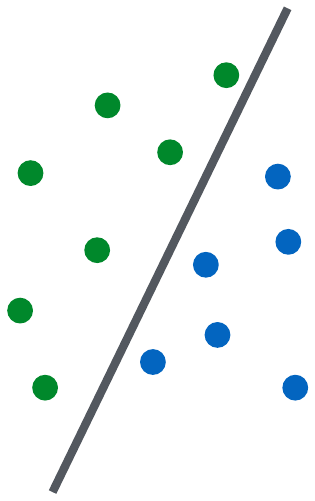}
\hspace{2cm}
\includegraphics[height=.25\textwidth]{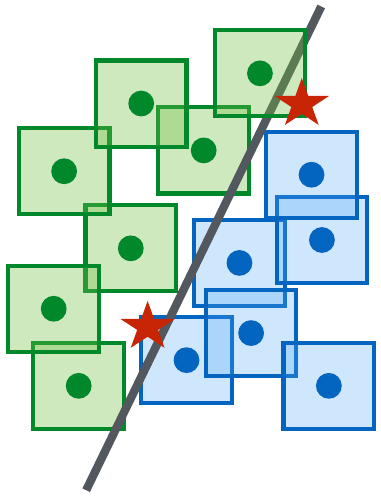}
\hspace{2cm}
\includegraphics[height=.25\textwidth]{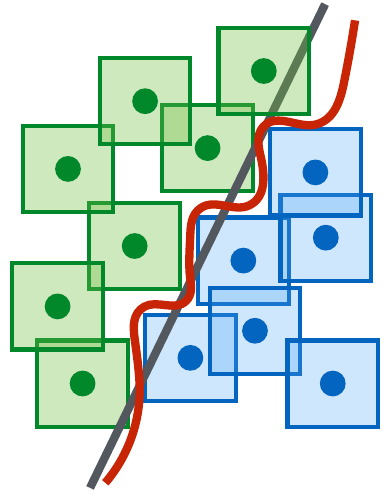}
\end{center}
\caption{A conceptual illustration of standard vs.\ adversarial decision boundaries.
Left: A set of points that can be easily separated with a simple (in this case, linear) decision boundary.
Middle: The simple decision boundary does not separate the $\ell_\infty$-balls (here, squares) around the data points. Hence there are adversarial examples (the red stars) that will be misclassified.
Right: Separating the $\ell_\infty$-balls requires a significantly more complicated decision boundary.
The resulting classifier is robust to adversarial examples with bounded $\ell_\infty$-norm perturbations.}
\label{fig:separation}
\end{figure}

Our experiments verify that capacity is crucial for robustness, as well as for the ability to successfully train against strong adversaries.  For the MNIST dataset, we consider a simple convolutional network and study how its behavior changes against different adversaries as we keep doubling the size of network (i.e. double the number of convolutional filters and the size of the fully connected layer).  The initial network has a convolutional layer with 2 filters, followed by another convolutional layer with 4 filters, and a fully connected hidden layer with 64 units. Convolutional layers are followed by $2\times 2$ max-pooling layers and adversarial examples are constructed with $\epsilon=0.3$. The results are in Figure~\ref{fig:capacity_plots}. 

For the CIFAR10 dataset, we used a ResNet model~\cite{he2016deep}. We
performed data augmentation using random crops and flips, as well as per image
standarization.  To increase the capacity, we modified the network incorporating
wider layers by a factor of 10. This results in a network with 5 residual units
with (16, 160, 320, 640) filters each. This network can achieve an accuracy of
95.2\% when trained with natural examples. Adversarial examples were constructed
with $\epsilon=8$. Results on capacity experiments appear in Figure~\ref{fig:capacity_plots}. 

We observe the following phenomena:

\paragraph{Capacity alone helps.} We observe that increasing the capacity of the network when training using only natural examples (apart from increasing accuracy on these examples) increases the robustness against one-step perturbations. This effect is greater when considering adversarial examples with smaller $\epsilon$.

\paragraph{FGSM adversaries don't increase robustness (for large $\eps$).} When
training the network using adversarial examples generated with the FGSM, we observe
that the network overfits to these adversarial examples. This behavior is known
as label leaking~\cite{kurakin2017adversarial} and stems from the fact that the adversary
produces a very restricted set of adversarial examples that the network can
overfit to.  These networks have poor performance on natural examples and don't
exhibit any kind of robustness against PGD adversaries.
For the case of smaller $\eps$ the loss is ofter linear enough in the
$\ell_\infty$-ball around natural examples, that FGSM finds adversarial examples
close to those found by PGD thus being a reasonable adversary to train against.

\paragraph{Weak models may fail to learn non-trivial classifiers.}
In the case of small capacity networks, attempting to train against a strong
adversary (PGD) prevents the network from learning anything meaningful. The
network converges to always predicting a fixed class, even though it could
converge to an accurate classifier through standard training. The small capacity of the
network forces the training procedure to sacrifice performance on natural
examples in order to provide any kind of robustness against adversarial inputs.

\paragraph{The value of the saddle point problem decreases as we increase the capacity.}
Fixing an adversary model, and training against it, the value of \eqref{eq:minmax} drops as capacity increases, indicating the the model can fit the adversarial examples increasingly well.

\paragraph{More capacity and stronger adversaries decrease transferability.}
Either increasing the capacity of the network, or using a stronger method for the inner optimization
problem reduces the effectiveness of transferred adversarial inputs.
We validate this experimentally by observing that the correlation between gradients from the source
and the transfer network, becomes less significant as capacity increases.
We describe our experiments in Section~\ref{sec:trans} of the appendix.
\begin{figure}[htb]
\begin{center}
{\setlength\tabcolsep{.06cm}
\begin{tabular}{c c c c}
 & \multicolumn{2}{c}{MNIST} &  \\
\begin{tikzpicture}[scale=0.4] 
    \begin{semilogxaxis}[
        xlabel=Capacity scale,
        ylabel=Accuracy,
        ylabel near ticks,
        xtick = data,
        log basis x = 2,
        log ticks with fixed point,
        grid = both,
        label style={font=\Huge},
        tick label style={font=\LARGE},
        every axis plot/.append style={ultra thick},
    ]
    \addplot plot coordinates {
        (1,  98.3)
        (2,  98.9)
        (4,  99.1)
        (8, 99.2)
        (16, 99.2)
    };
    \addplot plot coordinates {
        (1,  0.6)
        (2,  2.3)
        (4,  4.0)
        (8,  3.6)
        (16, 3.9)
    };
    \addplot plot coordinates {
        (1,  0.0)
        (2,  0.0)
        (4,  0.0)
        (8,  0.0)
        (16, 0.0)
    };
    \legend{}

    \end{semilogxaxis}
\end{tikzpicture} &
\begin{tikzpicture}[scale=0.4] 
    \begin{semilogxaxis}[
        xlabel=Capacity scale,
        ylabel near ticks,
        xtick = data,
        log basis x = 2,
        label style={font=\Huge},
        tick label style={font=\LARGE},
        log ticks with fixed point,
        grid = both,
        every axis plot/.append style={ultra thick},
    ]
    \addplot+[error bars/.cd,y dir=both,y explicit] plot coordinates {
        (1,  68.9) +- (0, 8.4)
        (2,  33.4) +- (0, 7.0)
        (4,  79.2) +- (0, 5.9)
        (8,  88.8) +- (0, 8.2)
        (16, 93.7) +- (0, 3.1)
    };
    \addplot+[error bars/.cd,y dir=both,y explicit] plot coordinates {
        (1,  94.8) +- (0, 3.4)
        (2,  99.5) +- (0, 0.3)
        (4,  99.8) +- (0, 0.1)
        (8,  99.3) +- (0, 0.5)
        (16, 99.6) +- (0, 0.3)
    };
    \addplot+[error bars/.cd,y dir=both,y explicit] plot coordinates {
        (1,  0.0)
        (2,  0.0)
        (4,  0.0) 
        (8,  0.0)
        (16, 0.0) 
    };
    \legend{}

    \end{semilogxaxis}
\end{tikzpicture} &
\begin{tikzpicture}[scale=0.4] 
    \begin{semilogxaxis}[
        xlabel=Capacity scale,
        ylabel near ticks,
        xtick = data,
        ymin = -10,
        label style={font=\Huge},
        tick label style={font=\LARGE},
        log ticks with fixed point,
        log basis x = 2,
        every axis plot/.append style={ultra thick},
        grid = both
    ]
    \addplot plot coordinates {
        (1,  11.4)
        (2,  11.4)
        (4,  97.6)
        (8,  98.0)
        (16, 98.9)
    };
    \addplot plot coordinates {
        (1,  11.4)
        (2,  11.4)
        (4,  89.6)
        (8, 93.9)
        (16, 95.5)
    };
    \addplot plot coordinates {
        (1,  11.4)
        (2,  11.4)
        (4,  87.7)
        (8, 91.5)
        (16, 93.3)
    };
    \legend{}

    \end{semilogxaxis}
\end{tikzpicture} &
\begin{tikzpicture}[scale=0.4] 
    \begin{loglogaxis}[
        xlabel=Capacity scale,
        ylabel=Average loss,
        ylabel near ticks,
        label style={font=\Huge},
        tick label style={font=\LARGE},
        log ticks with fixed point,
        xtick = data,
        log basis x = 2,
        grid = both,
        every axis plot/.append style={ultra thick},
        legend style = {at={(1.3,1.1)}, font=\Huge}
    ]
    \addplot plot coordinates {
        (1,  0.0537)
        (2,  0.0394)
        (4,  0.0371)
        (8,  0.0364)
        (16, 0.0351)
    };
    \addplot+[error bars/.cd,y dir=both,y explicit] plot coordinates {
        (1,  0.1647) +- (0, 0.1039)
        (2,  0.0175) +- (0, 0.0095)
        (4,  0.0067) +- (0, 0.0016)
        (8,  0.0206) +- (0, 0.0147)
        (16, 0.0110) +- (0, 0.0088)
    };
    \addplot plot coordinates {
        (1,  2.3011)
        (2,  2.3012)
        (4,  0.3650)
        (8,  0.2504)
        (16, 0.1770)
    };
    \legend{Natural\\ FGSM\\ PGD\\}

    \end{loglogaxis}
\end{tikzpicture} \\

 & \multicolumn{2}{c}{CIFAR10} &  \\
 
 \begin{tabular}{cc|c}
  & Simple & Wide \\ \hline
	Natural\quad & 92.7\% & 95.2\% \\
	FGSM\quad  & 27.5\% & 32.7\% \\
	PGD\quad  & 0.8\% & 3.5\% \\
\end{tabular}
 &
  \begin{tabular}{c|c}
  Simple & Wide \\ \hline
	 87.4\% & 90.3\% \\
	 90.9\% & 95.1\% \\
	 0.0\% & 0.0\% \\
\end{tabular}
 &
  \begin{tabular}{c|c}
   Simple & Wide \\ \hline
	 79.4\% & 87.3\% \\
	 51.7\% & 56.1\% \\
	 43.7\% & 45.8\% \\
\end{tabular}
 &
   \begin{tabular}{c|c}
   Simple & Wide \\ \hline
	 0.00357 & 0.00371 \\
	 0.0115 & 0.00557 \\
	 1.11 & 0.0218 \\
\end{tabular}
\\	
(a) Standard training & (b) FGSM training & (c) PGD training & (d) Training Loss

\end{tabular}}
\end{center}
\caption{The effect of network capacity on the performance of the network. We
trained MNIST and CIFAR10 networks of varying capacity on: (a) natural examples, (b)
with FGSM-made adversarial examples, (c) with PGD-made adversarial examples. In the first three plots/tables of each
dataset, we show how the standard and adversarial accuracy changes with respect
to capacity for each training regime. In the final plot/table, we show the value
of the cross-entropy loss on the adversarial examples the networks were trained
on. This corresponds to the value of our saddle point formulation
(\ref{eq:minmax}) for different sets of allowed perturbations.}
\label{fig:capacity_plots}
\end{figure}
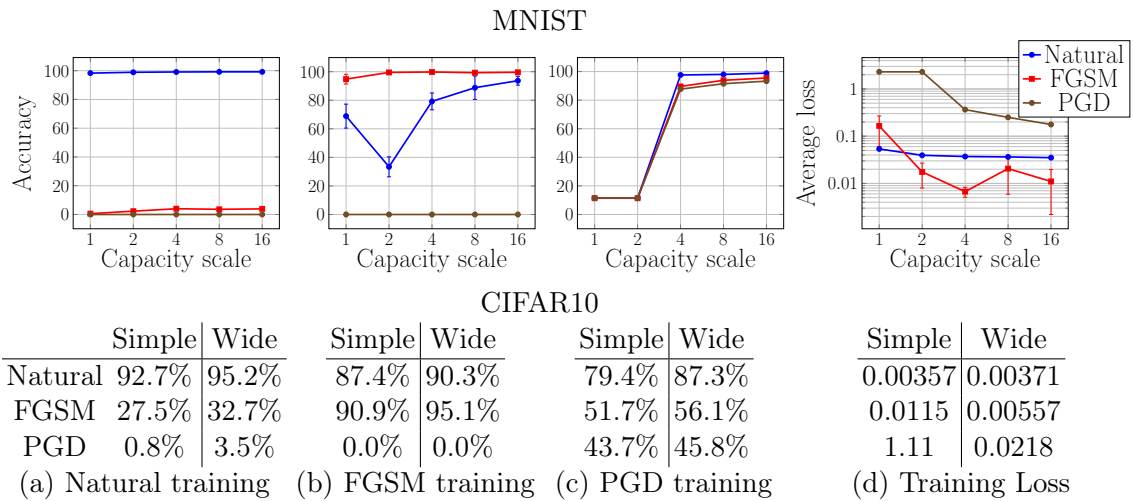

\section{Experiments: Adversarially Robust Deep Learning Models}
\label{sec:experiments}

Following the understanding of the problem we developed in previous sections, we
can now apply our proposed approach to train robust classifiers.  As our experiments
so far demonstrated, we need to focus on two key elements: a)
train a sufficiently high capacity network, b) use the strongest possible
adversary.

For both MNIST and CIFAR10, the adversary of choice will be projected gradient
descent (PGD)
starting from a random perturbation around the natural example. This corresponds to our
notion of a "complete" first-order adversary, an
algorithm that can efficiently maximize the loss of an example using only first order
information. Since we are training the model for multiple epochs, there is no
benefit from restarting PGD multiple times per batch---a new start will be chosen the
next time each example is encountered.

When training against that adversary, we observe a steady decrease in the
training loss of adversarial examples, illustrated in
Figure~\ref{fig:train_loss}. This behavior indicates that we are indeed
successfully solving our original optimization problem during training.

\begin{figure}[htb]
\begin{center}
{\setlength\tabcolsep{.06cm}
\begin{tabular}{c c}
\includegraphics[height=.25\textwidth]{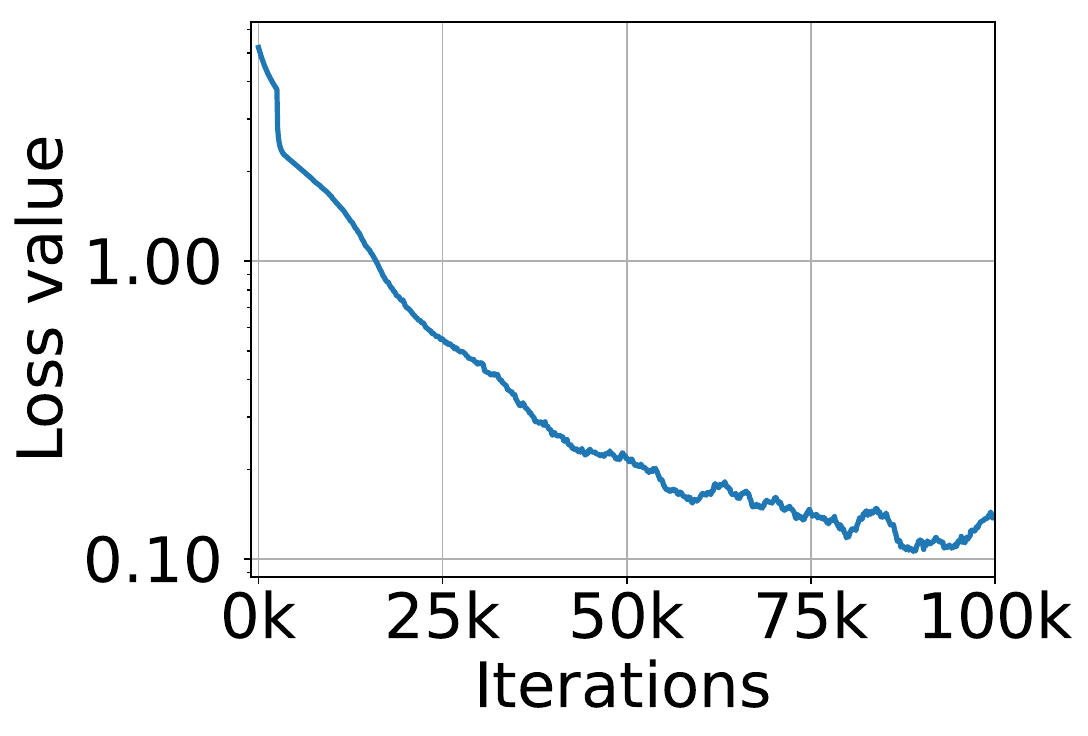} &
\includegraphics[height=.25\textwidth]{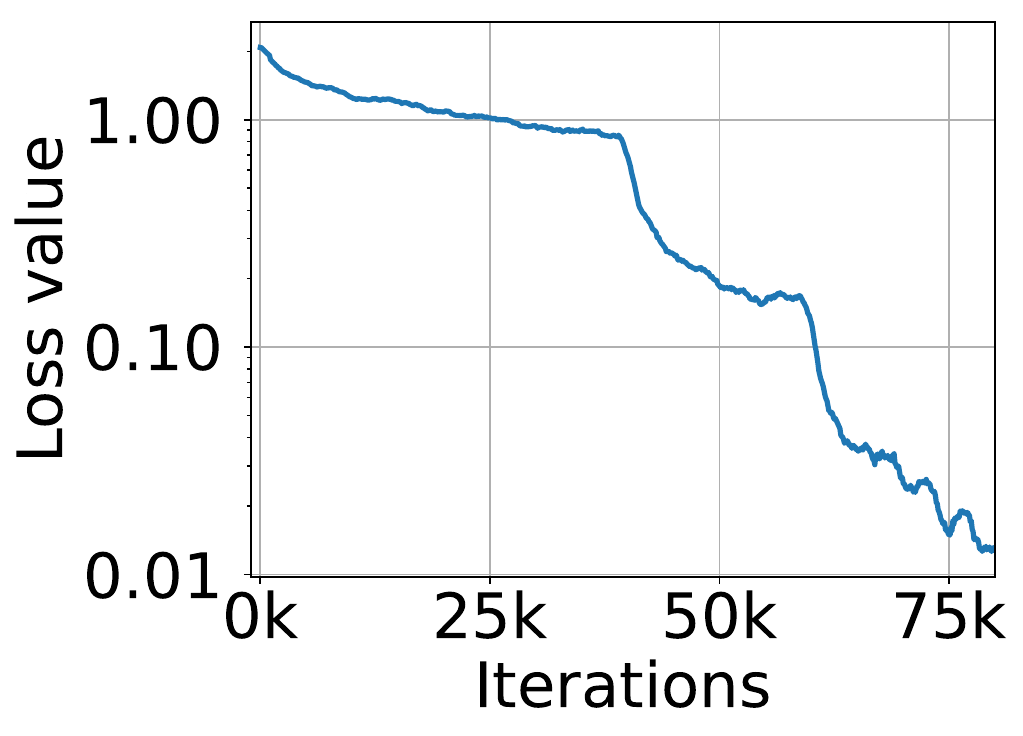}\\
(a) MNIST & (b) CIFAR10
\end{tabular}}
\end{center}
\caption{Cross-entropy loss on adversarial examples during training. The plots
show how the adversarial loss on training examples evolves during training the MNIST and
CIFAR10 networks against a PGD adversary.  The sharp drops in the CIFAR10 plot
correspond to decreases in training step size.
These plots illustrate that we can consistently reduce the value of the inner
problem of the saddle point formulation (\ref{eq:minmax}), thus producing an
increasingly robust classifier.
}
\label{fig:train_loss}
\end{figure}

We evaluate the trained models against a range of adversaries. We illustrate our
results in Table~\ref{fig:super_mnist} for MNIST and Table~\ref{fig:super-cifar}
for CIFAR10. 
The adversaries we consider are:
\begin{itemize}
\item White-box attacks with PGD for a different number of of iterations and
restarts, denoted by source A.
\item White-box attacks with PGD using the Carlini-Wagner (CW) loss
function~\cite{carlini2017towards} (directly optimizing the difference
between correct and incorrect logits).
This is denoted as CW, where the
corresponding attack with a high confidence parameter ($\kappa=50$) is denoted
as CW+.
\item Black-box attacks from an independently trained copy of the network,
denoted A'.
\item Black-box attacks from a version of the same network trained only on
natural examples, denoted $A_{nat}$.
\item Black-box attacks from a different convolution architecture, denoted B,
described in Tramer et al. 2017~\cite{tramer2017space}.
\end{itemize}

\paragraph{MNIST.} We run
40 iterations of projected gradient descent as our adversary, with a step size of
$0.01$ (we choose to take gradient steps in the $\ell_\infty$-norm, i.e. adding
the sign of the gradient, since this makes the choice of the step size simpler).
We train and evaluate against perturbations of size $\eps = 0.3$.
We use a network consisting of two convolutional
layers with 32 and 64 filters respectively, each followed by $2\times 2$
max-pooling, and a fully connected layer of size 1024.
When trained with natural examples, this network reaches 99.2\% accuracy on the
evaluation set. However, when evaluating on examples perturbed with FGSM the
accuracy drops to 6.4\%.
The resulting adversarial accuracies are reported in
Table~\ref{fig:super_mnist}.
Given that the resulting MNIST model is very robust to $\ell_\infty$-bounded
adversaries, we investigated the learned parameters in order to understand how they affect adversarial robustness. The results of the investigation are presented in Appendix~\ref{app:inspection}.
In particular, we found that the first convolutional layer of the network is
learning to threshold input pixels while other weights tend to be sparser.

\begin{table}[htb]
\begin{center}
\setlength{\tabcolsep}{0.2em}
\begin{tabular}{| p{2cm} | p{1.5cm} | p{1.5cm} | p{1.5cm} || r |}
\hline
Method     & Steps & Restarts & Source & Accuracy \\
\hline
\hline
Natural    & -       & -        & -      &  98.8\% \\ 
\hline
FGSM       & -       & -        & A      &  95.6\% \\
\hline
PGD        & 40      & 1        & A      &  93.2\% \\
\hline
PGD        & 100     & 1        & A      &  91.8\% \\
\hline
PGD        & 40      & 20       & A      &  90.4\% \\
\hline
PGD        & 100     & 20       & A      &  {\bf 89.3\%} \\
\hline
Targeted   & 40      & 1        & A      &  92.7\% \\
\hline
CW         & 40      & 1        & A      & 94.0\%\\
\hline
CW+        & 40      & 1        & A      & 93.9\%\\
\hline
\hline
FGSM       & -       & -        & A'     &  96.8\% \\
\hline
PGD        & 40      & 1        & A'     &  96.0\% \\
\hline
PGD        & 100     & 20       & A'     &  {\bf 95.7\%} \\
\hline
CW         & 40      & 1        & A'     & 97.0\%\\
\hline
CW+        & 40      & 1        & A'     & 96.4\%\\
\hline
\hline
FGSM       & -       & -        & B      &  {\bf 95.4\%} \\
\hline
PGD        & 40      & 1        & B      &  96.4\% \\
\hline
CW+        & -       & -        & B      &  95.7\% \\
\hline
\end{tabular}
\end{center}
\caption{MNIST: Performance of the adversarially trained network against
different adversaries for $\eps=0.3$. For each model of attack we show the most
successful attack with bold.  The source networks used for the attack are: the
network itself (A) (white-box attack), an indepentenly initialized and trained
copy of the network (A'), architecture B from \cite{tramer2017space} (B).}
\label{fig:super_mnist}
\end{table}

\paragraph{CIFAR10.} For the CIFAR10 dataset, we use the two architectures described in
\ref{sec:capacity} (the original ResNet and its $10\times$ wider variant). We trained
the network against a PGD adversary with $\ell_\infty$ projected gradient descent again, this time using 7 steps of size 2, and a total $\epsilon=8$. For our hardest adversary we chose 20 steps with the same settings, since other hyperparameter choices didn't offer a significant decrease in accuracy. The results of our experiments appear in Table~\ref{fig:super-cifar}.

\begin{table}
\begin{center}
\setlength{\tabcolsep}{0.2em}
\begin{tabular}{| p{2cm} | p{1.5cm} | p{1.5cm} || r |}
\hline
Method    & Steps & Source & Accuracy \\
\hline
\hline
Natural    & -       & -      &  87.3\% \\ 
\hline
FGSM       & -       & A      &  56.1\% \\
\hline
PGD        & 7       & A      &  50.0\% \\
\hline
PGD        & 20      & A      &  {\bf 45.8}\% \\
\hline
CW         & 30      & A      &  46.8\% \\
\hline
\hline
FGSM       & -       & A'     &  67.0\% \\
\hline
PGD        & 7       & A'     &  {\bf 64.2\%} \\
\hline
CW         & 30      & A'     &  78.7\% \\
\hline
FGSM       & -       & $\textnormal{A}_{nat}$ & 85.6\% \\
\hline
PGD        & 7     & $\textnormal{A}_{nat}$ & 86.0\% \\
\hline
\end{tabular}
\end{center}
\caption{CIFAR10: Performance of the adversarially trained network against
different adversaries for $\eps=8$. For each model of attack we show the most
effective attack in bold.  The source networks considered for the attack are:
the network itself (A) (white-box attack), an independtly initialized and
trained copy of the network (A'), a copy of the network trained on natural
examples
($\textnormal{A}_{nat}$).}
\label{fig:super-cifar}
\end{table}

The adversarial robustness of our network is significant, given the power of
iterative adversaries, but still far from satisfactory. We believe that these
results can be improved by further pushing along these directions, and training
networks of larger capacity.

\paragraph{Resistance for different values of $\eps$ and $\ell_2$-bounded attacks.}
In order to perform a broader evaluation of the adversarial robustness of our
models, we run two additional experiments. On one hand, we investigate
the resistance to $\ell_\infty$-bounded attacks for different values of $\eps$.
On the other hand, we examine the resistance of our model to attacks that are
bounded in $\ell_2$-norm as opposed to $\ell_\infty$-norm.
In the case of $\ell_2$-bounded PGD we take steps in the gradient direction (not the
sign of it) and normalize the steps to be of fixed norm to facilitate step size tuning.
For all PGD attacks, we use $100$ steps and set the step size to be
$2.5\cdot\epsilon / 100$ to ensure that we can reach the boundary of the
$\epsilon$-ball from any starting point within it (and still allow for movement
on the boundary).
Note that the models were training against $\ell_\infty$-bounded attacks with
the original value of $\eps=0.3$, for MNIST, and $\eps=8$ for CIFAR10.
The results appear in Figure~\ref{fig:acc_vs_eps}.

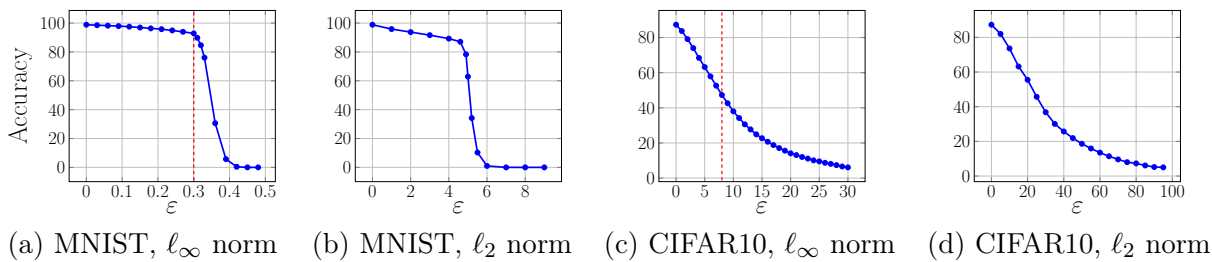
\begin{figure}
\begin{center}
\begin{tabular}{cccc}
\begin{tikzpicture}[scale=0.4]
    \begin{axis}[
        xlabel=$\eps$,
        ylabel=Accuracy,
        ylabel near ticks,
        grid = both,
        label style={font=\Huge},
        tick label style={font=\LARGE},
        every axis plot/.append style={ultra thick},
    ]
    \addplot plot coordinates {
        (0.00, 98.53)
        (0.03, 98.23)
        (0.06, 97.76)
        (0.09, 97.40)
        (0.12, 97.05)
        (0.15, 96.51)
        (0.18, 96.07)
        (0.21, 95.67)
        (0.24, 95.01)
        (0.27, 94.26)
        (0.30, 93.22)
        (0.33, 75.16)
        (0.36, 28.16)
        (0.39,  4.14)
        (0.42,  0.22)
    };
    \addlegendentry{PGD adv};
    \addplot plot coordinates {
        (0.00, 98.53)
        (0.03, 98.53)
        (0.06, 98.53)
        (0.09, 98.50)
        (0.12, 98.40)
        (0.15, 98.32)
        (0.18, 98.24)
        (0.21, 98.08)
        (0.24, 98.08)
        (0.27, 98.08)
        (0.30, 98.00)
        (0.33, 97.68)
        (0.36, 97.04)
        (0.39, 96.16)
        (0.42, 95.36)
    };
    \addlegendentry{DBA adv};
    \addplot plot coordinates {
        (0.00, 99.17)
        (0.03, 94.08)
        (0.06, 69.40)
        (0.09, 22.92)
        (0.12,  2.48)
        (0.15,  0.12)
        (0.18,  0.00)
        (0.21,  0.01)
        (0.24,  0.00)
        (0.27,  0.00)
        (0.30,  0.00)
        (0.33,  0.00)
        (0.36,  0.00)
        (0.39,  0.00)
        (0.42,  0.00)
    };
    \addlegendentry{PGD nat};
    \addplot plot coordinates {
        (0.00, 99.17)
        (0.03, 77.97)
        (0.06, 68.49)
        (0.09, 60.26)
        (0.12, 49.91)
        (0.15, 36.97)
        (0.18, 24.71)
        (0.21, 15.37)
        (0.24,  9.03)
        (0.27,  5.26)
        (0.30,  3.00)
        (0.33,  1.57)
        (0.36,  0.77)
        (0.39,  0.43)
        (0.42,  0.17)
    };
    \addlegendentry{DBA nat};
    \legend{};
    \draw [red, dashed,very thick] (300,-90) -- (300,1090);

    \end{axis}
\end{tikzpicture} &

\begin{tikzpicture}[scale=0.4]
    \begin{axis}[
        xlabel=$\eps$,
        ylabel near ticks,
        grid = both,
        label style={font=\Huge},
        tick label style={font=\LARGE},
        every axis plot/.append style={ultra thick},
        legend style={at={(1.1, 1.1)},anchor=north east}
    ]
    \addplot plot coordinates {
        (0.0, 98.53)
        (0.5, 96.01)
        (1.0, 94.25)
        (1.5, 92.25)
        (2.0, 89.49)
        (2.5, 86.39)
        (3.0, 82.34)
        (3.5, 73.66)
        (4.0, 56.57)
        (4.5, 33.05)
        (5.0, 14.87)
        (5.5,  5.34)
        (6.0,  1.68)
    };
    \addlegendentry{PGD adv. trained};
    \addplot plot coordinates {
        (0.0, 98.53)
        (0.5, 96.03)
        (1.0, 83.10)
        (1.5, 43.12)
        (2.0, 10.65)
        (2.5,  1.22)
        (3.0,  0.05)
        (3.5,  0.00)
        (4.0,  0.00)
        (4.5,  0.00)
        (5.0,  0.00)
        (5.5,  0.00)
        (6.0,  0.00)
    };
    \addlegendentry{DBA adv. trained};
    \addplot plot coordinates {
        (0.0, 99.17)
        (0.5, 93.47)
        (1.0, 64.14)
        (1.5, 27.74)
        (2.0, 18.47)
        (2.5, 15.15)
        (3.0, 12.22)
        (3.5,  9.78)
        (4.0,  7.11)
        (4.5,  6.03)
        (5.0,  4.41)
        (5.5,  3.11)
        (6.0,  2.31)
    };
    \addlegendentry{PGD standard};
    \addplot plot coordinates {
        (0.0, 99.17)
        (0.5, 65.40)
        (1.0, 43.49)
        (1.5, 18.31)
        (2.0,  5.43)
        (2.5,  1.57)
        (3.0,  0.14)
        (3.5,  0.03)
        (4.0,  0.00)
        (4.5,  0.00)
        (5.0,  0.00)
        (5.5,  0.00)
        (6.0,  0.00)
    };
    \addlegendentry{DBA standard};

    \end{axis}
\end{tikzpicture} &
\begin{tikzpicture}[scale=0.4]
    \begin{axis}[
        xlabel=$\eps$,
        ylabel near ticks,
        grid = both,
        label style={font=\Huge},
        tick label style={font=\LARGE},
        every axis plot/.append style={ultra thick},
    ]
    \addplot plot coordinates {
(0.0, 87.3)
(1.0, 83.7)
(2.0, 79.10000000000001)
(3.0, 74.0)
(4.0, 68.4)
(5.0, 63.2)
(6.0, 57.9)
(7.0, 52.6)
(8.0, 47.3)
(9.0, 42.6)
(10.0, 38.0)
(11.0, 34.2)
(12.0, 30.599999999999998)
(13.0, 27.700000000000003)
(14.0, 24.9)
(15.0, 22.7)
(16.0, 20.7)
(17.0, 18.8)
(18.0, 17.1)
(19.0, 15.6)
(20.0, 14.099999999999998)
(21.0, 13.100000000000001)
(22.0, 12.0)
(23.0, 11.1)
(24.0, 10.100000000000001)
(25.0, 9.5)
(26.0, 8.7)
(27.0, 8.1)
(28.0, 7.3999999999999995)
(29.0, 6.6000000000000005)
(30.0, 6.1)

    };
    \draw [red, dashed,very thick] (80,-110) -- (80,1100);
    \legend{}

    \end{axis}
\end{tikzpicture} &

\begin{tikzpicture}[scale=0.4]
    \begin{axis}[
        xlabel=$\eps$,
        ylabel near ticks,
        grid = both,
        label style={font=\Huge},
        tick label style={font=\LARGE},
        every axis plot/.append style={ultra thick},
    ]
    \addplot plot coordinates {
(0.0, 87.3)
(5.0, 82.0)
(10.0, 73.6)
(15.0, 63.2)
(20.0, 55.5)
(25.0, 45.7)
(30.0, 36.8)
(35.0, 30.1)
(40.0, 25.7)
(45.0, 21.9)
(50.0, 18.6)
(55.0, 15.9)
(60.0, 13.5)
(65.0, 11.4)
(70.0, 9.6)
(75.0, 8.0)
(80.0, 7.3)
(85.0, 6.1)
(90.0, 5.2)
(95.0, 5.0)

    };
    \legend{}

    \end{axis}
\end{tikzpicture}

\\
(a) MNIST, $\ell_\infty$-norm& (b) MNIST, $\ell_2$-norm&
(c) CIFAR10, $\ell_\infty$-norm & (d) CIFAR10, $\ell_2$-norm
\end{tabular}
\end{center}
\caption{Performance of our adversarially trained networks against PGD
adversaries of different strength.
The MNIST and CIFAR10 networks were trained
against $\eps=0.3$ and $\eps=8$ PGD $\ell_\infty$ adversaries respectively (the
training $\eps$ is denoted with a red dashed lines in the $\ell_\infty$ plots).
In the case of the MNIST adversarially trained networks, we also evaluate the
performance of the Decision Boundary Attack (DBA)~\cite{brendel2017decision}
with 2000 steps and
PGD on standard and adversarially trained models.
We observe that for $\eps$ less or equal to the value used during training, the
performance is equal or better. For MNIST there is a sharp drop shortly after.
Moreover, we observe that the performance of PGD on the MNIST $\ell_2$-trained
networks is poor and significantly overestimates the robustness of the model.
This is potentially due to the threshold filters learned by the model masking
the loss gradients (the decision-based attack does not utilize gradients).}
\label{fig:acc_vs_eps}
\end{figure}

We observe that for smaller $\eps$ than the one used during training the models
achieve equal or higher accuracy, as expected.
For MNIST, we notice a large drop in robustness for slightly large $\eps$
values, potentially due to the fact that the threshold operators learned are
tuned to the exact value of $\eps$ used during training
(Appendix~\ref{app:inspection}).
In contrast, the decay for the case of CIFAR10 is smoother.

For the case of $\ell_2$-bounded attacks on MNIST, we observe that PGD is unable
to find adversarial examples even for quite large values of $\eps$, e.g.,
$\eps=4.5$.
To put this value of $\eps$ into perspective, we provide a sample of corresponding adversarial examples in Figure~\ref{fig:l2_mnist_examples} of Appendix~\ref{sec:omitted}.
We observe that these perturbations are significant enough that they would
change the ground-truth label of the images and it is thus unlikely that our
models are actually that robust.
Indeed, subsequent work~\cite{li2018second,schott2018towards} has found that PGD
is in fact overestimating the $\ell_2$-robustness of this model.
This behavior is potentially due to the fact that the learned threshold filters
(Appendix~\ref{app:inspection}) mask the gradient, preventing PGD from
maximizing the loss.
Attacking the model with a decision-based attack~\cite{brendel2017decision}
which does not rely on model gradients reveals that the model is significantly
more brittle against $\ell_2$-bounded attacks.
Nevertheless, the $\ell_\infty$-trained model is still more robust to $\ell_2$
attacks compared to a standard model.
 
\section{Related Work}
Due to the large body of work on adversarial
examples we focus only on the most related papers here.
Before we compare our contributions, we remark that robust optimization has been
studied outside deep learning for multiple decades (see~\cite{ben2009robust} for
an overview of this field).
We also want to note that the study of adversarial ML predates the widespread
use of deep neural networks~\cite{dalvi2004adversarial,globerson2006nightmare}
(see~\cite{biggio2018wild} for an overview of earlier work).

Adversarial training was introduced in~\cite{goodfellow2015explaining}, however
the adversary utilized was quite weak---it relied on linearizing the loss around
the data points. As a result, while these models were robust against this
particular adversary, they were completely vulnerable to slightly more
sophisticated adversaries utilizing iterative attacks.

Recent work on adversarial training on ImageNet also observed that the model
capacity is important for adversarial training \cite{kurakin2017adversarial}.
In contrast to this paper, we find that training against multi-step methods
(PGD) \emph{does} lead to resistance against such adversaries.

In~\cite{huang2015learning} and~\cite{shaham2018understanding} a version of the min-max optimization problem is also considered for adversarial training.
There are, however, three important differences between the formerly mentioned result and the present paper.
Firstly, the authors claim that the inner maximization problem can be difficult to solve, whereas we explore the loss surface in more detail and find that randomly re-started projected gradient descent often converges to solutions with comparable quality. This shows that it is possible to obtain sufficiently good solutions to the inner maximization problem, which offers good evidence that deep neural network can be immunized against adversarial examples.
Secondly, they consider only one-step adversaries, while we work with multi-step methods.
Additionally, while the experiments in \cite{shaham2018understanding} produce promising results, they are only evaluated against FGSM. However, FGSM-only evaluations are not fully reliable. One evidence for that is that \cite{shaham2018understanding} reports 70\% accuracy for $\eps=0.7$, but any
adversary that is allowed to perturb each pixel by more than 0.5 can construct a
uniformly gray image, thus fooling any classifier.

A more recent paper~\cite{tramer2017space} also explores the transferability
phenomenon. This exploration focuses mostly on the region around natural
examples where the loss is (close to) linear. When large perturbations are 
allowed, this region does not give a complete picture of the adversarial
landscape. This is confirmed by our experiments, as well as pointed out by
\cite{tramer2017space}.
  

\section{Conclusion}

Our findings provide evidence that deep neural networks can be made resistant to adversarial attacks. As our theory and experiments indicate, we can design reliable adversarial training methods. One of the key insights behind this is the unexpectedly regular structure of the underlying optimization task: even though the relevant problem corresponds to the maximization of a highly non-concave function with many distinct local maxima, their \emph{values} are highly concentrated.
Overall, our findings give us hope that adversarially robust deep learning
models may be within current reach. 

For the MNIST dataset, our networks are very robust, achieving high accuracy for
a wide range of powerful $\ell_\infty$-bound adversaries and large perturbations. Our experiments on CIFAR10
have not reached the same level of performance yet. However, our results already show
that our techniques lead to significant increase in the robustness of the
network.  We believe that further exploring this direction will lead to
adversarially robust networks for this dataset.

\section*{Acknowledgments}
Aleksander M\k{a}dry, Aleksandar Makelov, and Dimitris Tsipras were supported by the NSF Grant No. 1553428, a Google Research Fellowship, and a Sloan Research Fellowship. Ludwig Schmidt was supported by a Google PhD Fellowship. Adrian Vladu was supported by the NSF Grants No. 1111109 and No. 1553428.

We thank Wojciech Matusik for kindly providing us with computing resources to perform this work.

\bibliographystyle{plain}
\bibliography{bibliography/bib.bib}
\appendix
\section{Statement and Application of Danskin's Theorem}\label{sec:app_math}
Recall that our goal is to minimize the value of the saddle point problem
\begin{equation*}
	\min_\theta \rho(\theta),\quad \text{ where }\quad \rho(\theta) =
    \mathbb{E}_{(x,y)\sim\mathcal{D}}\left[\max_{\delta\in 
    \hood}
    \loss(\theta,x+\delta,y)\right] \; .
\end{equation*}
In practice, we don't have access to the distribution $\mathcal{D}$ so both the
gradients and the value of $\rho(\theta)$ will be computed using sampled input
points.  Therefore we can consider --without loss of generality-- the case of a
single random example $x$ with label $y$, in which case the problem becomes
\begin{equation*}
    \min_\theta \max_{\delta\in \hood} g(\theta,\delta) ,\quad
    \text{ where }\quad g(\theta,\delta) = \loss(\theta, x+\delta,y) \; .
\end{equation*}

If we assume that the loss $\loss$ is continuously differentiable in $\theta$, we
can compute a descent direction for $\theta$ by utilizing the
classical theorem of Danskin.

\begin{theorem}[Danskin]\label{thm:danskin}
Let $\hood$ be  nonempty compact topological space and $g:\mathbb{R}^n \times \hood \rightarrow \mathbb{R}$ be such that $g(\cdot, \delta)$ is differentiable for every $\delta \in \hood$ and $\nabla_\theta g(\theta,\delta)$ is continuous on $\mathbb{R}^n \times \hood$.  Also, let $\delta^*(\theta) = \{\delta \in \arg \max_{\delta \in \hood}  g(\theta,\delta)\}$.

Then the corresponding max-function $$\phi(\theta) = \max_{\delta \in \hood} g(\theta,\delta)$$ is locally Lipschitz continuous, directionally differentiable, and its directional derivatives satisfy
$$\phi'(\theta, h) = \sup_{\delta \in \delta^*(\theta)} h^\top \nabla_\theta g(\theta, \delta)\,{.}$$
In particular, if for some $\theta \in\mathbb{R}^n$ the set $\delta^*(\theta) = \{ \delta^*_\theta \}$ is a singleton, the the max-function is differentiable at $\theta$ and 
$$\nabla \phi(\theta) = \nabla_\theta g(\theta, \delta^*_\theta)$$
\end{theorem}

The intuition behind the theorem is that since gradients are local objects, and
the function $\phi(\theta)$ is locally the same as $g(\theta, \delta^*_\theta)$
their gradients will be the same. The theorem immediately gives us the following
corollary, stating the we can indeed compute gradients for the saddle point by
computing gradients at the inner optimizers.
\begin{corollary}\label{cor:descentdir}
Let $\bar{\delta}$ be such that $ \bar{\delta} \in \hood$ and is a maximizer for
$\max_\delta L(\theta,x+\delta,y)$. Then, as long as it is nonzero, $-\nabla_\theta \loss(\theta, x+\bar{\delta}, y)$ is a descent direction for $\phi(\theta) = \max_{\delta \in \hood} \loss(\theta, x+\delta, y)$.
\end{corollary}
\begin{proof}[Proof of Corollary~\ref{cor:descentdir}]
We apply Theorem~\ref{thm:danskin} to $g(\theta, \delta) := \loss(\theta, x+\delta, y)$ and $\hood = B_{\|\cdot\|}(\epsilon)$. We see that the directional derivative in the direction of $h = \nabla_\theta \loss(\theta, x+\bar{\delta}, y)$ satisfies
\begin{align*}
\phi'(\theta, h) &= \sup_{\delta \in \delta^*(\theta)} h^\top \nabla_\theta \loss(\theta, x+\delta, y) 
\geq h^\top h = \Vert \nabla_\theta \loss(\theta, x+\bar{\delta}, y) \Vert_2^2 \geq 0\,{.}
\end{align*}
If this gradient is nonzero, then the inequality above is strict. Therefore it gives a descent direction.
\end{proof}

A technical issue is that, since we use ReLU and max-pooling units in our neural
network architecture, the loss function is not continuously differentiable.
Nevertheless, since the set of discontinuities has measure zero, we can assume
that this will not be an issue in practice, as we will never encounter the
problematic points.

Another technical issue is that, due to the not concavity of the inner problem,
we are not able to compute global maximizers, since PGD will converge to local
maxima. In such cases, we can consider a subset $\hood'$ of $\hood$ such that
the local maximum is a global maximum in the region $\hood'$.
Applying the theorem for $\hood'$ gives us that the gradient corresponds to a
descent direction for the saddle point problem when the adversary is constrained
in $\hood'$. Therefore if the inner maximum is a true adversarial example for
the network, then SGD using the gradient at that point will decrease the loss
value at this particular adversarial examples, thus making progress towards a
robust model.

These arguments suggest that the conclusions of the theorem are still valid in
our saddle point problem, and --as our experiments confirm-- we can solve it
reliably.

\section{Transferability}\label{sec:trans}

A lot of recent literature on adversarial training discusses the phenomenon of
transferability~\cite{szegedy2014intriguing,goodfellow2015explaining,tramer2017space}---adversarial examples
transfer between independently trained networks. This raises concerns for practical
applications, since it suggests that deep networks are vulnerable to attacks,
\emph{even when there is no direct access to the target network}.

This phenomenon is further confirmed by our current experiments.~\footnote{Our experiments involve transferability between  networks with the same architecture (potentially with layers of varying sizes), trained with the same method, but with different random initializations. The reason we consider these models rather than highly different architectures is that they are likely the worst case instances for transferability.} Moreover, we notice that the extent to which adversarial examples transfer decreases as we increase either network capacity or the power of the adversary used for training the network. This serves as evidence for the fact that the transferability phenomenon can be alleviated by using high capacity networks in conjunction with strong oracles for the inner optimization problem. 

\paragraph{MNIST.} 
In an attempt to understand these phenomena we inspect the loss functions corresponding to the trained models we used for testing transferability. More precisely, we compute angles between gradients of the loss functions evaluated over a large set of input examples, and plot their distribution. Similarly, we plot the value of the loss functions between clean and perturbed examples for both the source and transfer networks.
In Figure~\ref{fig:transfer} we plot our experimental findings on the MNIST
dataset for $\epsilon = 0.3$. We consider a large standard network (two convolutional layers of sizes $32$ and $64$, and a fully connected layer of size $1024$), which we train twice starting with different initializations. 
We plot the distribution of angles between gradients for the same test image in the two resulting networks ({\color{orange} orange} histograms), noting that they are somewhat correlated. As opposed to this, we see that pairs of gradients for random pairs of inputs for one architecture are as uncorrelated as they can be ({\color{blue} blue} histograms), since the distribution of their angles looks Gaussian.

Next, we run the same experiment on a very large standard network (two convolutional layers of sizes $64$ and $128$, and a fully connected layer of size $1024$). We notice a mild increase in classification accuracy for transferred examples.

Finally, we repeat the same set of experiments, after training the large and very large networks against the FGSM adversary.
We notice that gradients between the two architectures become significantly less
correlated. Also, the classification accuracy for transferred examples increases
significantly compared to the standard networks.

We further plot how the value of the loss function changes when moving from the natural input towards the adversarially perturbed input (in Figure~\ref{fig:transfer} we show these plots for four images in the MNIST test dataset), for each pair of networks we considered. We observe that, while for the naturally trained networks, when moving towards the perturbed point, the value of the loss function on the transfer architecture tends to start increasing  soon after it starts increasing on the source architecture. In contrast, for the stronger models, the loss function on the transfer network tends to start increasing later, and less aggressively.

\paragraph{CIFAR10.}
\label{par:}

For the CIFAR10 dataset, we investigate the transferability of the FGSM and PGD adversaries between our simple and wide architectures, each trained on natural, FGSM and PGD examples. Transfer accuracies for the FGSM adversary and PGD adversary between all pairs of such configurations (model + training method) with independently random weight initialization are given in tables \ref{cifar10-fgsm-transfer} and \ref{cifar10-pgd-transfer} respectively. The results exhibit the following trends:
\begin{itemize}
	\item \textbf{Stronger adversaries decrease transferability:} In particular,
        transfer attacks between two PGD-trained models are less successful than
        transfer attacks between their standard counterparts. Moreover, adding PGD training helps with transferability from all adversarial datasets, \emph{except} for those with source a PGD-trained model themselves. This applies to both FGSM attacks and PGD attacks. 
	\item \textbf{Capacity decreases transferability:} In particular, transfer attacks between two PGD-trained wide networks are less successful than transfer attacks between their simple PGD-trained counterparts. Moreover, with few close exceptions, changing the architecture from simple to wide (and keeping the training method the same) helps with transferability from all adversarial datasets. 
\end{itemize}

We additionally plotted how the loss of a network behaves in the direction of
FGSM and PGD examples obtained from itself and an independently trained copy;
results for the simple standard network and the wide PGD trained network are given in Table \ref{fig:cifar_transfer_loss_plots}. As expected, we observe the following phenomena:
\begin{itemize}
	\item sometimes, the FGSM adversary manages to increase loss faster near the natural example, but as we move towards the boundary of the $\ell_\infty$ box of radius $\epsilon$, the PGD attack always achieves higher loss. 
	\item the transferred attacks do worse than their white-box counterparts in terms of increasing the loss;
	\item and yet, the transferred PGD attacks dominate the white-box FGSM
        attacks for the standard network (and sometimes for the PGD-trained one too). 
\end{itemize}

\begin{table}[htb]
\centering
\setlength{\tabcolsep}{0.2em}
\begin{tabular}{|l|l|l|l|l|l|l|}
\hline
            \backslashbox[3.1cm]{Target}{Source}
            &\begin{tabular}{@{}c@{}}Simple \\ (standard\\ training)\end{tabular} & \begin{tabular}{@{}c@{}}Simple \\ (FGSM\\ training)\end{tabular} & \begin{tabular}{@{}c@{}}Simple \\ (PGD\\ training)\end{tabular}  & \begin{tabular}{@{}c@{}}Wide \\ (natural\\ training)\end{tabular} &\begin{tabular}{@{}c@{}}Wide \\ (FGSM\\ training)\end{tabular} & \begin{tabular}{@{}c@{}}Wide \\ (PGD\\ training)\end{tabular} \\ \hline
\begin{tabular}{@{}c@{}}Simple \\ (standard training)\end{tabular} & 32.9\%                                   & 74.0\%                                & 73.7\%                               & 27.6\%                                 & 71.8\%                              & 76.6\%                      \\ \hline
\begin{tabular}{@{}c@{}}Simple \\ (FGSM training)\end{tabular}    & 64.2\%                                   & 90.7\%                                & 60.9\%                               & 61.5\%                                 & 90.2\%                              & 67.3\%        \\ \hline
\begin{tabular}{@{}c@{}}Simple \\ (PGD training)\end{tabular}     & 77.1\%                                   & 78.1\%                                & 60.2\%                               & 77.0\%                                 & 77.9\%                              & 66.3\%             \\ \hline
\begin{tabular}{@{}c@{}}Wide \\ (standard training)\end{tabular}   &  34.9\%                                   & 78.7\%                                & 80.2\%                               & 21.3\%                                 & 75.8\%                              & 80.6\%   \\ \hline
\begin{tabular}{@{}c@{}}Wide \\ (FGSM training)\end{tabular}     & 64.5\%                                   & 93.6\%                                & 69.1\%                               & 53.7\%                                 & 92.2\%                              & 72.8\%          \\ \hline
\begin{tabular}{@{}c@{}}Wide \\ (PGD training)\end{tabular}       & 85.8\%                                   & 86.6\%                                & 73.3\%                               & 85.6\%                                 & 86.2\%                              & 67.0\%           \\ \hline
\end{tabular}
\caption{CIFAR10: black-box FGSM attacks. We create FGSM adversarial examples with $\epsilon=8$ from the evaluation set on the source network, and then evaluate them on an independently initialized target network.}
\label{cifar10-fgsm-transfer}
\end{table}

\begin{table}[htb]
\centering
\setlength{\tabcolsep}{0.2em}
\begin{tabular}{|l|l|l|l|l|l|l|ll}
\cline{1-7}
        \backslashbox[3.1cm]{Target}{Source}                 &
        \begin{tabular}{@{}c@{}}Simple \\ (standard\\ training)\end{tabular} & \begin{tabular}{@{}c@{}}Simple \\ (FGSM\\ training)\end{tabular} & \begin{tabular}{@{}c@{}}Simple \\ (PGD\\ training)\end{tabular} & \begin{tabular}{@{}c@{}}Wide \\ (natural\\ training)\end{tabular} & \begin{tabular}{@{}c@{}}Wide \\ (FGSM\\ training)\end{tabular} & \begin{tabular}{@{}c@{}}Wide \\ (PGD\\ training)\end{tabular} &  &  \\ \cline{1-7}
\begin{tabular}{@{}c@{}}Simple \\ (standard training)\end{tabular} & 6.6\%                                   & 71.6\%                               & 71.8\%                              & 1.4\%                                 & 51.4\%                             & 75.6\%                         &  &  \\ \cline{1-7}
\begin{tabular}{@{}c@{}}Simple \\ (FGSM training)\end{tabular}   & 66.3\%                                  & 40.3\%                               & 58.4\%                              & 65.4\%                                & 26.8\%                             & 66.2\%         &  &  \\ \cline{1-7}
\begin{tabular}{@{}c@{}}Simple \\ (PGD training)\end{tabular}       & 78.1\%                                  & 78.2\%                               & 57.7\%                              & 77.9\%                                & 78.1\%                             & 65.2\%  &  &  \\ \cline{1-7}
\begin{tabular}{@{}c@{}}Wide \\ (standard training)\end{tabular}   & 10.9\%                                  & 79.6\%                               & 79.1\%                              & 0.0\%                                 & 51.3\%                             & 79.7\%       &  &  \\ \cline{1-7}
\begin{tabular}{@{}c@{}}Wide \\ (FGSM training)\end{tabular}     & 67.6\%                                  & 51.7\%                               & 67.4\%                              & 56.5\%                                & 0.0\%                              & 71.6\%              &  &  \\ \cline{1-7}
\begin{tabular}{@{}c@{}}Wide \\ (PGD training)\end{tabular}      & 86.4\%                                  & 86.8\%                               & 72.1\%                              & 86.0\%                                & 86.3\%                             & 64.2\%       &  &  \\ \cline{1-7}
\end{tabular}
\caption{CIFAR10: black-box PGD attacks. We create PGD adversarial examples with $\epsilon=8$ for $7$ iterations from the evaluation set on the source network, and then evaluate them on an independently initialized target network.}
\label{cifar10-pgd-transfer}
\end{table}


\begin{table}[htb]
\setlength{\tabcolsep}{0.2em}
\centering
\begin{tabular}{|c|c|c|c|c|c|}
\hline
\backslashbox{Model}{Adversary}         & Natural & FGSM   & FGSM random & PGD (7 steps) & PGD (20 steps) \\ \hline
\begin{tabular}{@{}c@{}}Simple \\ (standard training)\end{tabular}              & 92.7\%  & 27.5\% & 19.6\%      & 1.2\%   & 0.8\%          \\ \hline
\begin{tabular}{@{}c@{}}Simple \\ (FGSM training)\end{tabular}    & 87.4\%  & 90.9\% & 90.4\%      & 0.0\%   & 0.0\%      \\ \hline
\begin{tabular}{@{}c@{}}Simple \\ (PGD training)\end{tabular}  & 79.4\%  & 51.7\% & 55.9\%      & 47.1\%  & 43.7\%       \\ \hline
\begin{tabular}{@{}c@{}}Wide \\ (standard training)\end{tabular}        & 95.2\%  & 32.7\% & 25.1\%      & 4.1\%   & 3.5\%         \\ \hline
\begin{tabular}{@{}c@{}}Wide \\ (FGSM training)\end{tabular}       & 90.3\%  & 95.1\% & 95.0\%      & 0.0\%   & 0.0\%       \\ \hline
\begin{tabular}{@{}c@{}}Wide \\ (PGD training)\end{tabular}      & 87.3\%  & 56.1\% & 60.3\%      & 50.0\%  & 45.8\%  \\ \hline
\end{tabular}
\caption{CIFAR10: white-box attacks for $\epsilon=8$. For each architecture and
training method, we list the accuracy of the resulting network on the full
CIFAR10 evaluation set of 10,000 examples. The FGSM random method is the one
suggested by \cite{tramer2017space}, whereby we first do a small random perturbation of the natural example, and the apply FGSM to that.}
\label{cifar10-white-box}
\end{table}

\begin{figure}[htb]
\begin{center}
{\setlength\tabcolsep{-.0cm}
\begin{tabular}{c c c c c}
\includegraphics[height=.142\textwidth]{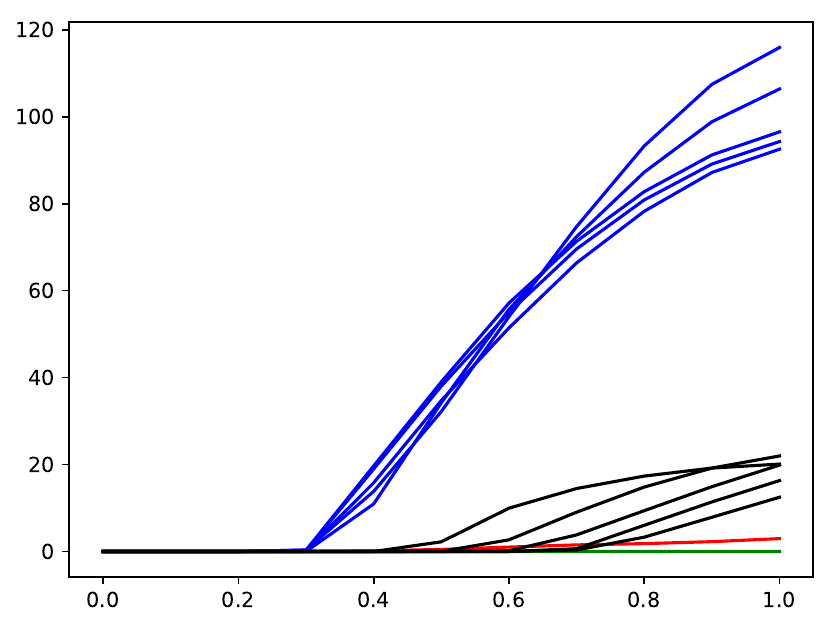} &
\includegraphics[height=.142\textwidth]{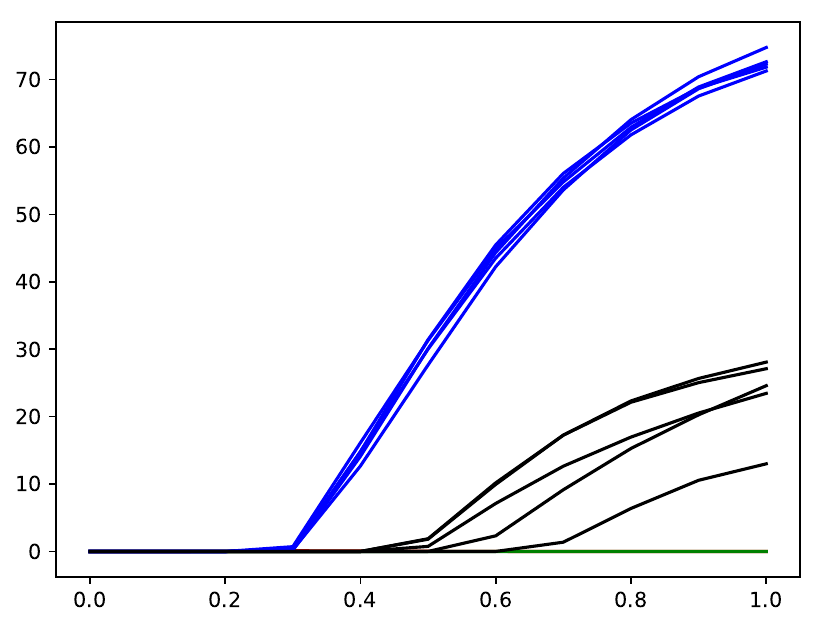} &
\includegraphics[height=.142\textwidth]{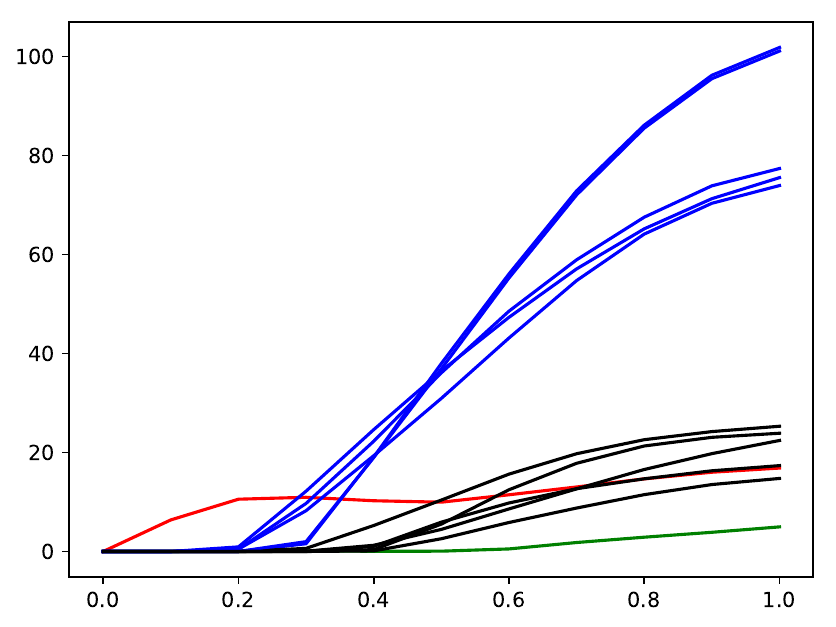} &
\includegraphics[height=.142\textwidth]{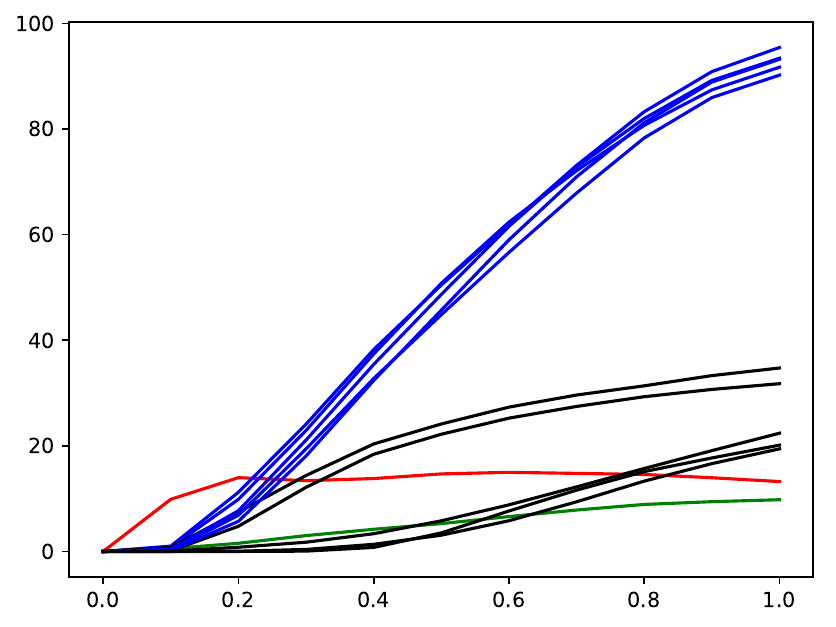} &
\includegraphics[height=.142\textwidth]{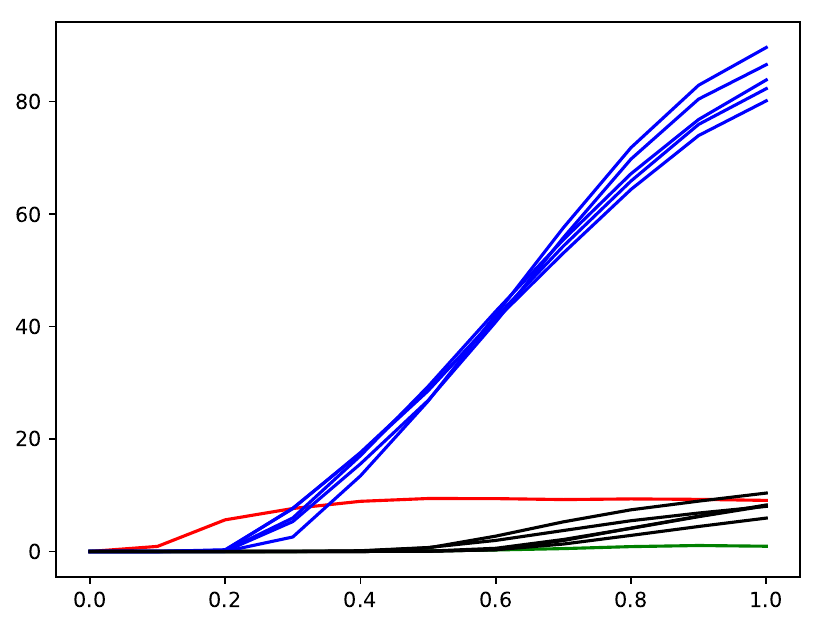}\\
\includegraphics[height=.142\textwidth]{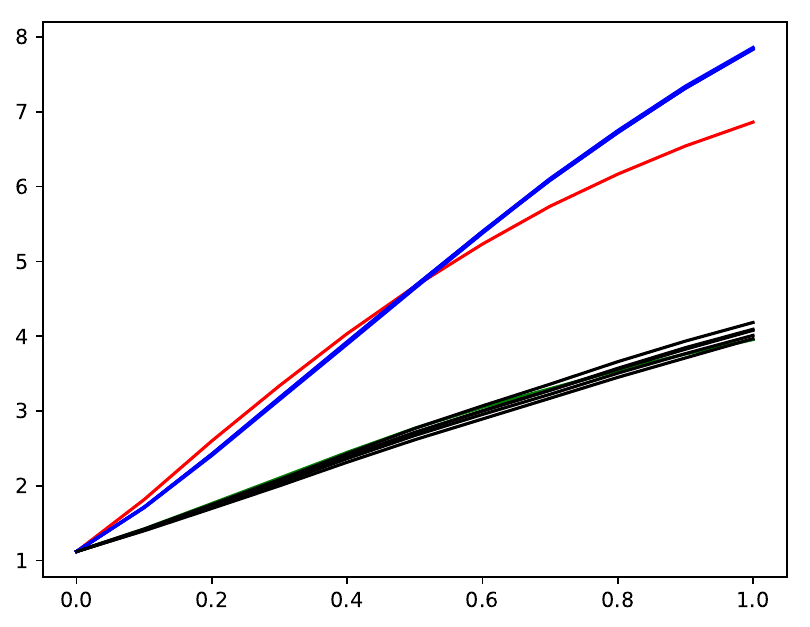} &
\includegraphics[height=.142\textwidth]{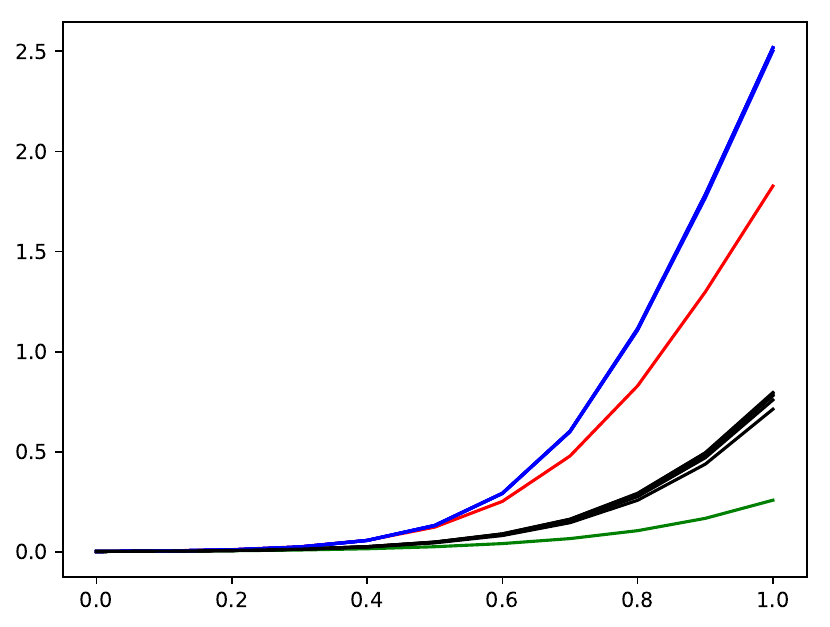} &
\includegraphics[height=.142\textwidth]{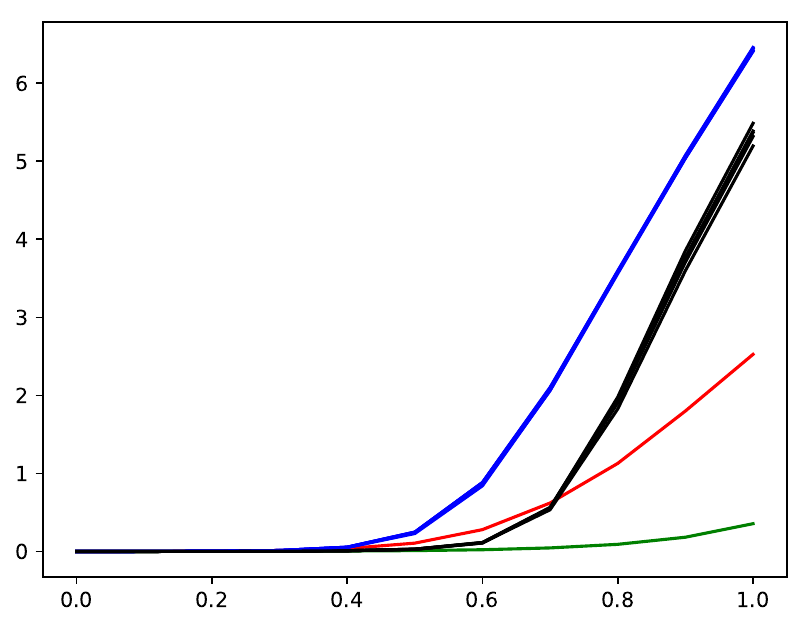} &
\includegraphics[height=.142\textwidth]{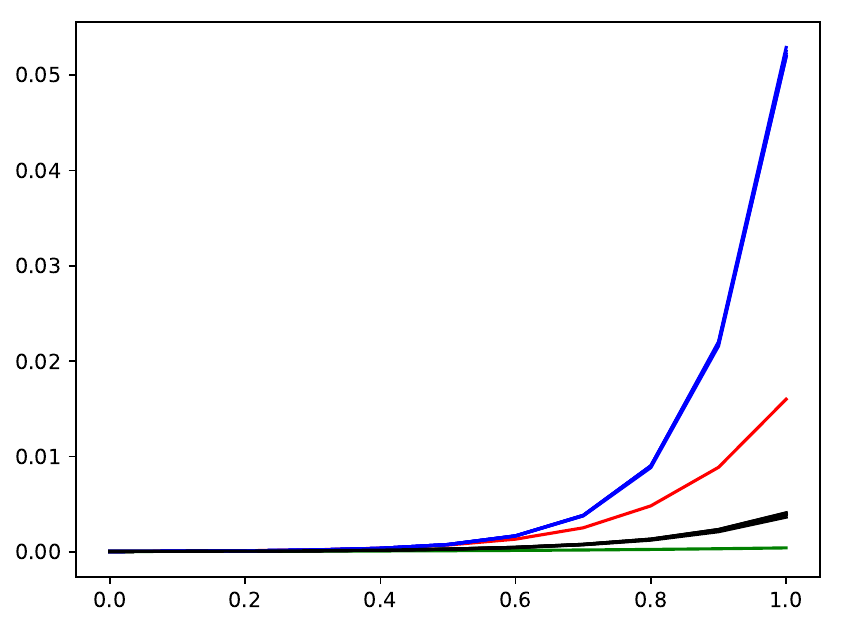} &
\includegraphics[height=.142\textwidth]{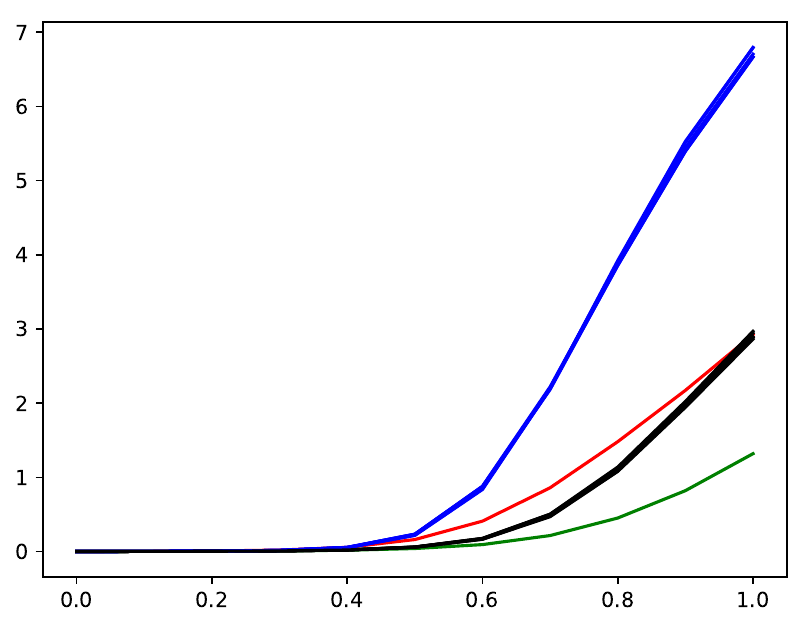}
\end{tabular}}
\end{center}
\caption{CIFAR10: change of loss function in the direction of white-box and black-box FGSM and PGD examples with $\epsilon=8$ for the same five natural examples. Each line shows how the loss changes as we move from the natural example to the corresponding adversarial example. Top: simple naturally trained model. Bottom: wide PGD trained model. We plot the loss of the original network in the direction of the FGSM example for the original network ({\color{red} red} lines), 5 PGD examples for the original network obtained from 5 random starting points ({\color{blue} blue} lines), the FGSM example for an independently trained copy network ({\color{green} green} lines) and 5 PGD examples for the copy network obtained from 5 random starting points ({\color{black} black} lines). All PGD attacks use 100 steps with step size $0.3$.}
\label{fig:cifar_transfer_loss_plots}
\end{figure}

\begin{figure}[h]

\begin{minipage}[c][3.25cm][t]{.25\textwidth}
\vspace{.1cm}
\setlength{\tabcolsep}{0.2em} 
\begin{tabular}{|l|r|r|}
\cline{2-3} 
\multicolumn{1}{l|}{} & {\small Source} & {\small Transfer}\tabularnewline
\hline 
{\small Clean} & \small{99.2\%} & \small{99.2\%} \tabularnewline
\hline 
{\small FGSM} & \small{3.9\%} & \small{41.9\%}\tabularnewline
\hline
{\small PGD} &\small{0.0\%} & \small{26.0\%}\tabularnewline
\hline 
\end{tabular}
\caption*{\small{Large network, standard training}}
 \vspace*{\fill}
 \end{minipage} \hspace{0.2cm}
\begin{minipage}[c][3cm][t]{.21\textwidth}
  \vspace*{\fill}
  \centering
  \includegraphics[width=3cm,height=3cm,trim={0 5.8cm 0 6.1cm},clip]{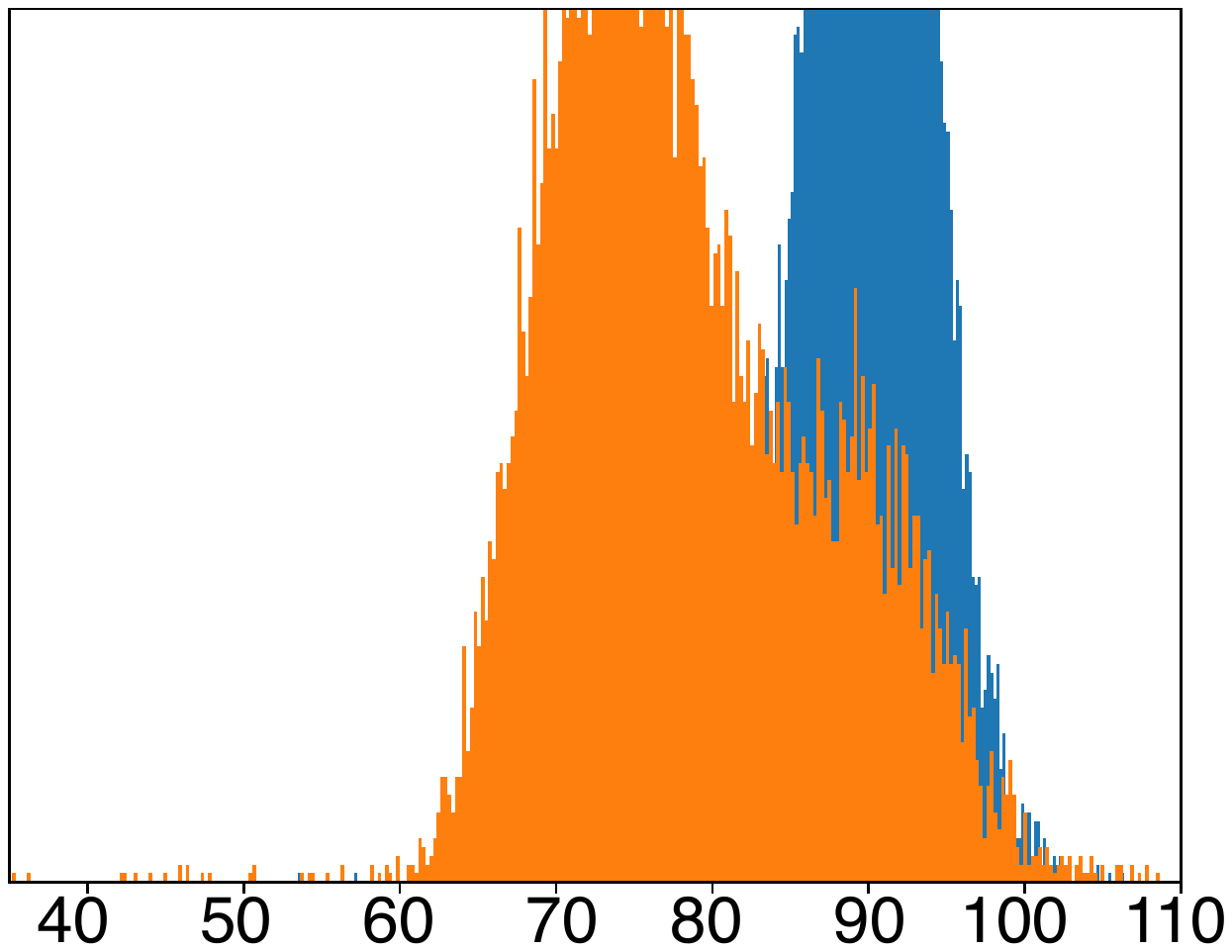}
\end{minipage}
\begin{minipage}[c][3cm][t]{.60\textwidth}
  \vspace*{\fill}
      \begin{subfigure}[H]{0.2\textwidth}
        \includegraphics[width=1.8cm, height=1.5cm, trim={0 0.7cm 0 0},clip]{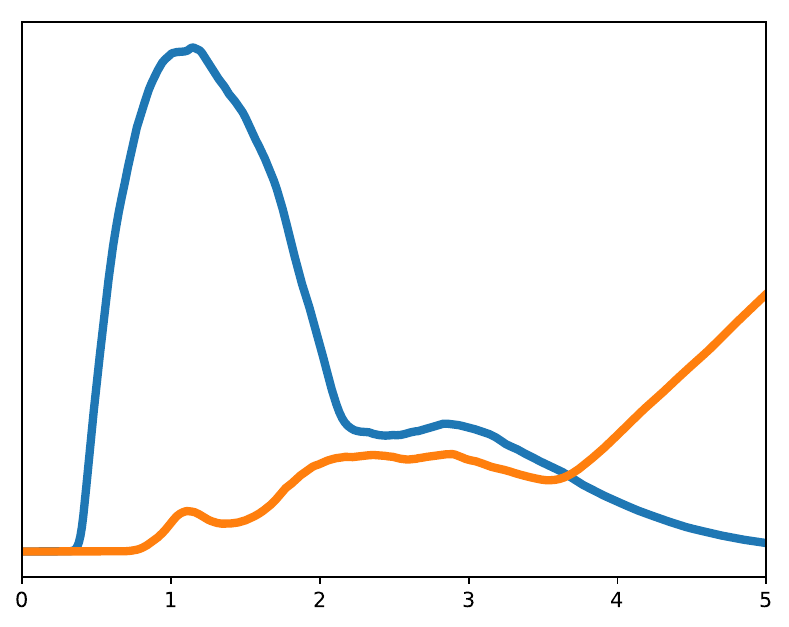}
       \end{subfigure}
      \begin{subfigure}[H]{0.2\textwidth}
        \includegraphics[width=1.8cm, height=1.5cm, trim={0 0.7cm 0 0},clip]{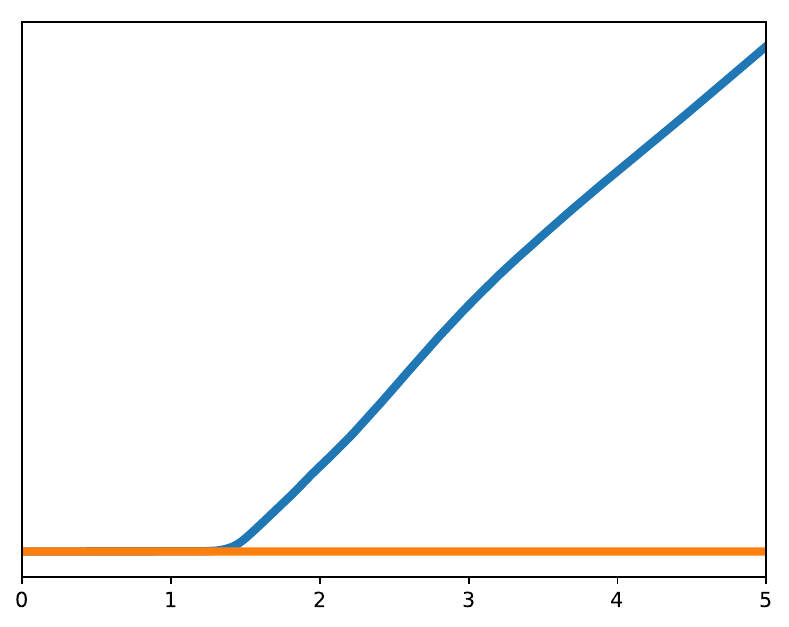}
    \end{subfigure}
      \begin{subfigure}[H]{0.2\textwidth}
        \includegraphics[width=1.8cm, height=1.5cm, trim={0 0.7cm 0 0},clip]{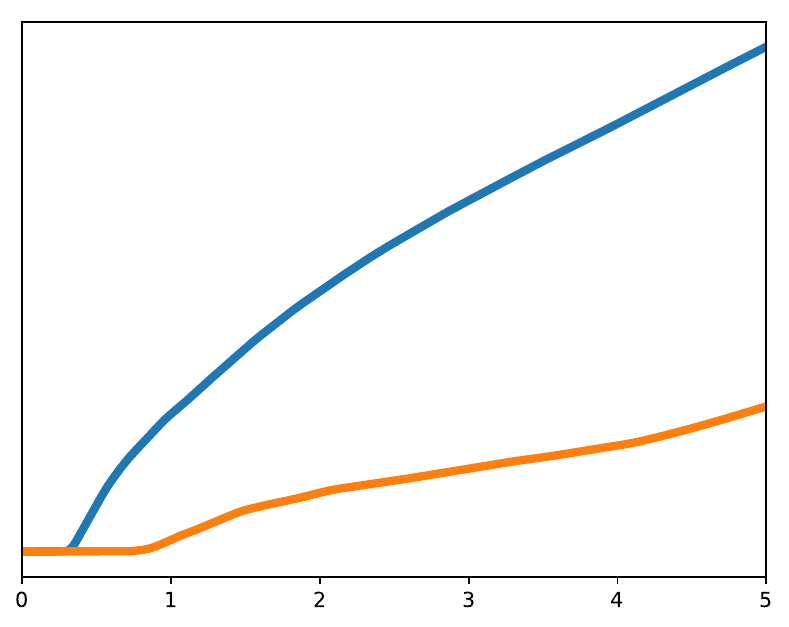}
       \end{subfigure}
      \begin{subfigure}[H]{0.2\textwidth}
        \includegraphics[width=1.8cm, height=1.5cm, trim={0 0.7cm 0 0},clip]{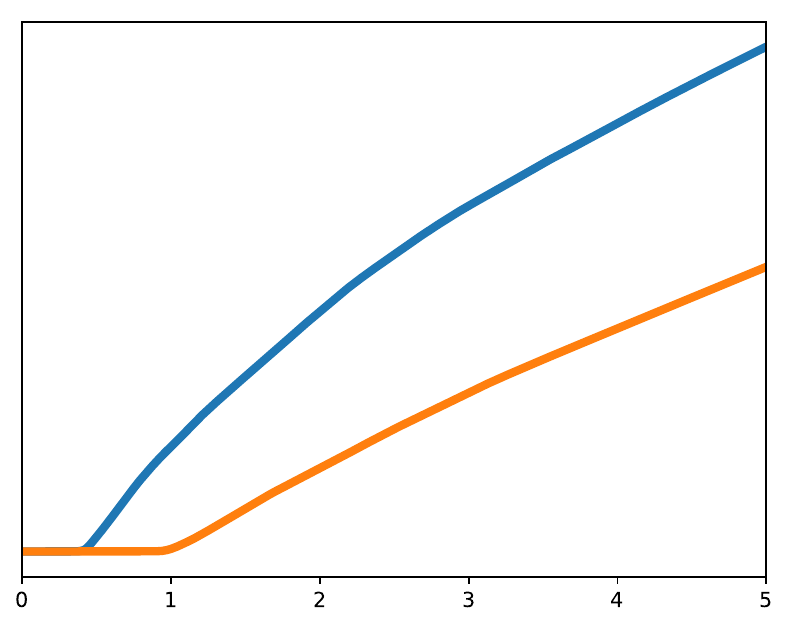}
    \end{subfigure}
   \newline\noindent
      \begin{subfigure}[H]{0.2\textwidth}
        \includegraphics[width=1.8cm, height=1.5cm, trim={0 0.7cm 0 0},clip]{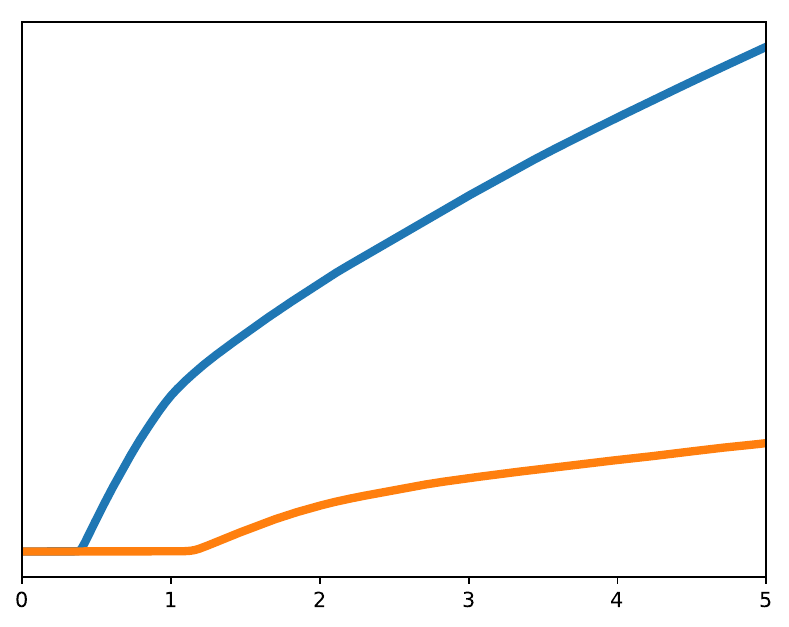}
       \end{subfigure}
      \begin{subfigure}[H]{0.2\textwidth}
        \includegraphics[width=1.8cm, height=1.5cm, trim={0 0.7cm 0 0},clip]{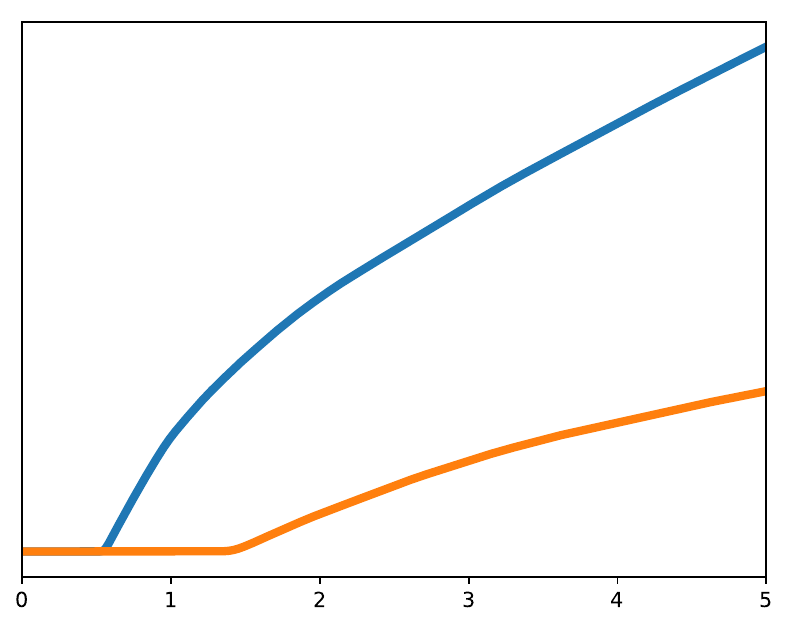}
    \end{subfigure}
      \begin{subfigure}[H]{0.2\textwidth}
        \includegraphics[width=1.8cm, height=1.5cm, trim={0 0.7cm 0 0},clip]{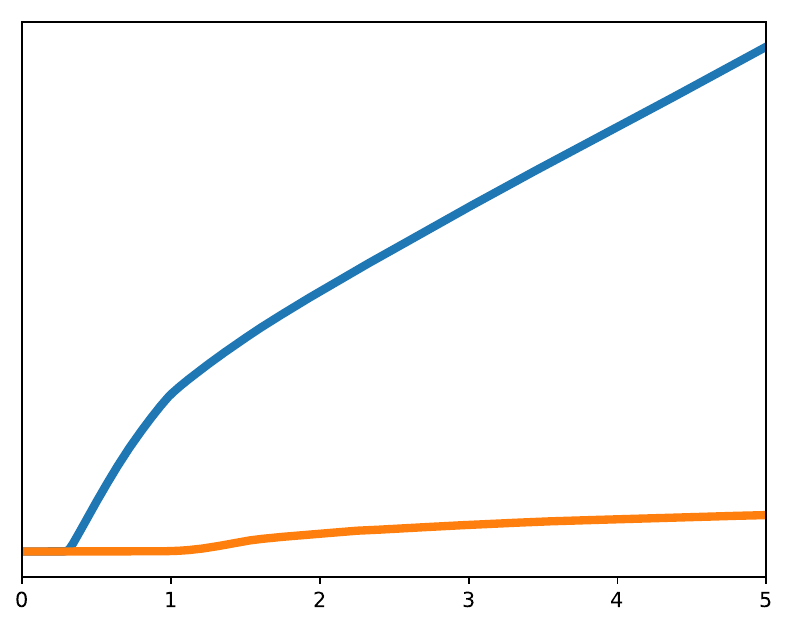}
       \end{subfigure}
      \begin{subfigure}[H]{0.2\textwidth}
        \includegraphics[width=1.8cm, height=1.5cm, trim={0 0.7cm 0 0},clip]{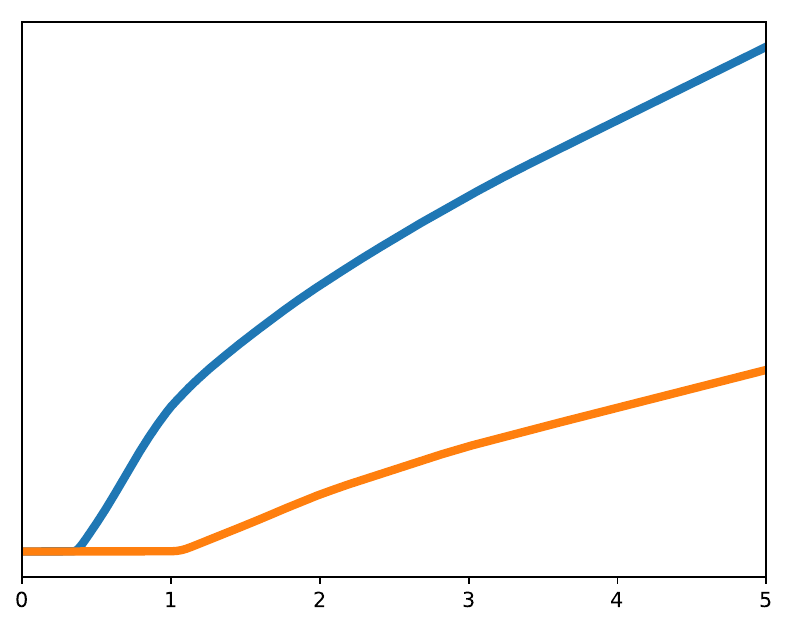}
    \end{subfigure}
        \newline\noindent
    \tiny{0\hspace{0.21cm}1\hspace{0.21cm}2\hspace{0.21cm}3\hspace{0.21cm}4\hspace{0.21cm}5}\hspace{0.20cm}
    \tiny{0\hspace{0.21cm}1\hspace{0.21cm}2\hspace{0.21cm}3\hspace{0.21cm}4\hspace{0.21cm}5}\hspace{0.20cm}
    \tiny{0\hspace{0.21cm}1\hspace{0.21cm}2\hspace{0.21cm}3\hspace{0.21cm}4\hspace{0.21cm}5}\hspace{0.20cm}
    \tiny{0\hspace{0.21cm}1\hspace{0.21cm}2\hspace{0.21cm}3\hspace{0.21cm}4\hspace{0.21cm}5}
\end{minipage}

\begin{minipage}[c][3.25cm][t]{.25\textwidth}
\vspace{.1cm}
\setlength{\tabcolsep}{0.2em} 
\begin{tabular}{|l|r|r|}
\cline{2-3} 
\multicolumn{1}{l|}{} & {\small Source} & {\small Transfer}\tabularnewline
\hline 
{\small Clean} & \small{99.2\%} & \small{99.3\%} \tabularnewline
\hline 
{\small FGSM} & \small{7.2\%} & \small{44.6\%}\tabularnewline
\hline
{\small PGD} &\small{0.0\%} & \small{35.0\%}\tabularnewline
\hline 
\end{tabular}
\caption*{\small{Very large network, standard training}}
 \vspace*{\fill}
 \end{minipage} \hspace{0.2cm}
\begin{minipage}[c][3cm][t]{.21\textwidth}
  \vspace*{\fill}
  \centering
\includegraphics[width=3cm,height=3cm,trim={0 5.8cm 0 6.1cm},clip]{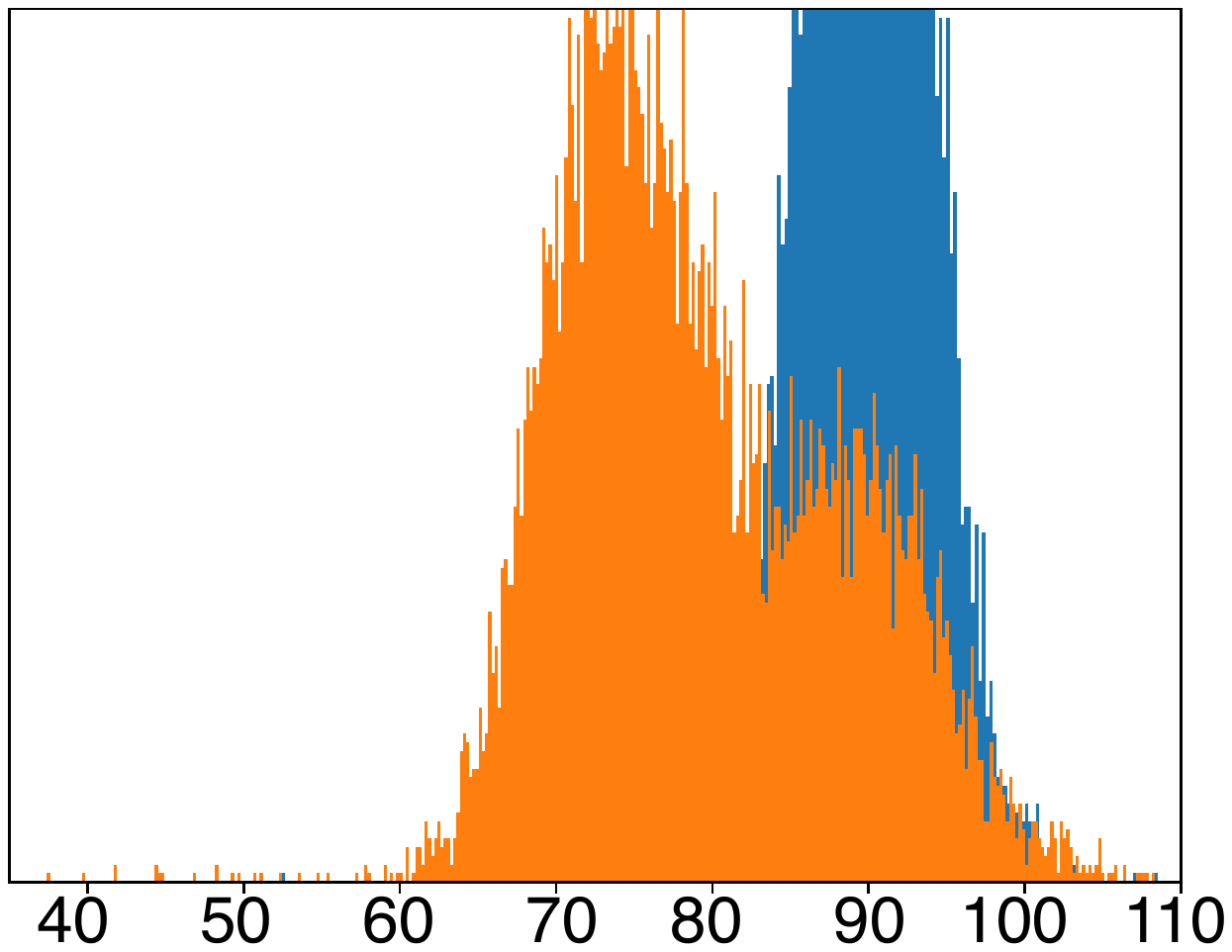}
\end{minipage}
\begin{minipage}[c][3cm][t]{.60\textwidth}
  \vspace*{\fill}
      \begin{subfigure}[H]{0.2\textwidth}
        \includegraphics[width=1.8cm, height=1.5cm, trim={0 0.7cm 0 0},clip]{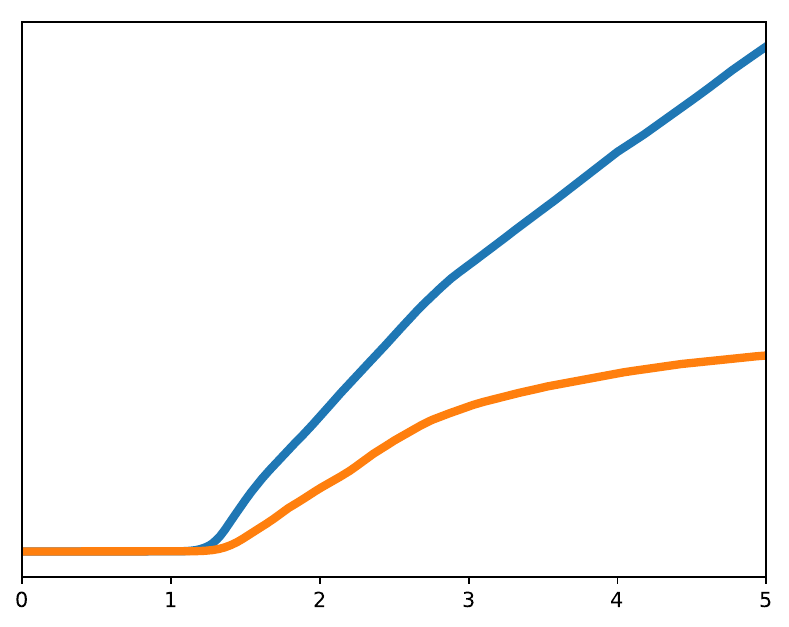}
       \end{subfigure}
      \begin{subfigure}[H]{0.2\textwidth}
        \includegraphics[width=1.8cm, height=1.5cm, trim={0 0.7cm 0 0},clip]{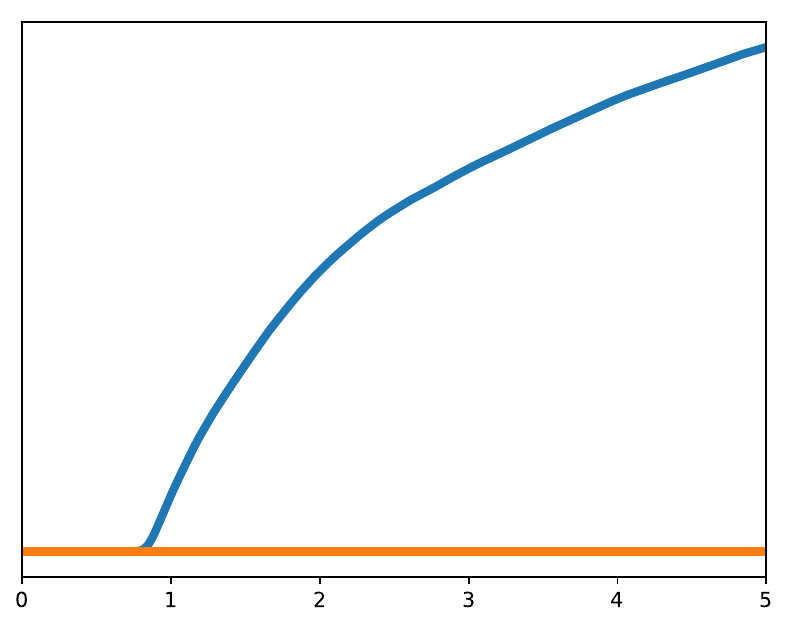}
    \end{subfigure}
      \begin{subfigure}[H]{0.2\textwidth}
        \includegraphics[width=1.8cm, height=1.5cm, trim={0 0.7cm 0 0},clip]{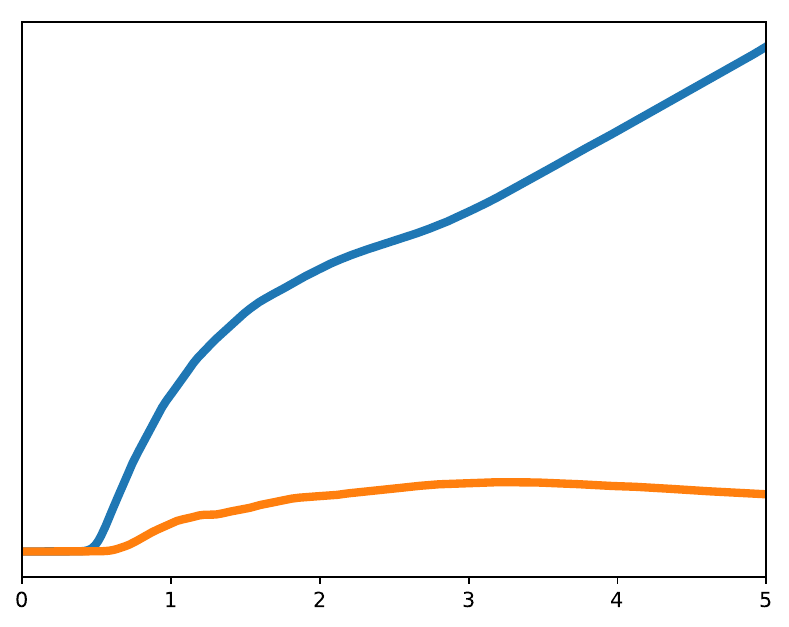}
       \end{subfigure}
      \begin{subfigure}[H]{0.2\textwidth}
        \includegraphics[width=1.8cm, height=1.5cm, trim={0 0.7cm 0 0},clip]{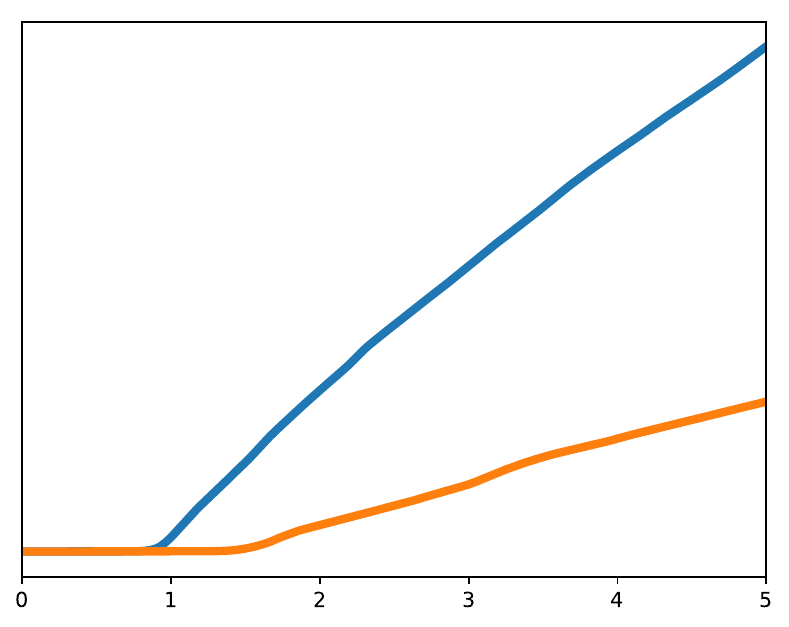}
    \end{subfigure}
   \newline\noindent
      \begin{subfigure}[H]{0.2\textwidth}
        \includegraphics[width=1.8cm, height=1.5cm, trim={0 0.7cm 0 0},clip]{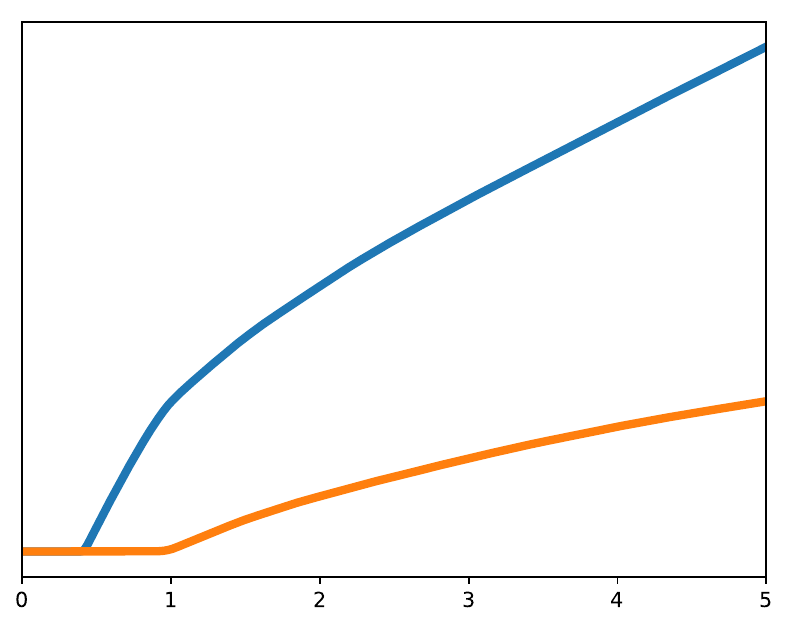}
       \end{subfigure}
      \begin{subfigure}[H]{0.2\textwidth}
        \includegraphics[width=1.8cm, height=1.5cm, trim={0 0.7cm 0 0},clip]{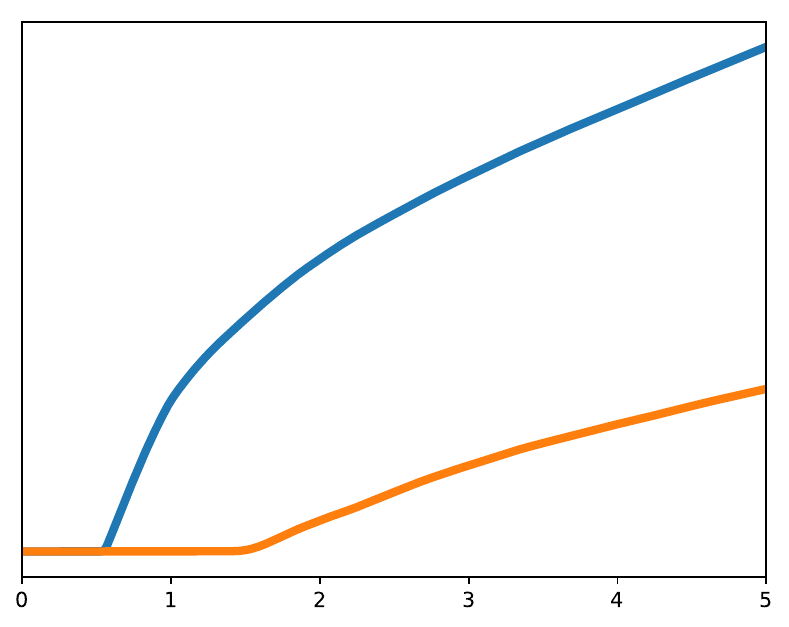}
    \end{subfigure}
      \begin{subfigure}[H]{0.2\textwidth}
        \includegraphics[width=1.8cm, height=1.5cm, trim={0 0.7cm 0 0},clip]{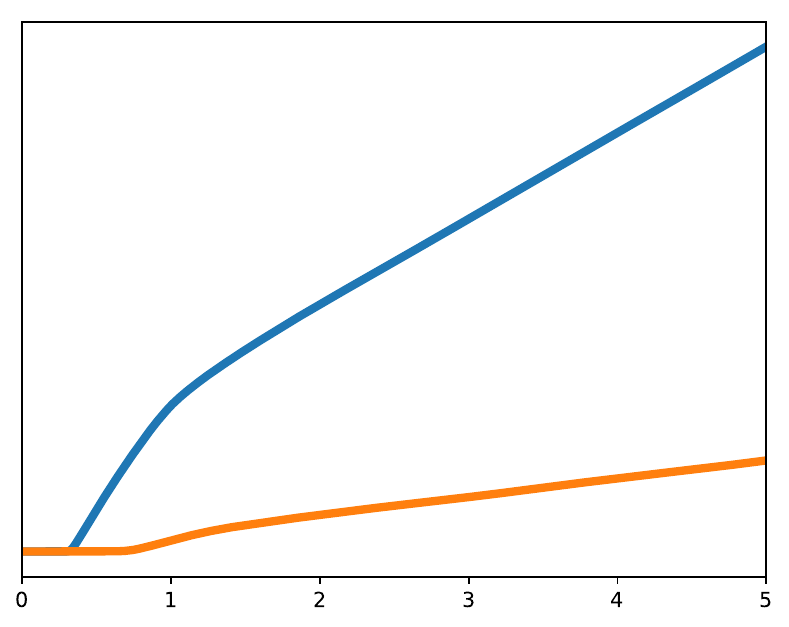}
       \end{subfigure}
      \begin{subfigure}[H]{0.2\textwidth}
        \includegraphics[width=1.8cm, height=1.5cm, trim={0 0.7cm 0 0},clip]{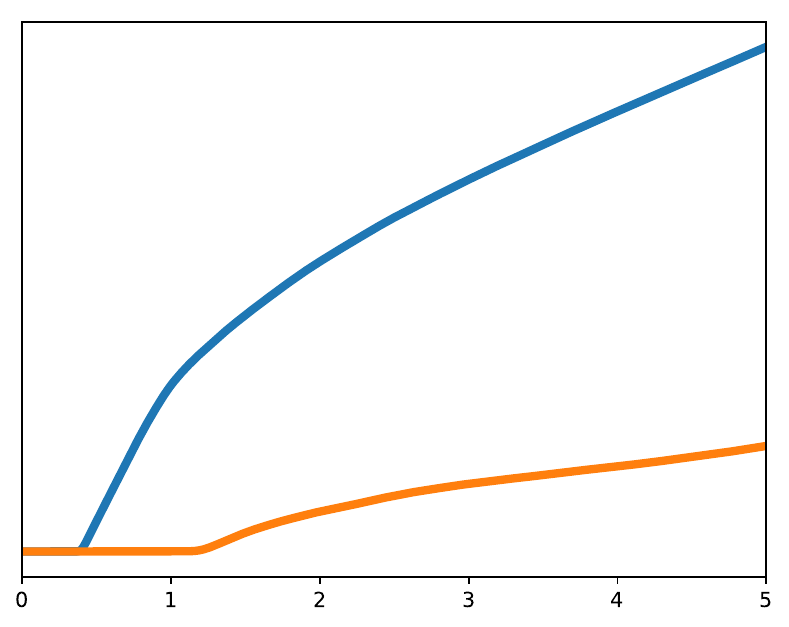}
    \end{subfigure}
    \newline\noindent
    \tiny{0\hspace{0.21cm}1\hspace{0.21cm}2\hspace{0.21cm}3\hspace{0.21cm}4\hspace{0.21cm}5}\hspace{0.20cm}
    \tiny{0\hspace{0.21cm}1\hspace{0.21cm}2\hspace{0.21cm}3\hspace{0.21cm}4\hspace{0.21cm}5}\hspace{0.20cm}
    \tiny{0\hspace{0.21cm}1\hspace{0.21cm}2\hspace{0.21cm}3\hspace{0.21cm}4\hspace{0.21cm}5}\hspace{0.20cm}
    \tiny{0\hspace{0.21cm}1\hspace{0.21cm}2\hspace{0.21cm}3\hspace{0.21cm}4\hspace{0.21cm}5}
\end{minipage}

\begin{minipage}[c][3.25cm][t]{.25\textwidth}
\vspace{.1cm}
\setlength{\tabcolsep}{0.2em} 
\begin{tabular}{|l|r|r|}
\cline{2-3} 
\multicolumn{1}{l|}{} & {\small Source} & {\small Transfer}\tabularnewline
\hline 
{\small Clean} & \small{92.9\%} & \small{96.1\%} \tabularnewline
\hline 
{\small FGSM} & \small{99.9\%} & \small{62.0\%}\tabularnewline
\hline
{\small PGD} &\small{0.0\%} & \small{54.1\%}\tabularnewline
\hline 
\end{tabular}
\caption*{\small{Large network, FGSM training}}
 \vspace*{\fill}
 \end{minipage} \hspace{0.2cm}
\begin{minipage}[c][3cm][t]{.21\textwidth}
  \vspace*{\fill}
  \centering
\includegraphics[width=3cm,height=3cm,trim={0 5.8cm 0 6.1cm},clip]{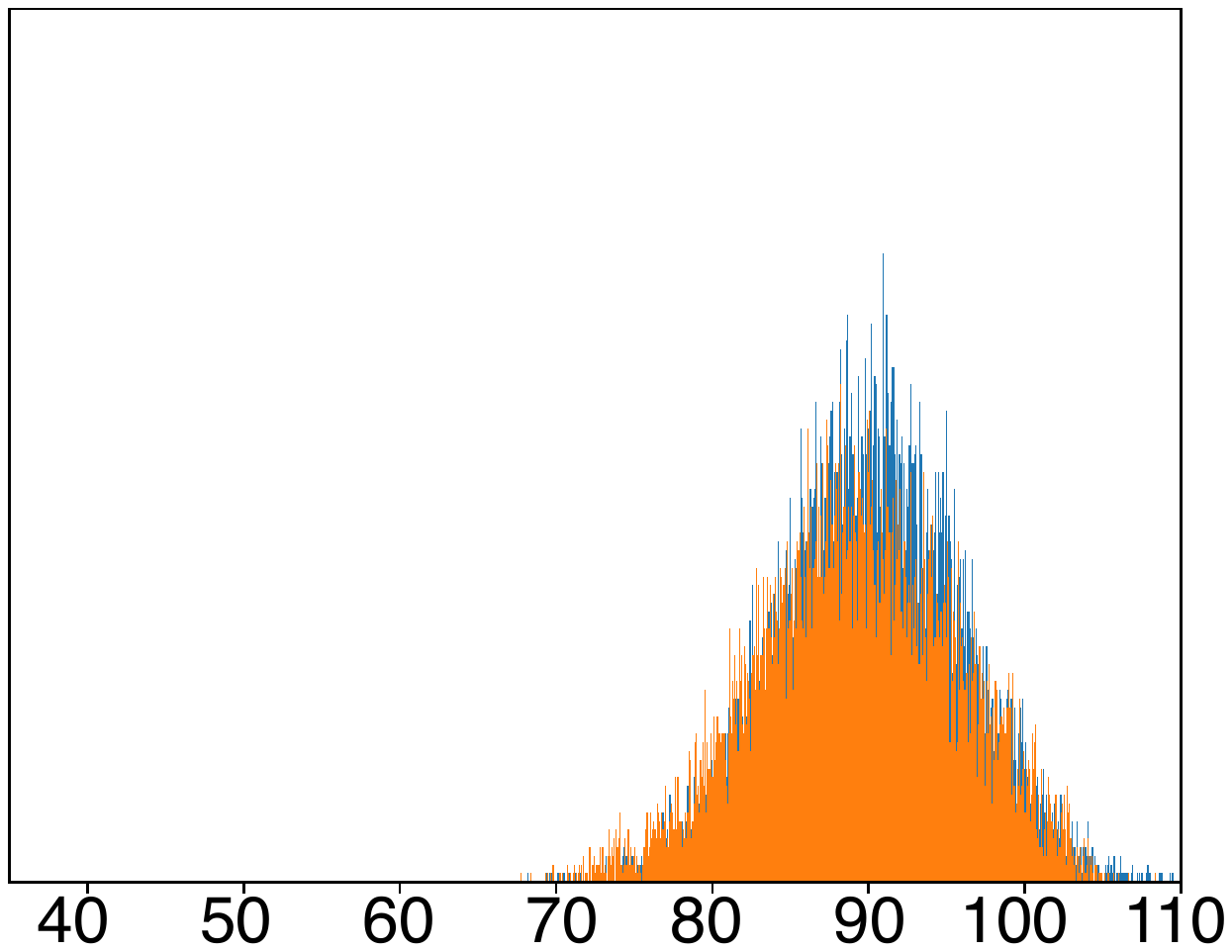}
\end{minipage}
\begin{minipage}[c][3cm][t]{.60\textwidth}
  \vspace*{\fill}
      \begin{subfigure}[H]{0.2\textwidth}
        \includegraphics[width=1.8cm, height=1.5cm, trim={0 0.7cm 0 0},clip]{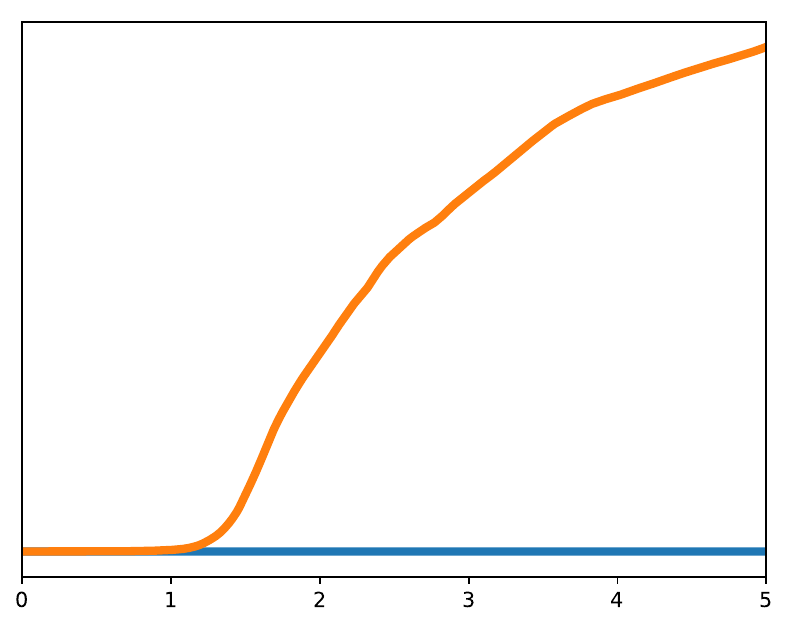}
       \end{subfigure}
      \begin{subfigure}[H]{0.2\textwidth}
        \includegraphics[width=1.8cm, height=1.5cm, trim={0 0.7cm 0 0},clip]{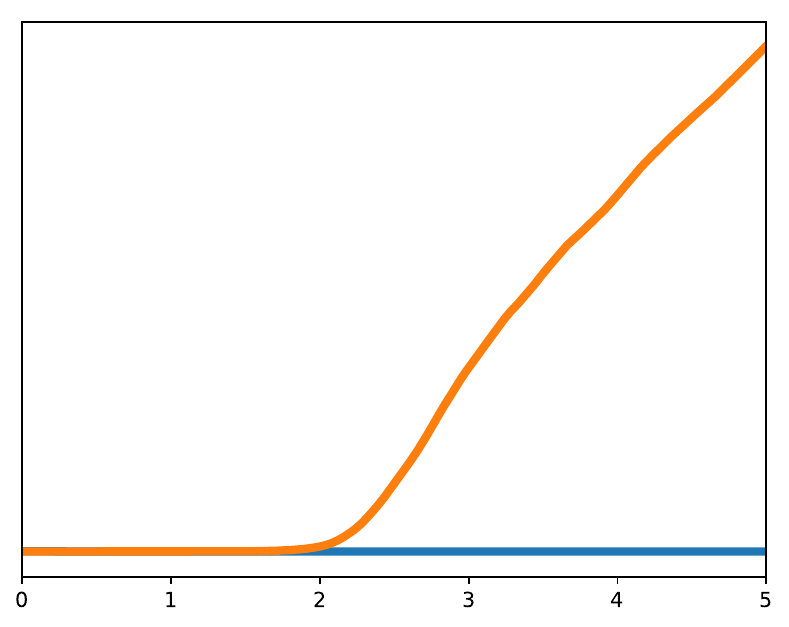}
    \end{subfigure}
      \begin{subfigure}[H]{0.2\textwidth}
        \includegraphics[width=1.8cm, height=1.5cm, trim={0 0.7cm 0 0},clip]{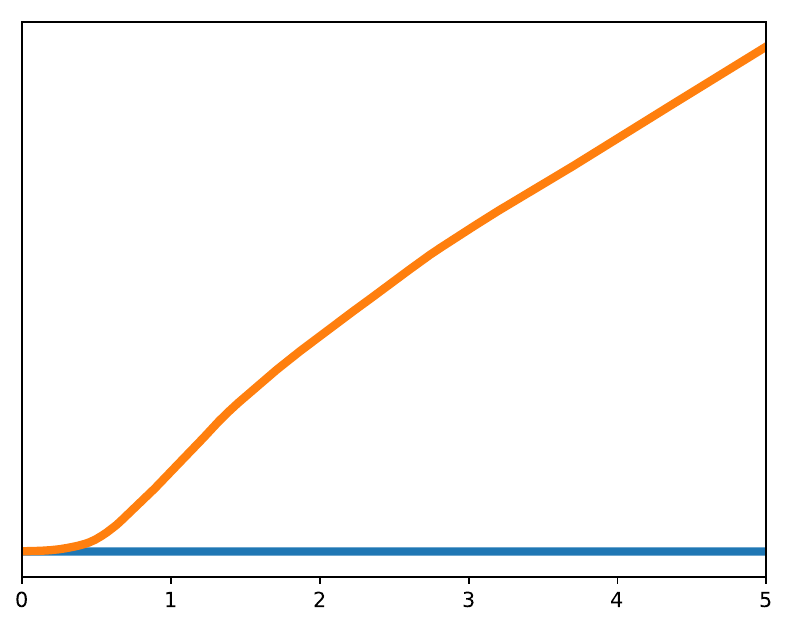}
       \end{subfigure}
      \begin{subfigure}[H]{0.2\textwidth}
        \includegraphics[width=1.8cm, height=1.5cm, trim={0 0.7cm 0 0},clip]{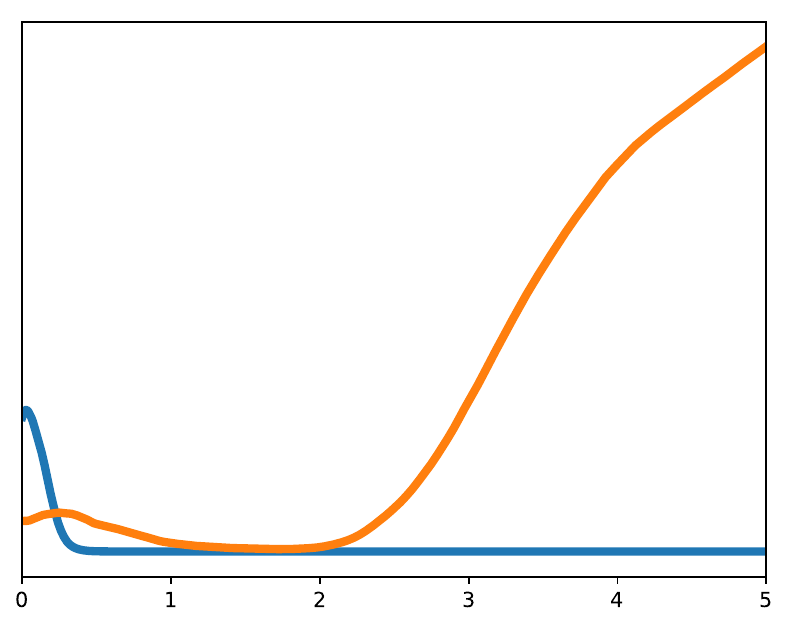}
    \end{subfigure}
   \newline\noindent
      \begin{subfigure}[H]{0.2\textwidth}
        \includegraphics[width=1.8cm, height=1.5cm, trim={0 0.7cm 0 0},clip]{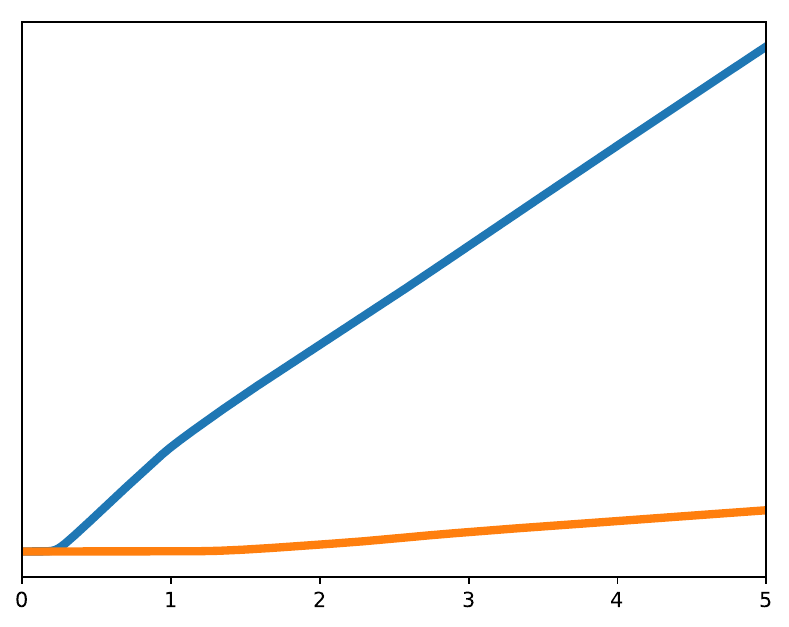}
       \end{subfigure}
      \begin{subfigure}[H]{0.2\textwidth}
        \includegraphics[width=1.8cm, height=1.5cm, trim={0 0.7cm 0 0},clip]{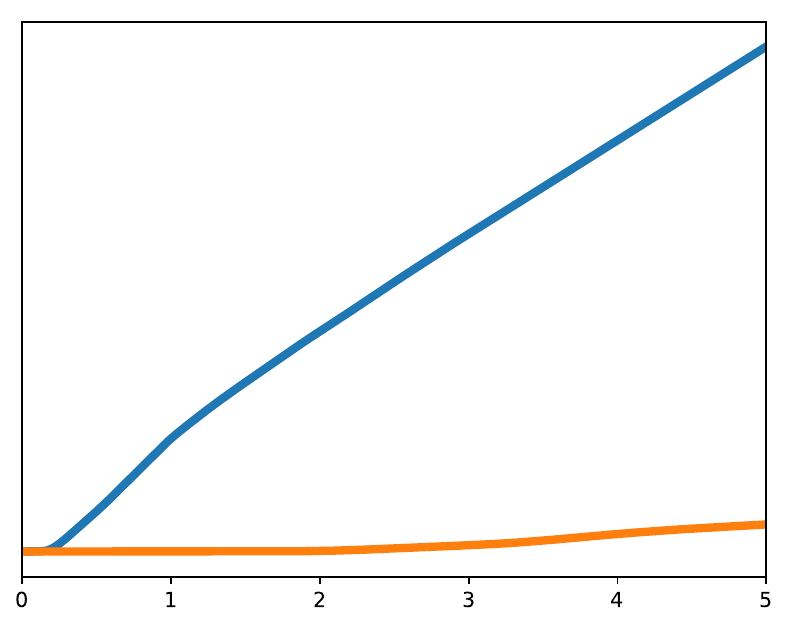}
    \end{subfigure}
      \begin{subfigure}[H]{0.2\textwidth}
        \includegraphics[width=1.8cm, height=1.5cm, trim={0 0.7cm 0 0},clip]{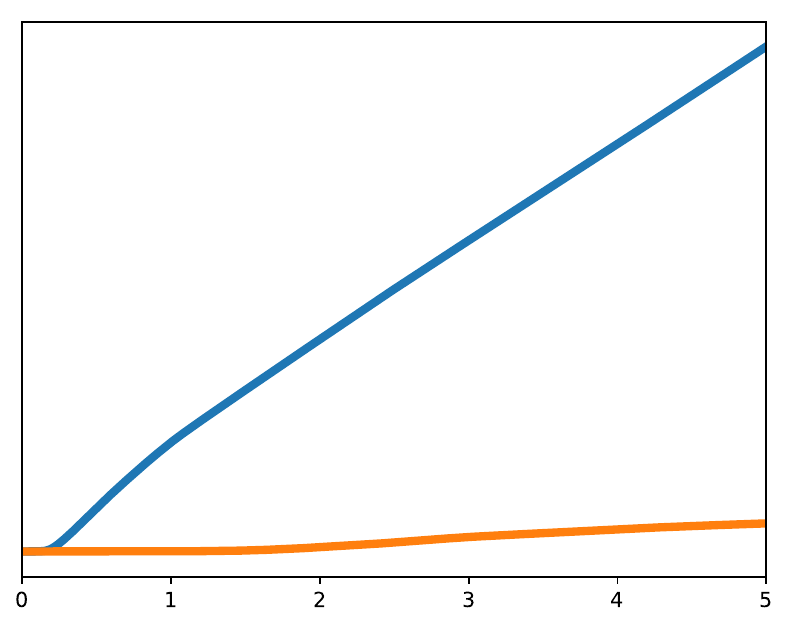}
       \end{subfigure}
      \begin{subfigure}[H]{0.2\textwidth}
        \includegraphics[width=1.8cm, height=1.5cm, trim={0 0.7cm 0 0},clip]{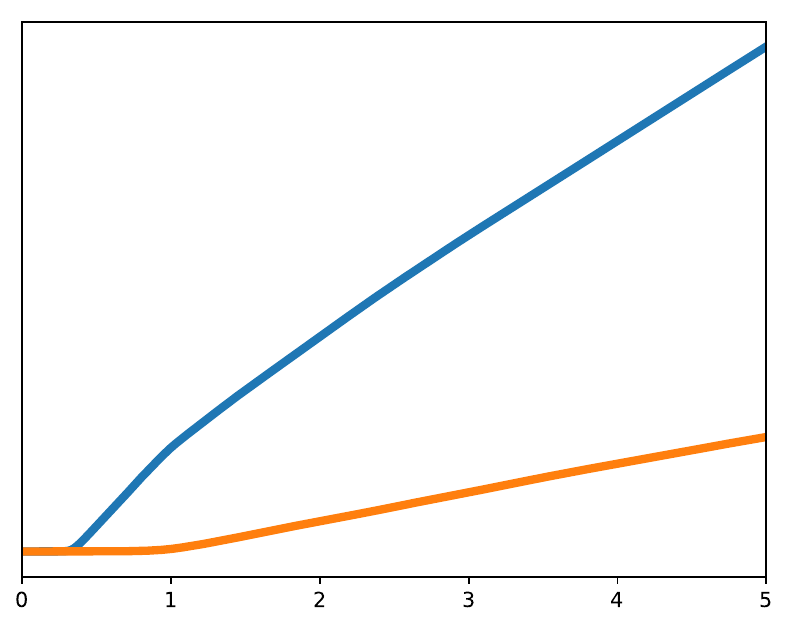}
    \end{subfigure}
    \newline\noindent
    \tiny{0\hspace{0.21cm}1\hspace{0.21cm}2\hspace{0.21cm}3\hspace{0.21cm}4\hspace{0.21cm}5}\hspace{0.20cm}
    \tiny{0\hspace{0.21cm}1\hspace{0.21cm}2\hspace{0.21cm}3\hspace{0.21cm}4\hspace{0.21cm}5}\hspace{0.20cm}
    \tiny{0\hspace{0.21cm}1\hspace{0.21cm}2\hspace{0.21cm}3\hspace{0.21cm}4\hspace{0.21cm}5}\hspace{0.20cm}
    \tiny{0\hspace{0.21cm}1\hspace{0.21cm}2\hspace{0.21cm}3\hspace{0.21cm}4\hspace{0.21cm}5}
\end{minipage}

\begin{minipage}[c][3.25cm][t]{.25\textwidth}
\vspace{.1cm}
\setlength{\tabcolsep}{0.2em} 
\begin{tabular}{|l|r|r|}
\cline{2-3} 
\multicolumn{1}{l|}{} & {\small Source} & {\small Transfer}\tabularnewline
\hline 
{\small Clean} & \small{96.4\%} & \small{97.8\%} \tabularnewline
\hline 
{\small FGSM} & \small{99.4\%} & \small{71.6\%}\tabularnewline
\hline
{\small PGD} &\small{0.0\%} & \small{60.6\%}\tabularnewline
\hline 
\end{tabular}
\caption*{\small{Very large network, FGSM training}}
 \vspace*{\fill}
 \end{minipage} \hspace{0.2cm}
\begin{minipage}[c][3cm][t]{.21\textwidth}
  \vspace*{\fill}
  \centering
\includegraphics[width=3cm,height=3cm,trim={0 5.8cm 0 6.1cm},clip]{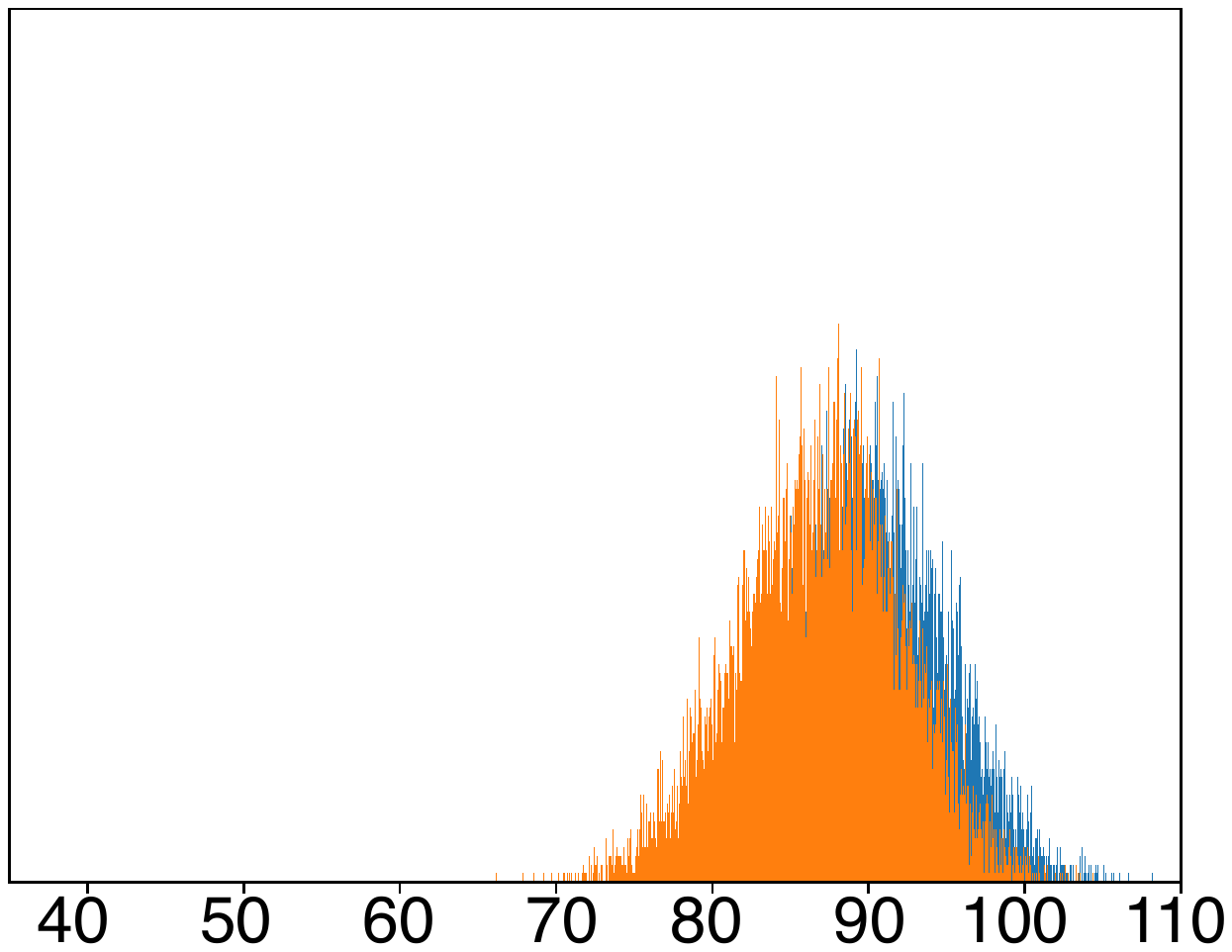}
\end{minipage}
\begin{minipage}[c][3cm][t]{.60\textwidth}
  \vspace*{\fill}
      \begin{subfigure}[H]{0.2\textwidth}
        \includegraphics[width=1.8cm, height=1.5cm, trim={0 0.7cm 0 0},clip]{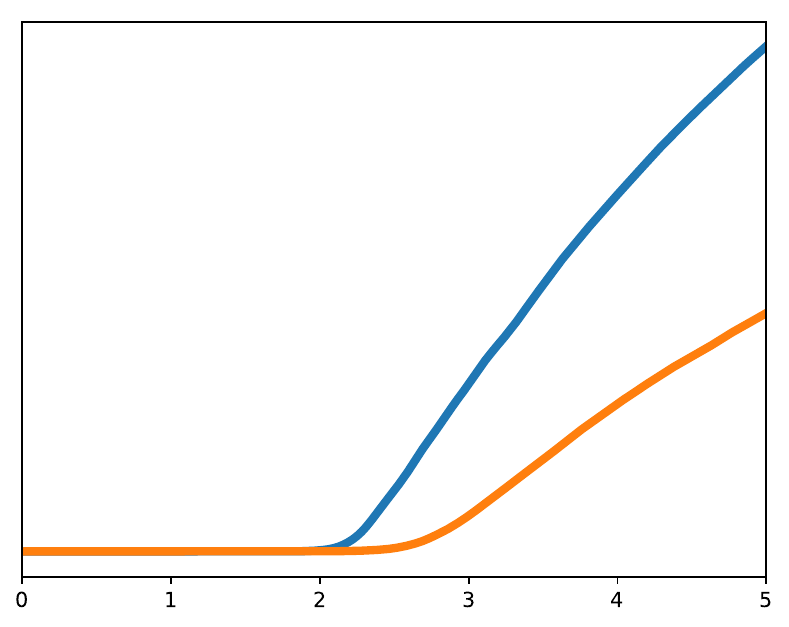}
       \end{subfigure}
      \begin{subfigure}[H]{0.2\textwidth}
        \includegraphics[width=1.8cm, height=1.5cm, trim={0 0.7cm 0 0},clip]{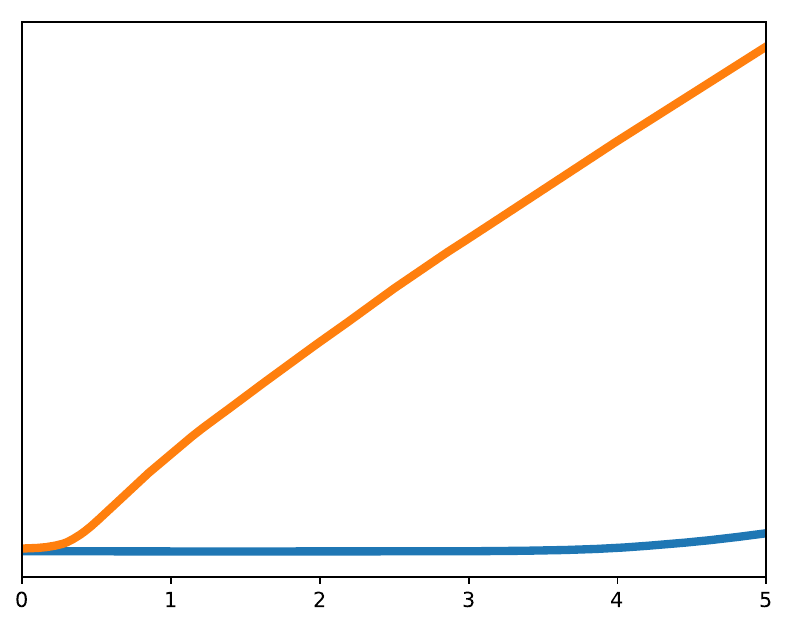}
    \end{subfigure}
      \begin{subfigure}[H]{0.2\textwidth}
        \includegraphics[width=1.8cm, height=1.5cm, trim={0 0.7cm 0 0},clip]{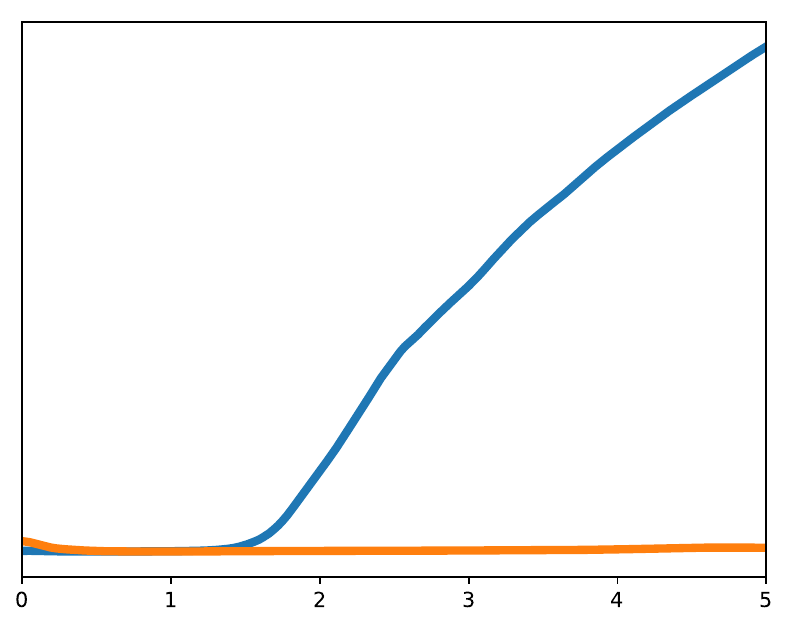}
       \end{subfigure}
      \begin{subfigure}[H]{0.2\textwidth}
        \includegraphics[width=1.8cm, height=1.5cm, trim={0 0.7cm 0 0},clip]{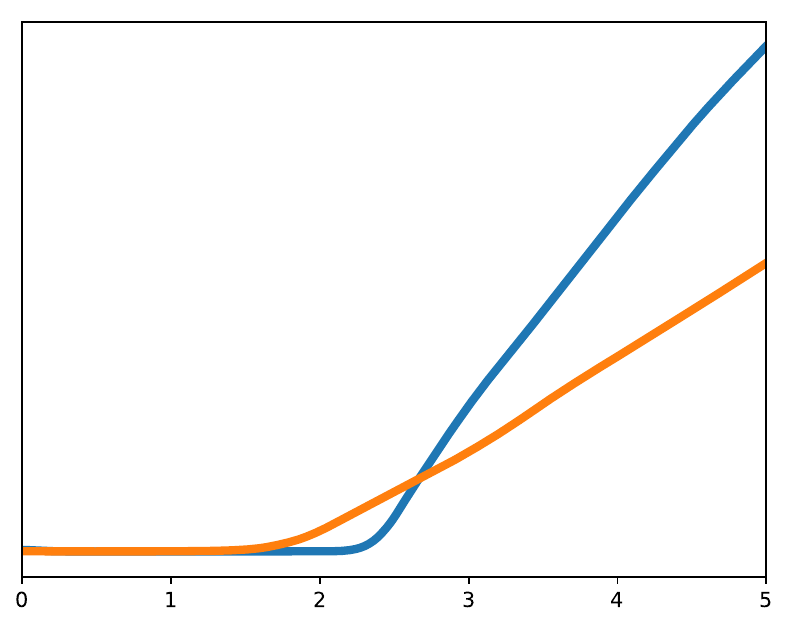}
    \end{subfigure}
   \newline\noindent
      \begin{subfigure}[H]{0.2\textwidth}
        \includegraphics[width=1.8cm, height=1.5cm, trim={0 0.7cm 0 0},clip]{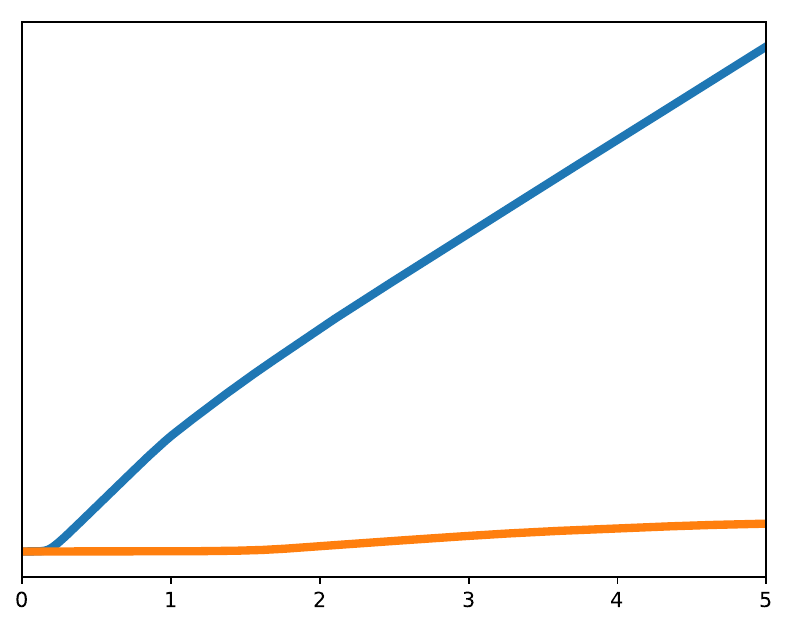}
       \end{subfigure}
      \begin{subfigure}[H]{0.2\textwidth}
        \includegraphics[width=1.8cm, height=1.5cm, trim={0 0.7cm 0 0},clip]{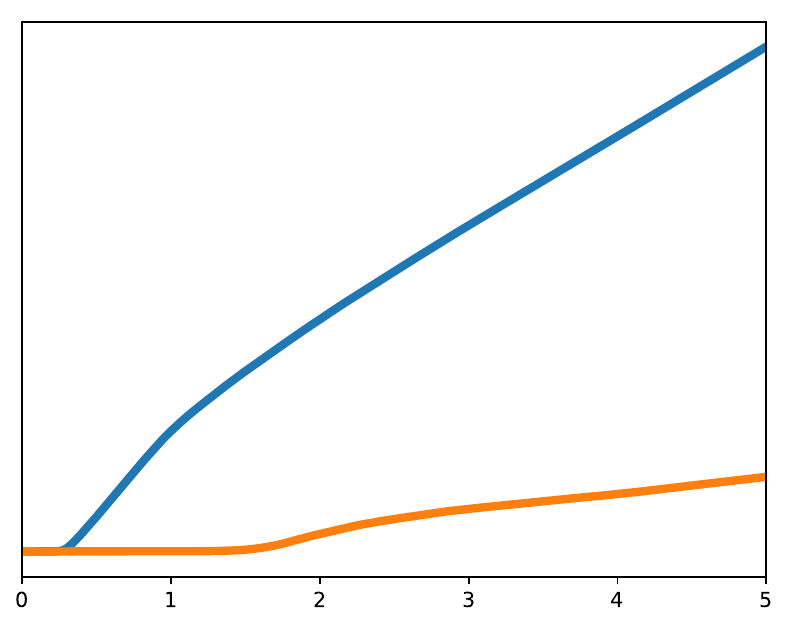}
    \end{subfigure}
      \begin{subfigure}[H]{0.2\textwidth}
        \includegraphics[width=1.8cm, height=1.5cm, trim={0 0.7cm 0 0},clip]{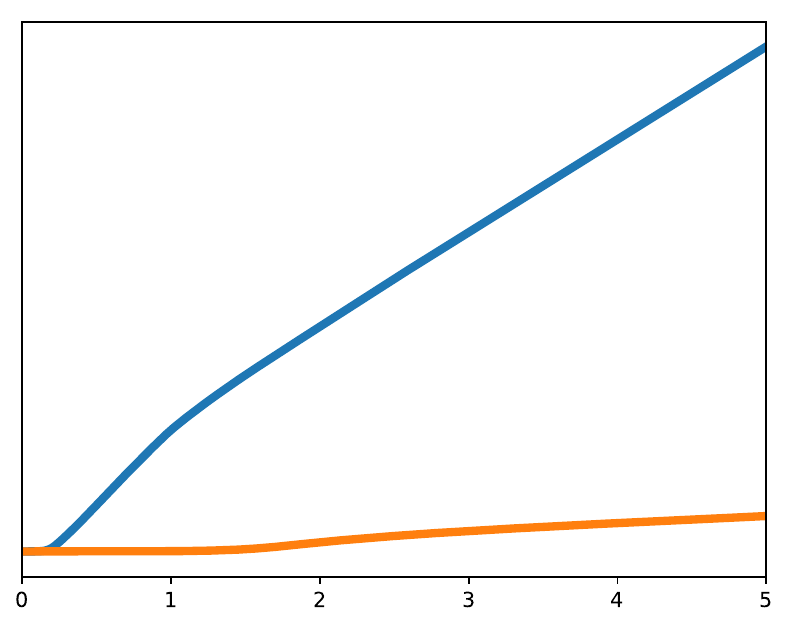}
       \end{subfigure}
      \begin{subfigure}[H]{0.2\textwidth}
        \includegraphics[width=1.8cm, height=1.5cm, trim={0 0.7cm 0 0},clip]{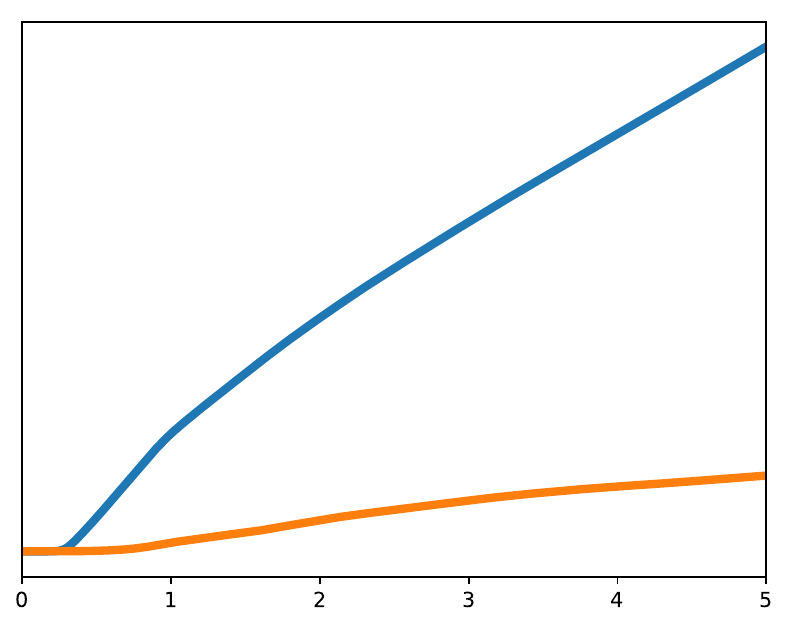}
    \end{subfigure}
    \newline\noindent
    \tiny{0\hspace{0.21cm}1\hspace{0.21cm}2\hspace{0.21cm}3\hspace{0.21cm}4\hspace{0.21cm}5}\hspace{0.20cm}
    \tiny{0\hspace{0.21cm}1\hspace{0.21cm}2\hspace{0.21cm}3\hspace{0.21cm}4\hspace{0.21cm}5}\hspace{0.20cm}
    \tiny{0\hspace{0.21cm}1\hspace{0.21cm}2\hspace{0.21cm}3\hspace{0.21cm}4\hspace{0.21cm}5}\hspace{0.20cm}
    \tiny{0\hspace{0.21cm}1\hspace{0.21cm}2\hspace{0.21cm}3\hspace{0.21cm}4\hspace{0.21cm}5}
\end{minipage}

\caption{Transferability experiments for four different instances (standard large and very large networks, and FGSM-trained large and very large networks, respectively). For each
instance we ran the same training algorithm twice, starting from different initializations.
Tables on the left show the accuracy of the networks against three types of input (clean, perturbed with FGSM, perturbed with PGD ran for 40 steps); the first column shows the resilience of the first network against  examples produced using its own gradients, the second column shows resilience of the second network against examples transferred from the former network.   The histograms reflect angles between pairs of gradients corresponding to the same inputs versus the baseline consisting of angles between  gradients from random pairs of points. Images on the right hand side reflect how the loss functions of the native and the transfer network change when moving in the direction of the perturbation; the perturbation is at $1$ on the horizontal axis. Plots in the top row are for FGSM perturbations, plots in the bottom row are for PGD perturbations produced over 40 iterations. }\label{fig:transfer}
\end{figure}

\clearpage
\section{MNIST Inspection}
\label{app:inspection}
The robust MNIST model described so far is small enough that we can visually
inspect most of its parameters. Doing so will allow us to understand how it is
different from a standard network and what are the general
characteristics of a network that is robust against $\ell_\infty$ adversaries.
We will compare three different networks: a standard model, and two
adversarially trained ones. The latter two models are identical,
modulo the random weight initialization, and were used as the public and secret
models used for our robustness challenge.

Initially, we examine the first convolutional layer of each network. We observe
that the robust models only utilize 3 out of the total 32 filters, and for
each of these filters only one weight is non-zero. By doing so, the convolution
degrades into a scaling of the original image. Combined with the bias and the
ReLU that follows, this results in a \emph{thresholding filter}, or equivalently
$\textrm{ReLU}(\alpha  x - \beta)$ for some constants $\alpha$, $\beta$.  From
the perspective of adversarial robustness, thresholding filters are immune to
any perturbations on pixels with value less than $\beta - \eps$. We visualize a
sample of the filters in Figure~\ref{fig:mnist_conv} (plots a, c, and e).

\begin{figure}[htp]
\begin{center}
\begin{tabular}{cc}
(a) Standard Model First Conv. Layers & (b) Natural Model Second Conv. Layer\\
\includegraphics[scale=0.13]{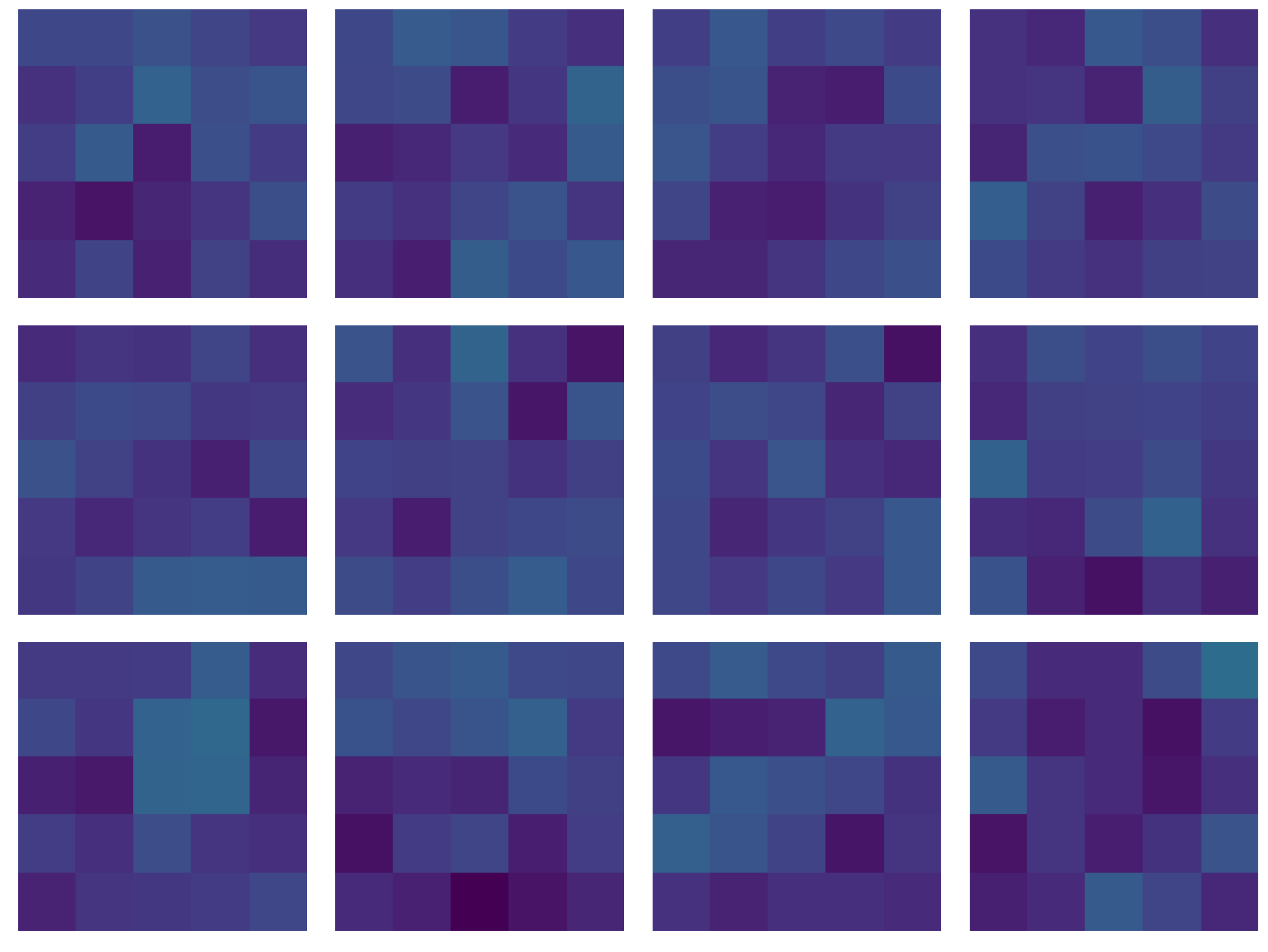} &
\includegraphics[scale=0.13]{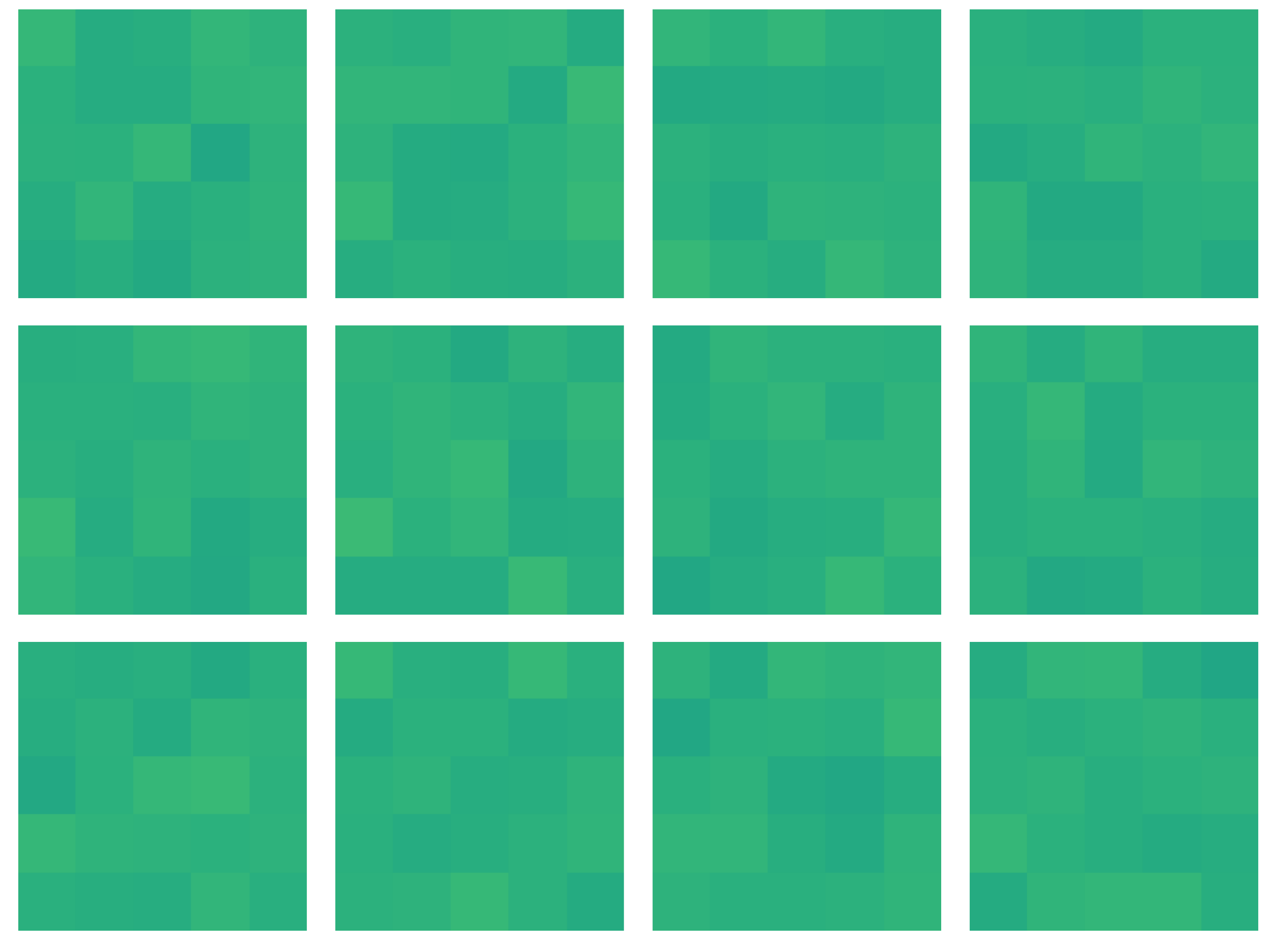} \\
(c) Public Model First Conv. Layers & (d) Public Model Second Conv. Layer\\
\includegraphics[scale=0.13]{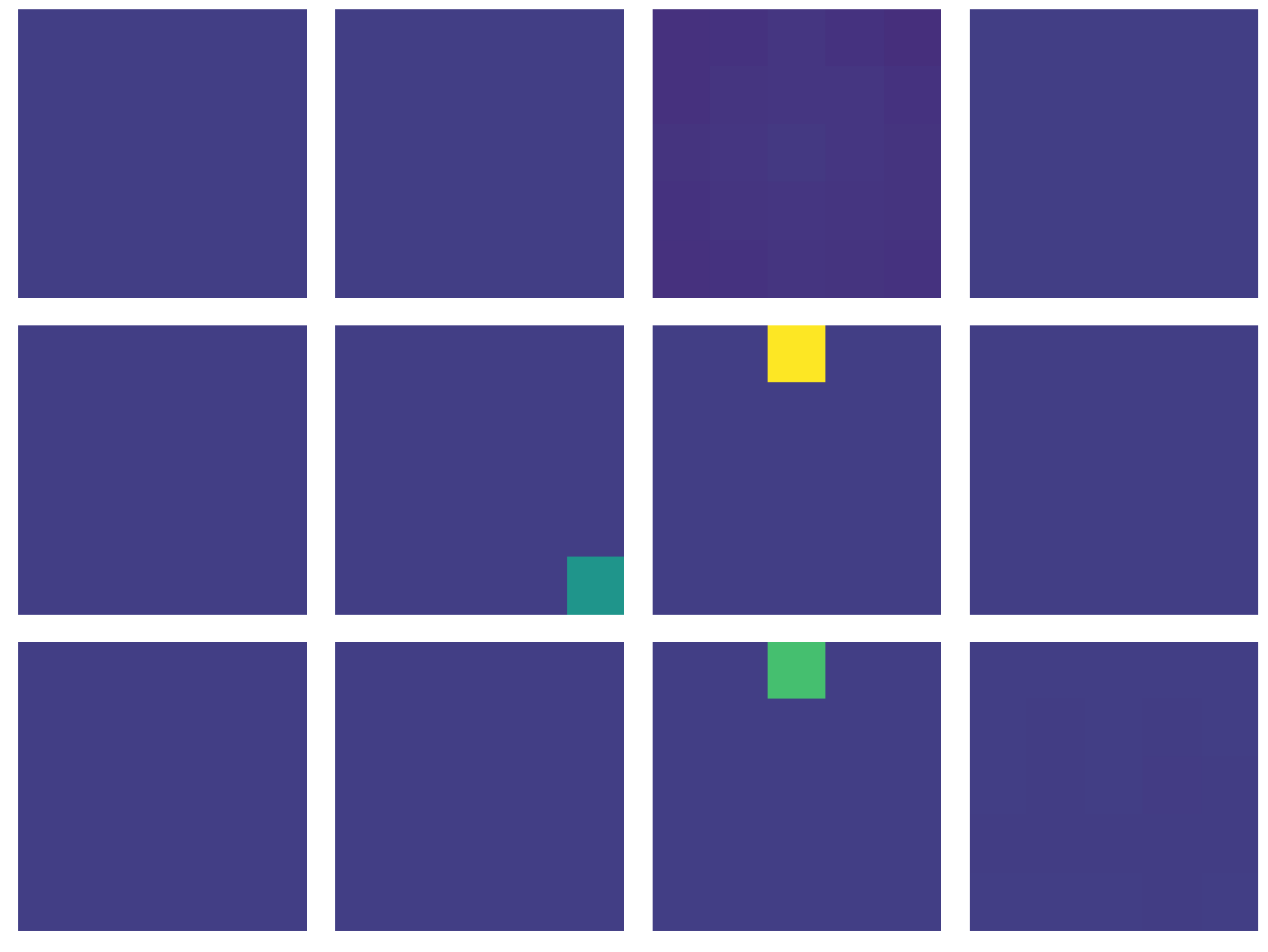} &
\includegraphics[scale=0.13]{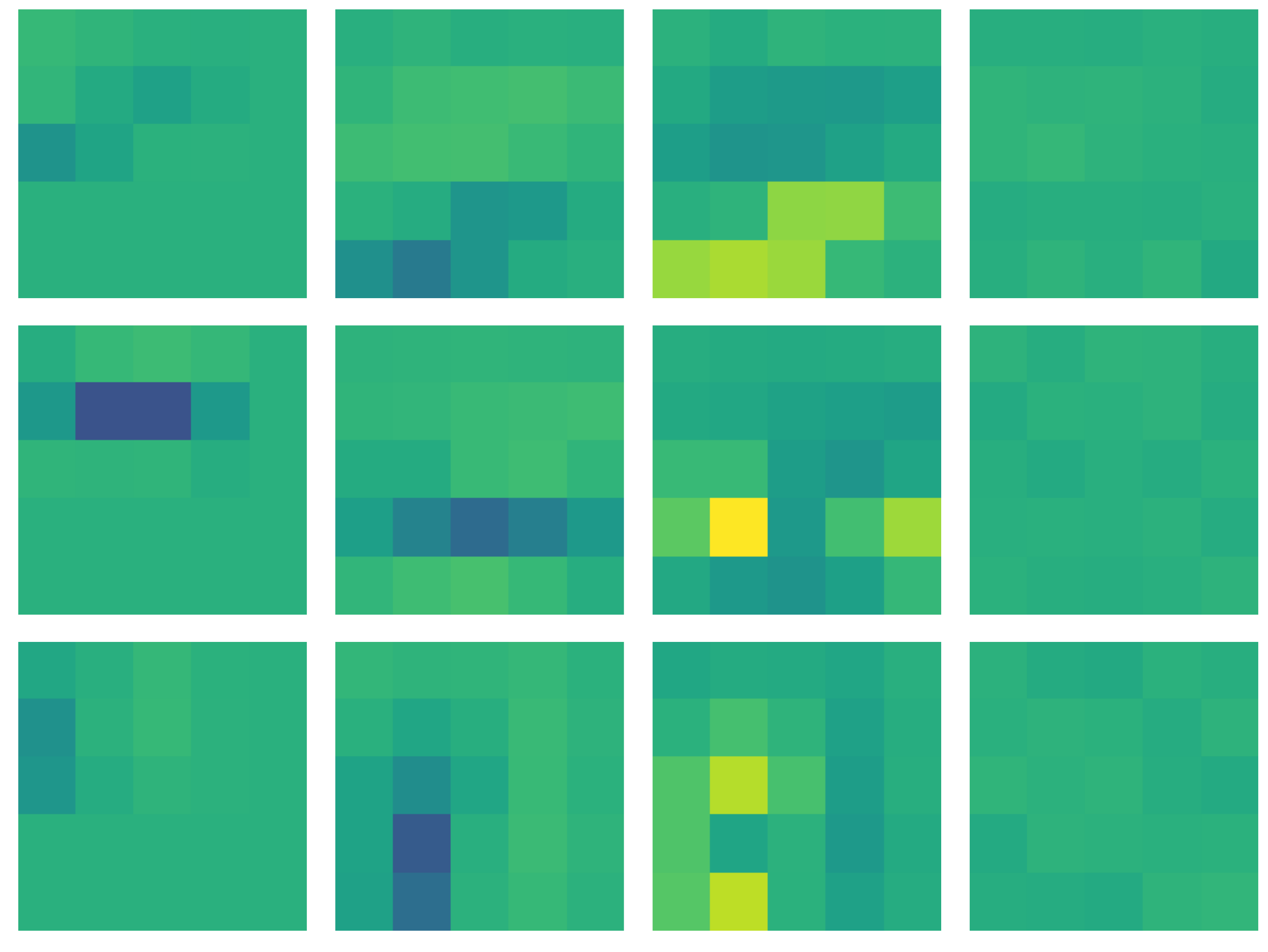} \\
(e) Secret Model First Conv. Layers & (f) Secret Model Second Conv. Layer\\
\includegraphics[scale=0.13]{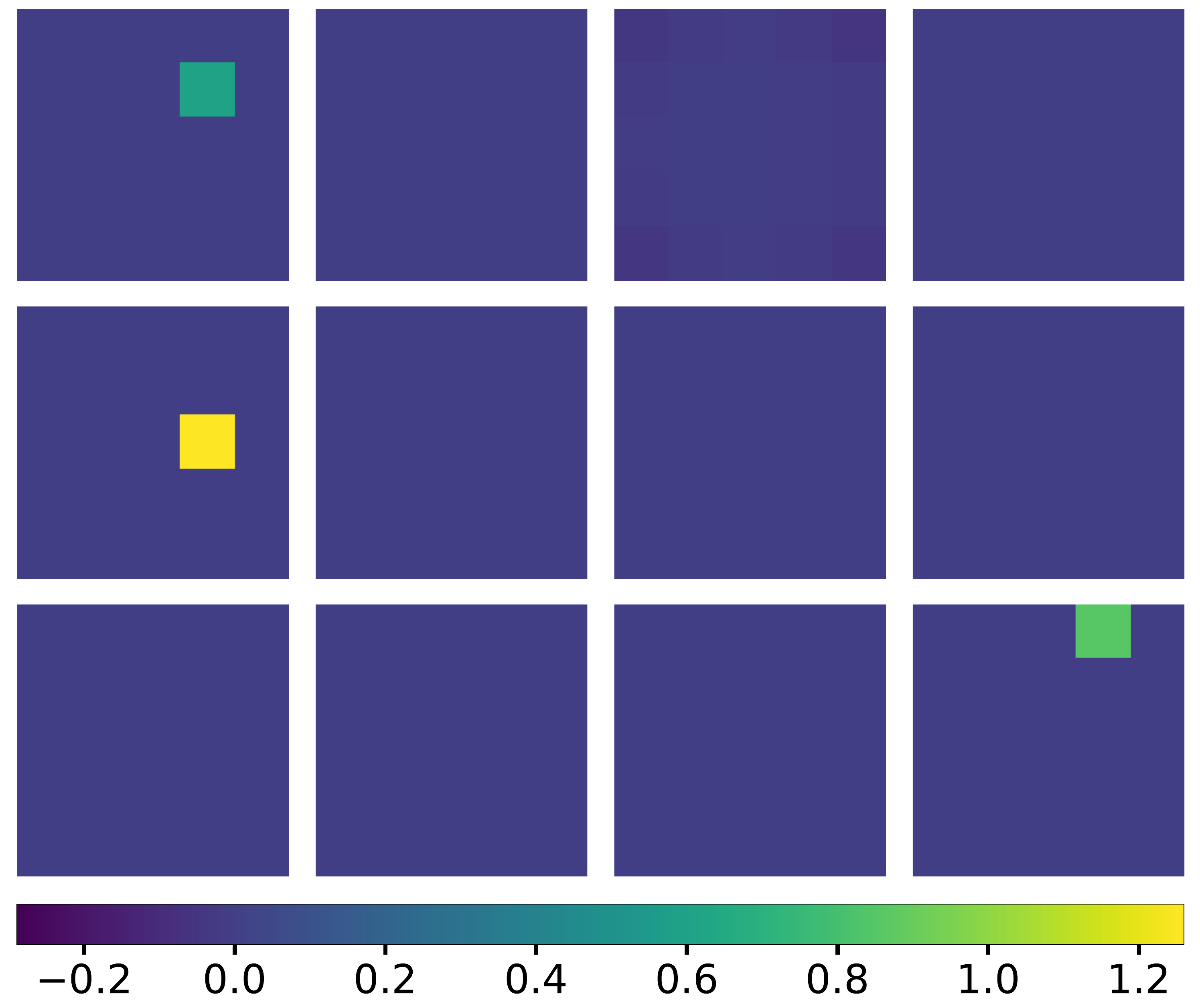} &
\includegraphics[scale=0.13]{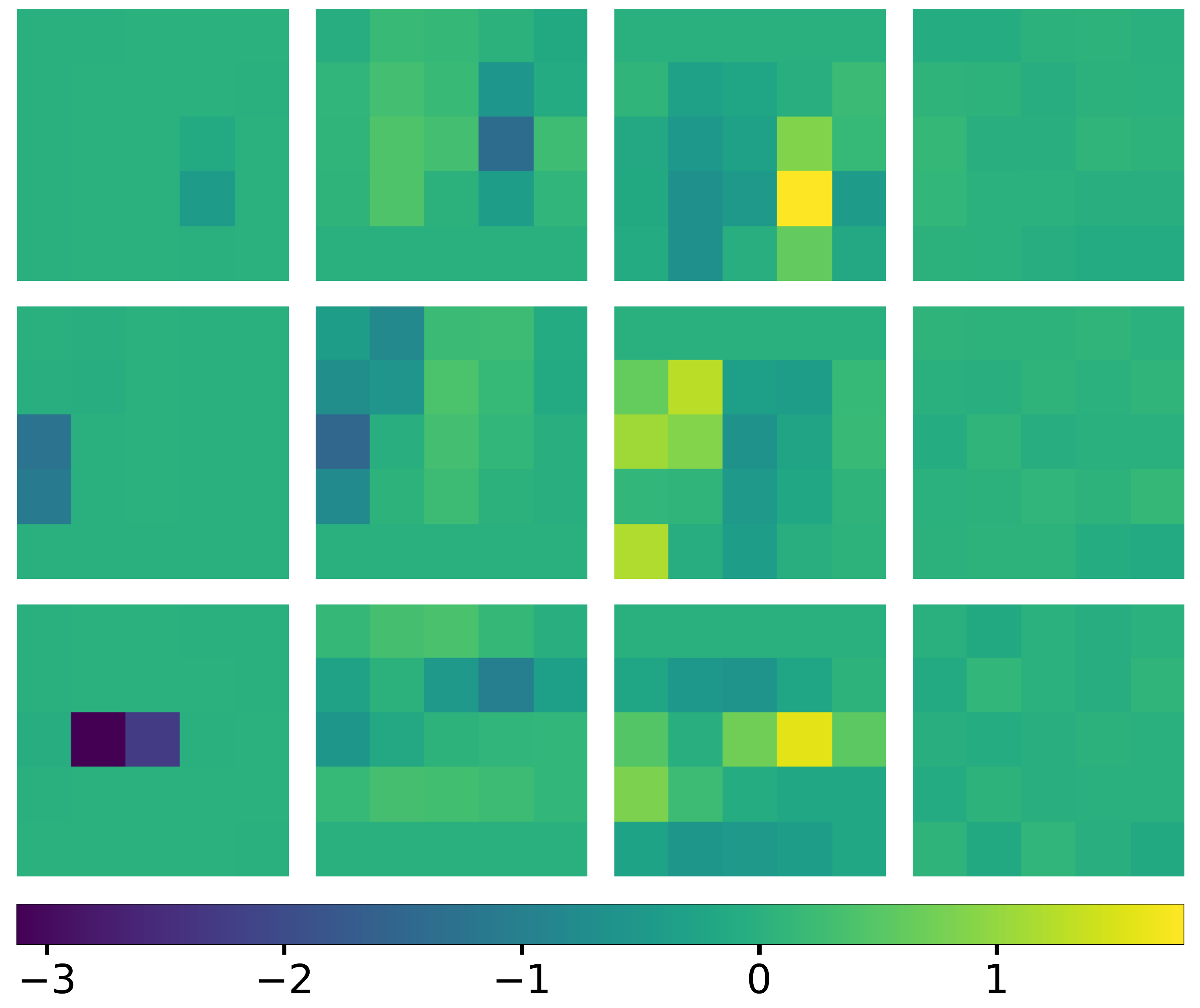}
\end{tabular}
\end{center}
\caption{Visualizing a sample of the convolutional filters. For the standard
model (a,b) we visualize random filters, since there is no observable difference
in any of them. For the first layer of robust networks we make sure to include
the 3 non-zero filters. For the second layer, the first three columns represent
convolutional filters that utilize the 3 non-zero channels, and we choose the
most interesting ones (larger range of values). We observe that adversarially
trained networks have significantly more concentrated weights. Moreover, the
first convolutional layer degrades into a few thresholding filters.}
\label{fig:mnist_conv}
\end{figure}

Having observed that the first layer of the network essentially maps the
original image to three copies thresholded at different values, we examine the
second convolutional layer of the classifier. Again, the filter weights are
relatively sparse and have a significantly wider value range than the standard
version. Since only three channels coming out of the first layer matter,
is follows (and is verified) that the only relevant convolutional filters are
those that interact with these three channels. We visualize a sample of the
filters in Figure~\ref{fig:mnist_conv} (plots b, d, and f).

Finally, we examine the softmax/output layer of the network. While the weights
seem to be roughly similar between all three version of the network, we notice a
significant difference in the class biases. The adversarially trained networks
heavily utilize class biases (far from uniform), and do so in a way very similar
to each other. A plausible explanation is that certain classes tend to be very
vulnerable to adversarial perturbations, and the network learns to be more
conservative in predicting them. The plots can be found in
Figure~\ref{fig:mnist_softmax}.

\begin{figure}[!ht]
\begin{center}
\begin{tabular}{cc}
\includegraphics[scale=0.7]{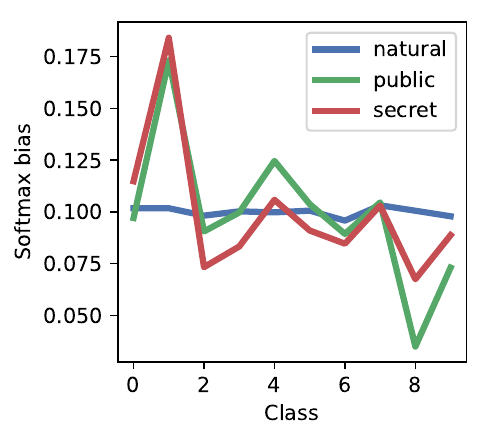} &
\includegraphics[scale=0.7]{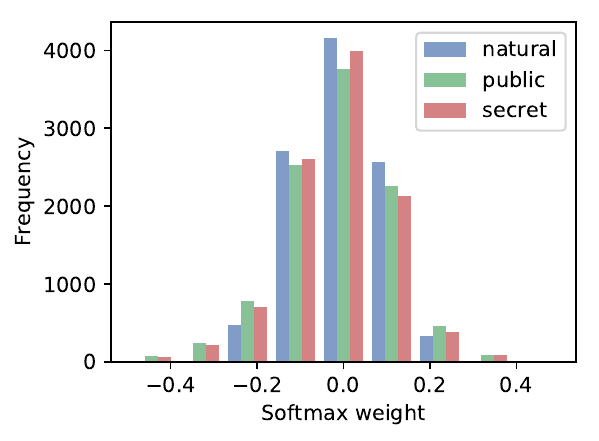}\\
(a) Softmax biases for each class & (b) Distribution of softmax weights
\end{tabular}
\end{center}
\caption{Softmax layer examination. For each network we create a histogram of
the layer's weights and plot the per-class bias. We observe that while weights
are similar (slightly more concentrated for the standard one) the biases are far
from uniform and with a similar pattern for the two adversarially trained
networks.}
\label{fig:mnist_softmax}
\end{figure}

All of the ``tricks'' described so far seem intuitive to a human and would seem
reasonable directions when trying to increase the adversarial robustness of a
classifier. We emphasize the none of these modifications were hard-coded in any
way and they were all \emph{learned solely through adversarial training}. We
attempted to manually introduce these modifications ourselves, aiming to achieve
adversarial robustness without adversarial training, but with no success. A
simple PGD adversary could fool the resulting models on all the test set
examples.

\section{Supplementary Figures}
\label{sec:omitted}

\begin{figure}[htb]
\begin{center}
{\setlength\tabcolsep{-.0cm}
\begin{tabular}{c c c c c}
& & MNIST & & \\
\includegraphics[height=.142\textwidth]{figures/mnist_prog_nat_1.pdf} &
\includegraphics[height=.142\textwidth]{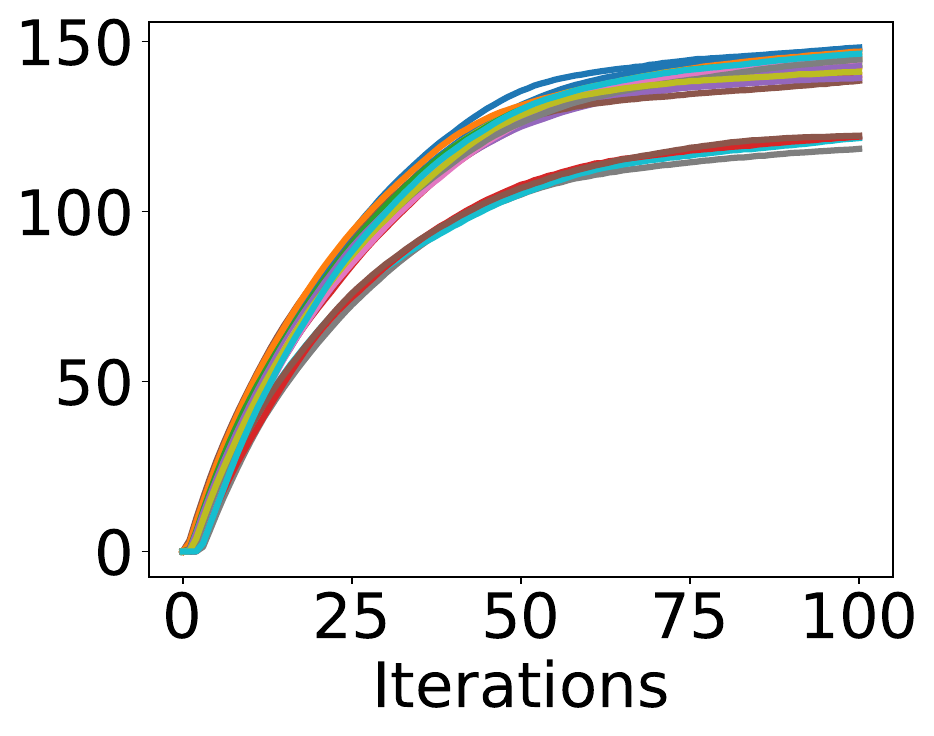} &
\includegraphics[height=.142\textwidth]{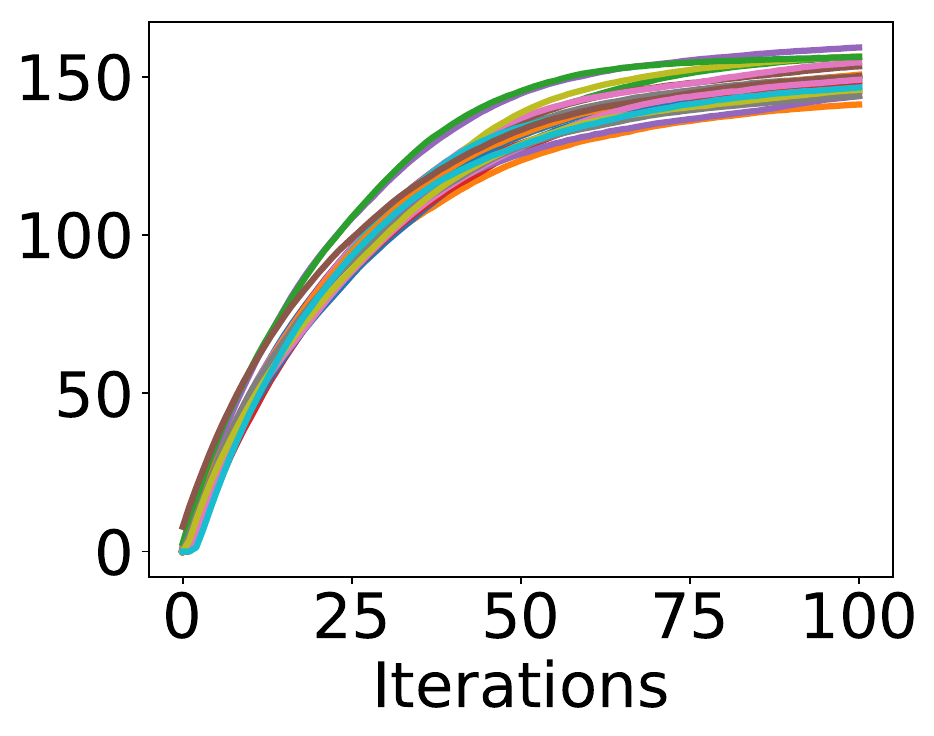} &
\includegraphics[height=.142\textwidth]{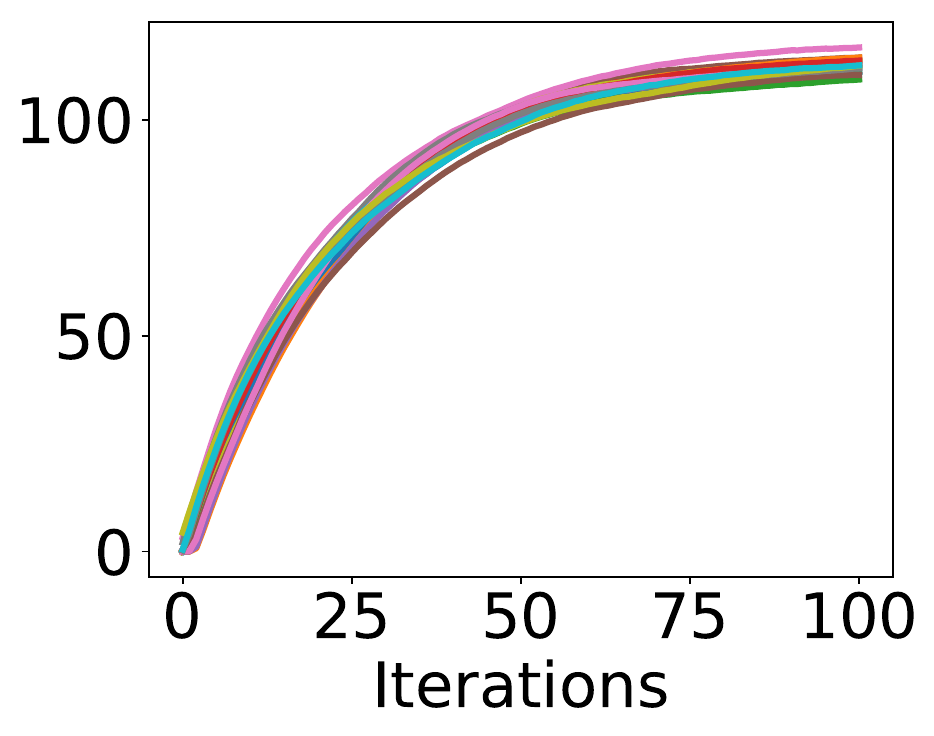} &
\includegraphics[height=.142\textwidth]{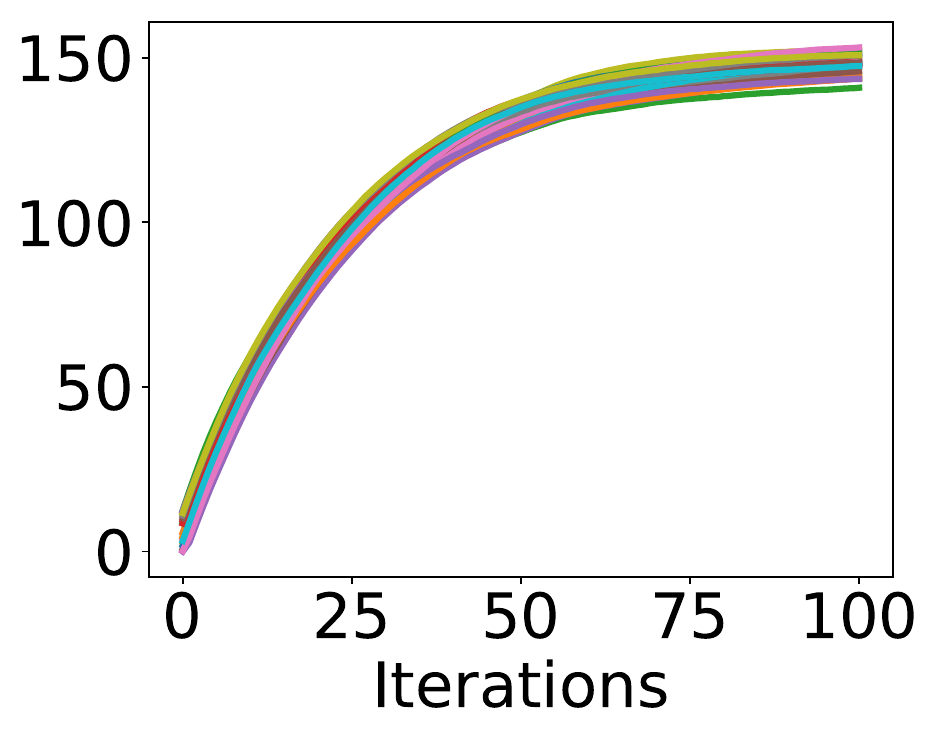} \\
\includegraphics[height=.142\textwidth]{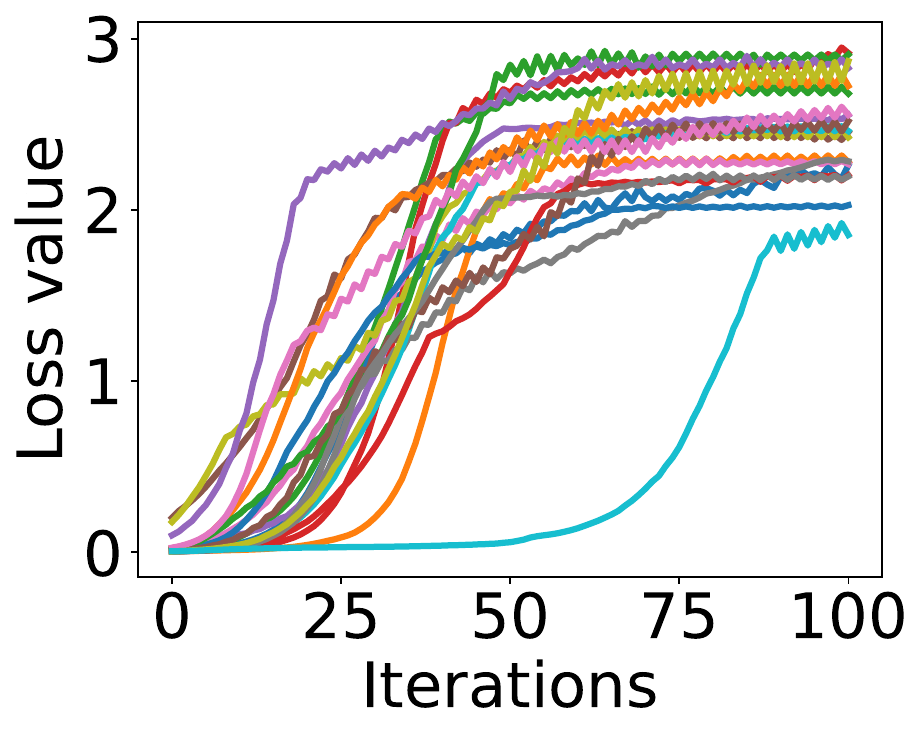} &
\includegraphics[height=.142\textwidth]{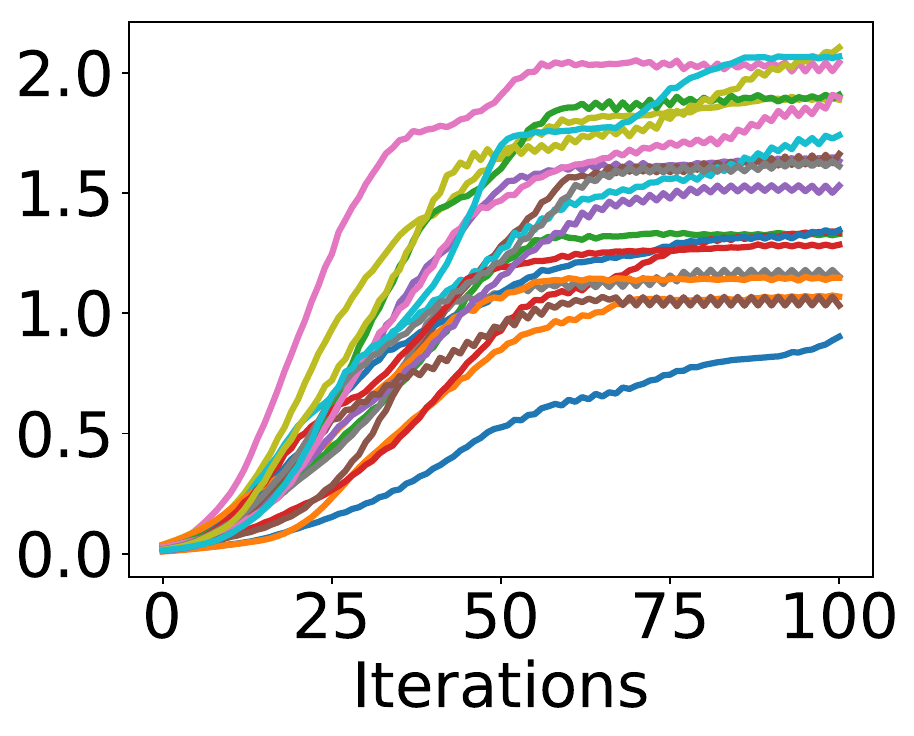} &
\includegraphics[height=.142\textwidth]{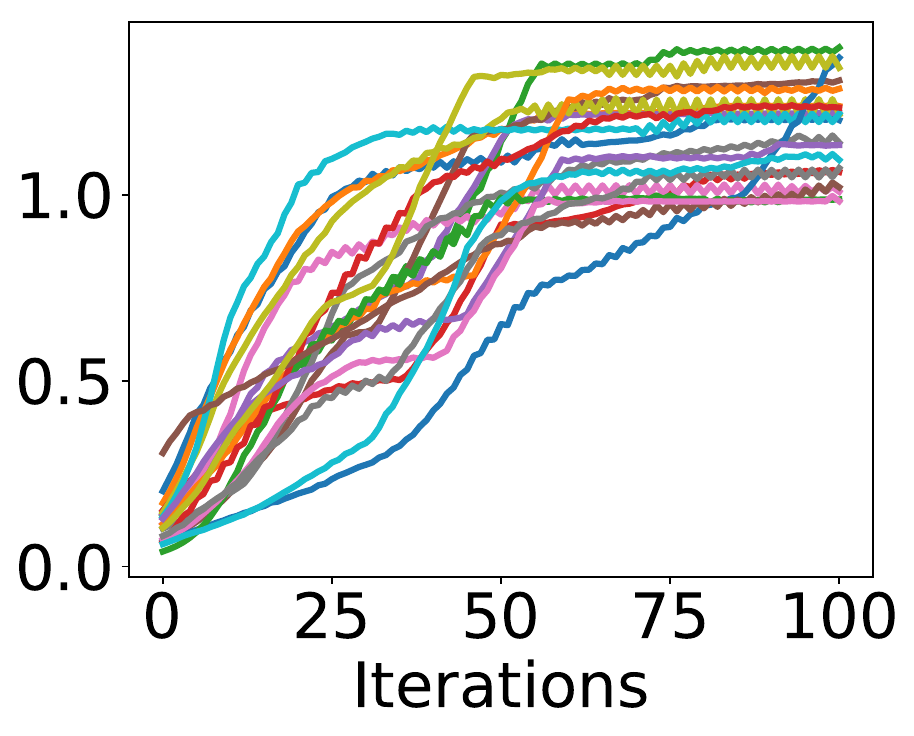} &
\includegraphics[height=.142\textwidth]{figures/mnist_prog_adv_4.pdf} &
\includegraphics[height=.142\textwidth]{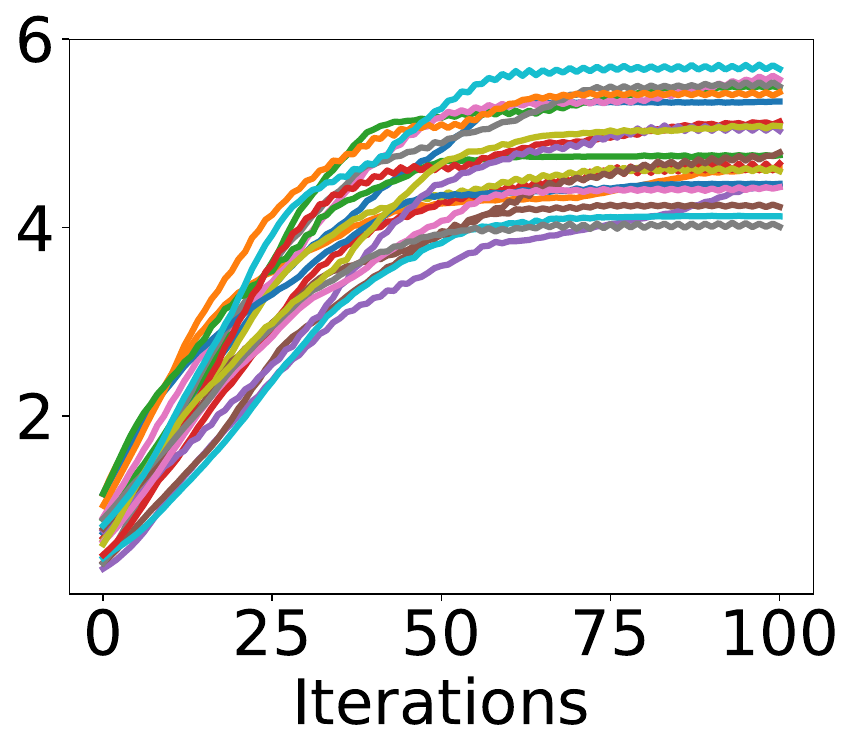} \\
& & CIFAR10 & & \\
\includegraphics[height=.142\textwidth]{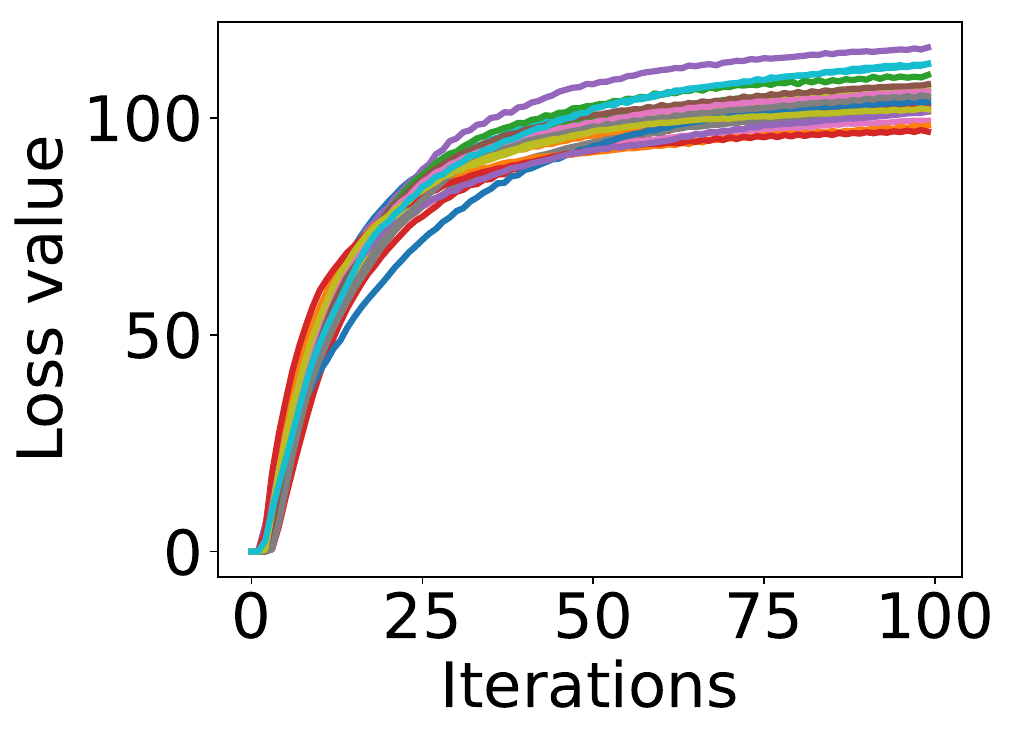} &
\includegraphics[height=.142\textwidth]{figures/amakelov/cifar10_vanilla/loss_progress_1.pdf} &
\includegraphics[height=.142\textwidth]{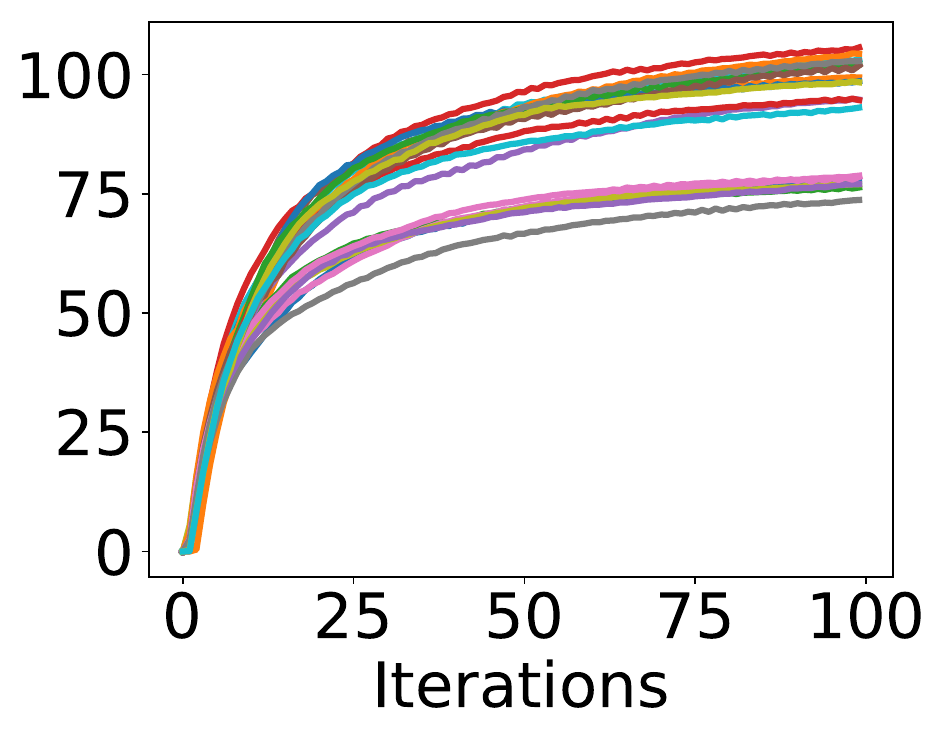} &
\includegraphics[height=.142\textwidth]{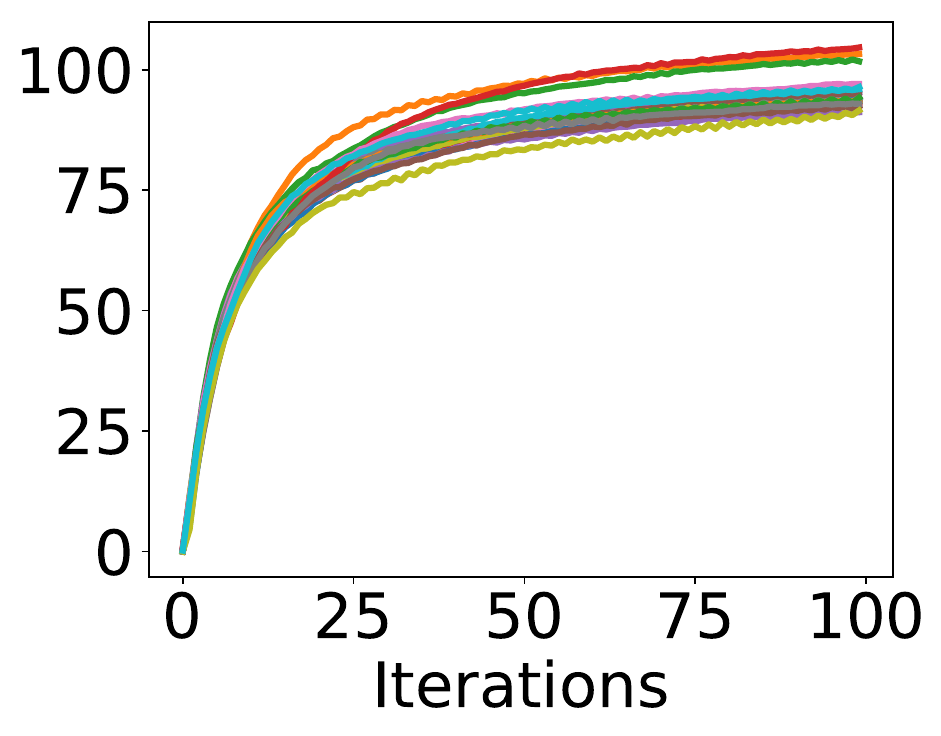} &
\includegraphics[height=.142\textwidth]{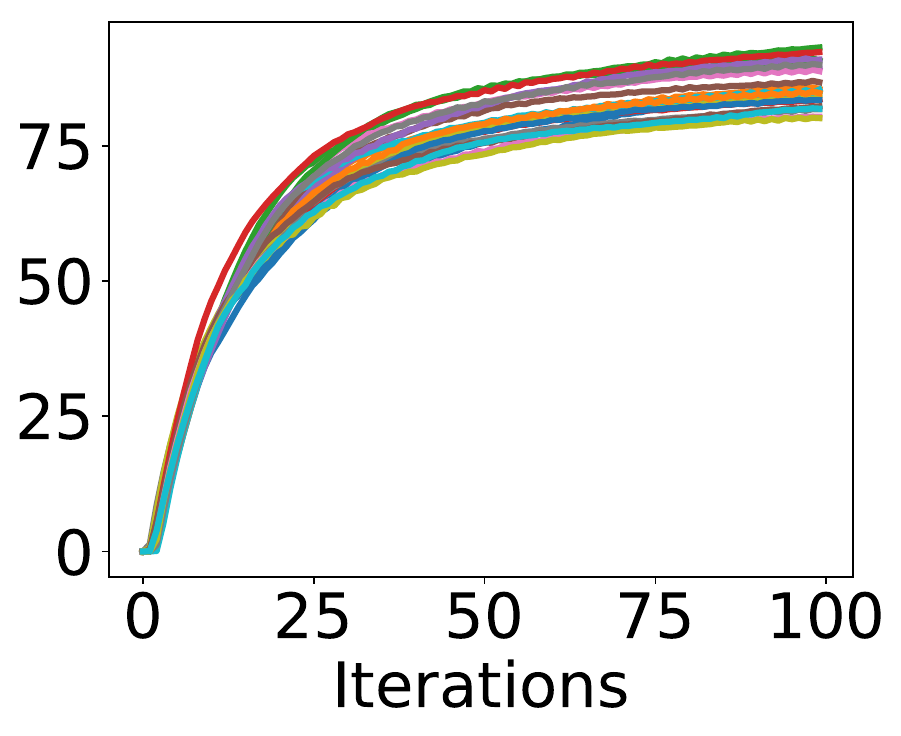}\\
\includegraphics[height=.142\textwidth]{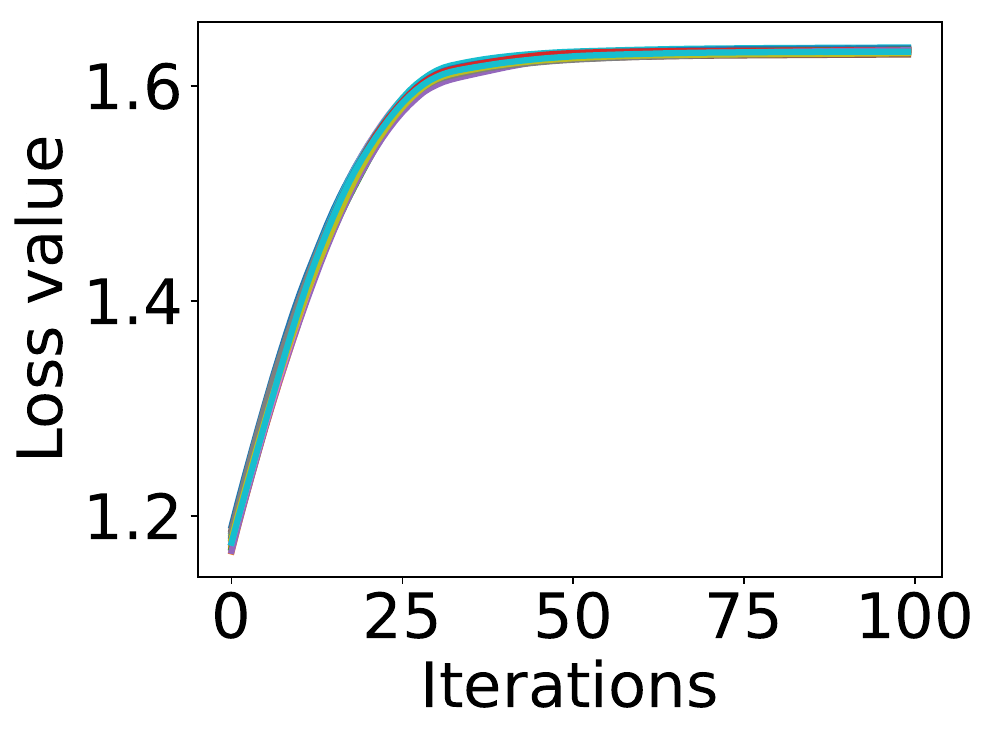} &
\includegraphics[height=.142\textwidth]{figures/amakelov/cifar10_wide_fgsm_k_trained/loss_progress_1.pdf} &
\includegraphics[height=.142\textwidth]{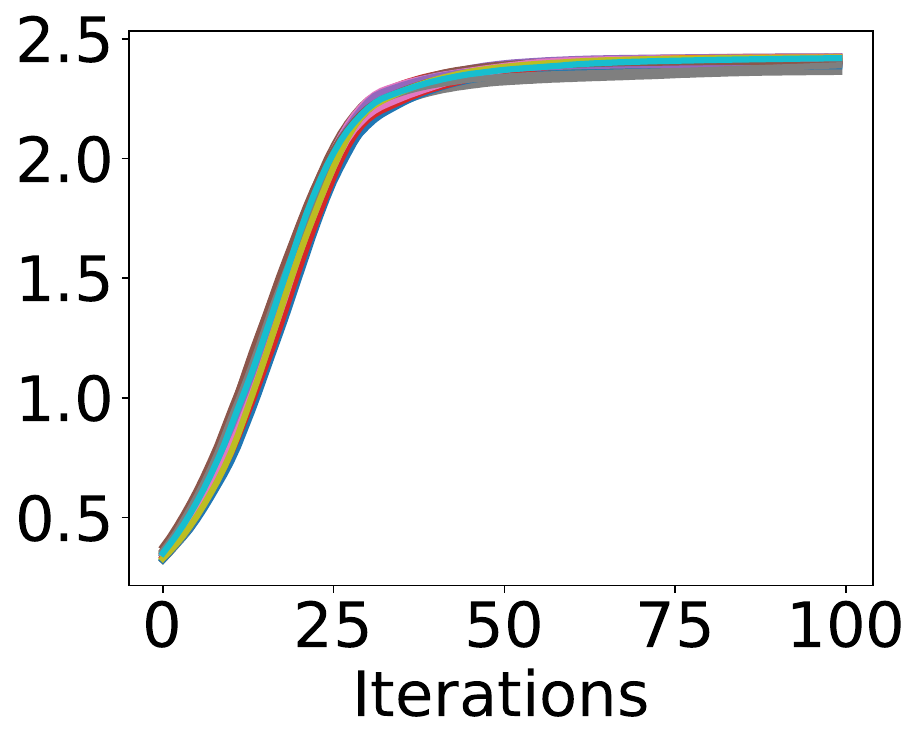} &
\includegraphics[height=.142\textwidth]{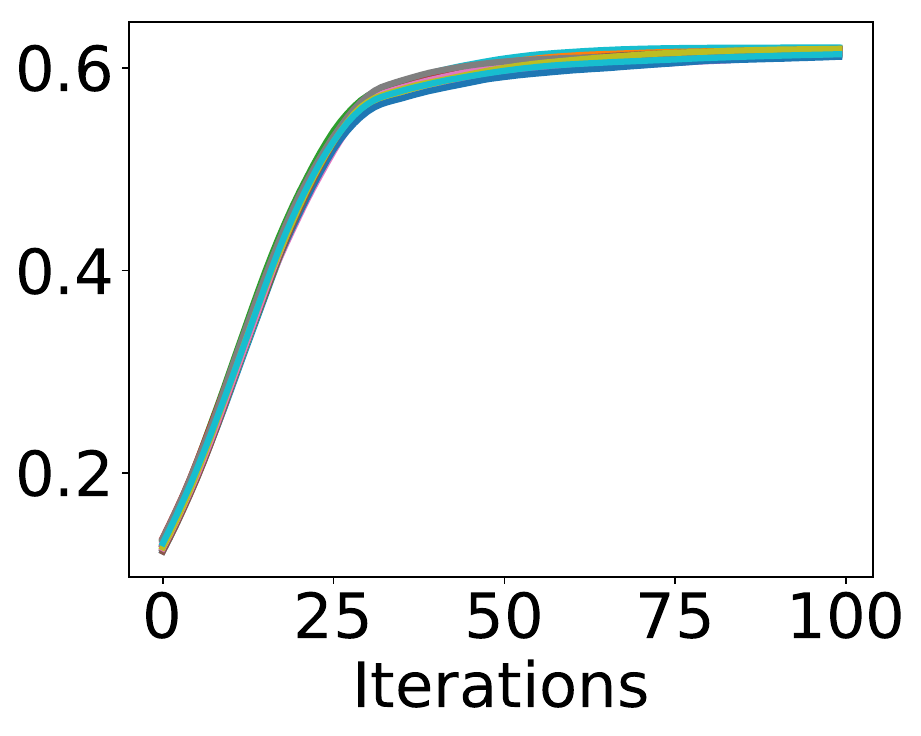} &
\includegraphics[height=.142\textwidth]{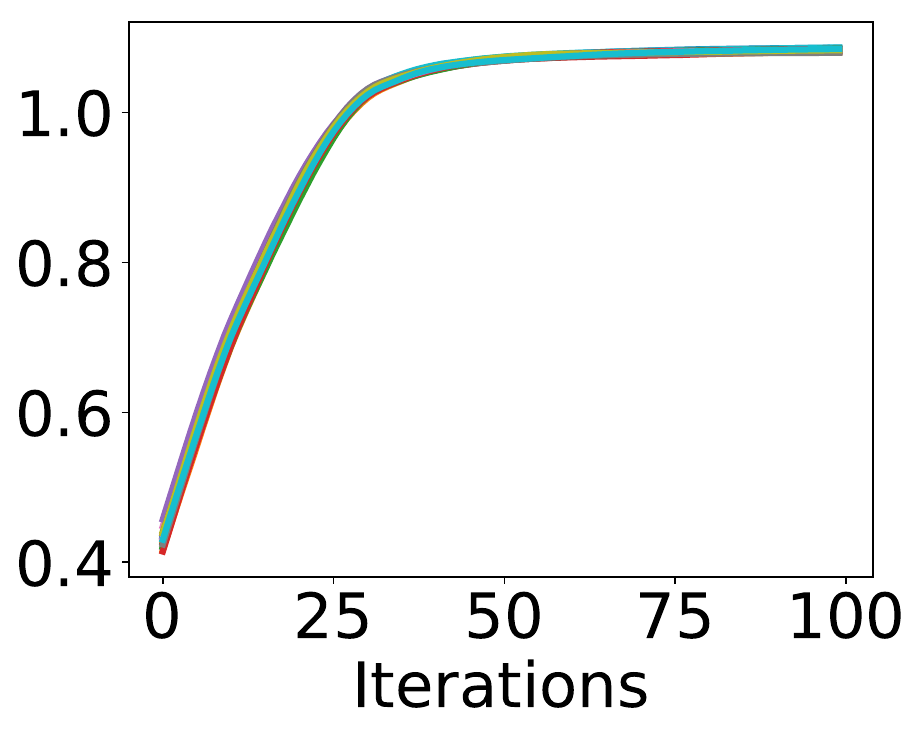}
\end{tabular}}
\end{center}
\caption{Loss function value over PGD iterations for 20 random restarts on
random examples. The 1st and 3rd rows correspond to standard networks, while the 2nd and 4th
to adversarially trained ones.}
\label{fig:mnist_progress_appendix}
\end{figure}

\begin{figure}[htb]
\begin{center}
{\setlength\tabcolsep{.2cm}
\begin{tabular}{c c c c c}
\includegraphics[height=.3\textwidth]{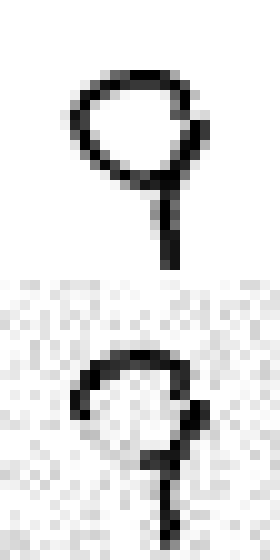} &
\includegraphics[height=.3\textwidth]{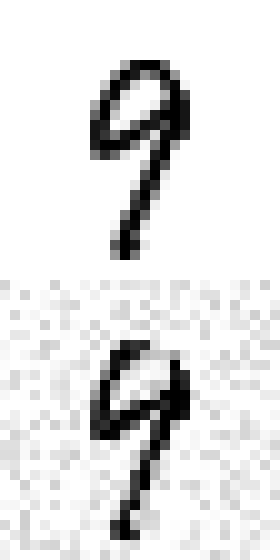} &
\includegraphics[height=.3\textwidth]{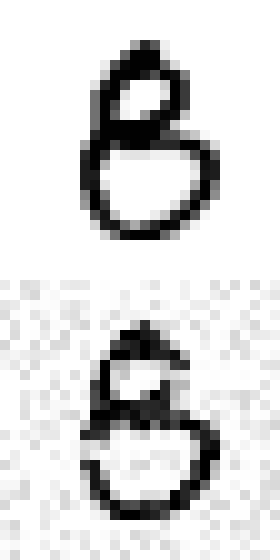} &
\includegraphics[height=.3\textwidth]{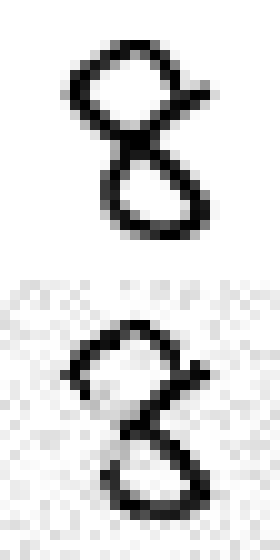} &
\includegraphics[height=.3\textwidth]{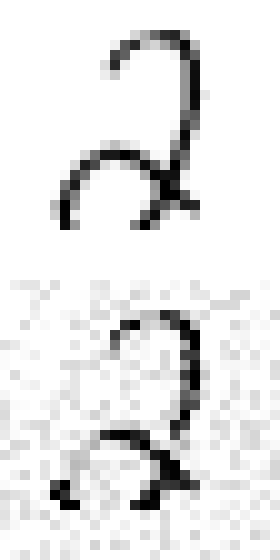} \\
Natural: 9 & Natural: 9 & Natural: 8 & Natural: 8 & Natural: 2 \\
Adversarial: 7 & Adversarial: 4 & Adversarial: 5 & Adversarial: 3 & Adversarial: 3 \\
\end{tabular}}
\end{center}
\caption{Sample adversarial examples with $\ell_2$ norm bounded by $4$. The
perturbations are significant enough to cause misclassification by humans too.}
\label{fig:l2_mnist_examples}
\end{figure}

\end{document}